\documentclass[twoside]{article}
\usepackage[accepted]{aistats2021}
\usepackage{alias}

\newif\ifbulletpoint
\newif\ifappendix
\newif\ifcolor
\bulletpointfalse
\colorfalse
\usepackage{hyperref}
\hypersetup{colorlinks, citecolor= black}%, citecolor=blakc}

\usepackage[round]{natbib}

\begin{document}



\twocolumn[

%\aistatstitle{Self-Concordant Analysis of Generalized Linear Bandits in Non-Stationary Environments}
\aistatstitle{Self-Concordant Analysis of Generalized Linear Bandits with
Forgetting}

\aistatsauthor{Yoan  Russac*  \And Louis Faury* \And  Olivier Cappé \And Aurélien Garivier }
\aistatsaddress{ DI ENS, CNRS, Inria,  \\ ENS, Université PSL \And  Criteo AI Lab,  \\ LTCI TélécomParis \And DI ENS, CNRS, Inria,  \\ ENS, Université PSL \And  UMPA, CNRS, \\ Inria, ENS Lyon} ]

\begin{abstract}
Contextual sequential decision problems with categorical or numerical
observations are ubiquitous and Generalized Linear Bandits
(GLB) offer a solid theoretical framework to address them.
In contrast to the case of linear bandits, existing algorithms for GLB
have two drawbacks undermining their applicability.  First, they rely
on excessively pessimistic concentration bounds due to the non-linear
nature of the model. Second, they require either non-convex projection
steps or burn-in phases to enforce boundedness of the estimators.
Both of these issues are worsened when considering non-stationary
models, in which the GLB parameter may vary with time.
In this work, we focus on self-concordant GLB (which include logistic
and Poisson regression) with forgetting achieved either by the use of
a sliding window or exponential weights. We propose a novel
confidence-based algorithm for the maximum-likehood estimator with
forgetting and analyze its perfomance in abruptly changing
environments.
These results as well as the accompanying numerical simulations
highlight the potential of the proposed approach to address
non-stationarity in GLB.
\end{abstract}

\section{INTRODUCTION}

\ifbulletpoint
\begin{itemize}
	{\color{red}\item GLM bandit: it is cool (say why) but it has some drawback}
	\end{itemize}
\else
\fi

In recent years, linear bandits \citep{abbasi2011improved,
  chu2011contextual, dani2008stochastic, rusmevichientong2010linearly}
have become the go-to paradigm to balance exploration and exploitation
in contextual sequential decision making problems. Linear bandits have typically
found applications for content-based recommendations
\citep{li2010contextual, valko2014spectral}, real-time bidding
\citep{flajolet2017real} and even mobile-health interventions
\citep{tewari2017ads}. Concurrently, Generalized linear bandits (GLB)
 have been introduced as a
generalization of linear bandits, able to describe broader reward models of
considerable practical relevance, in particular binary or categorical
rewards \citep{filippi2010parametric, li2017provably}. GLB are for instance a natural option in online advertising
applications where the rewards take the form of clicks
\citep{ChapelleLi11}.
In this work, we focus on deterministic algorithms and refer to
\citep{ChapelleLi11,kveton2020randomized} for randomized algorithms
applicable to GLB. Compared to the linear bandits case, there are two distinctive
drawbacks of GLB algorithms. The first is \textbf{(1)} the presence of
a problem-dependent constant, imposed by the non-linear nature of the
model, that is possibly \emph{prohibitively large} and has a negative impact
both on the design of algorithms and on their analysis. The second is
\textbf{(2)} the need to modify the Maximum Likelihood Estimator (MLE) to
ensure that it has a bounded norm. Usually this is achieved by resorting
to an additional \emph{non-convex} projection program applied to the MLE
\citep{filippi2010parametric}. These
distinctions correspond to a fundamental difference between the
models, and explain why methods developed for linear bandits
may fail in the case of GLB.
 
 \ifbulletpoint
 {\color{red}
 \begin{itemize}
 \item how it was solve in the stationary case
\begin{itemize}
	\item first drawback: tackled by \cite{faury2020improved}
	\end{itemize}
\end{itemize}}
\else
\fi

The first drawback \textbf{(1)} was recently addressed by
\cite{faury2020improved}, in the specific case of logistic
bandits. They showed that in this particular setting, the regret
bounds of carefully designed algorithms could be significantly
improved only at the cost of minor algorithmic modifications. Their
analysis tightens the gap with the linear case, and takes a significant step
towards the development of efficient GLB algorithms.

\ifbulletpoint
\begin{itemize}
	\item[] {\color{white}foo}
	\begin{itemize}
	{\color{red}\item second drawback: solution by \cite{li2017provably}, however not satisfactory and needs sto be generalized}
	\end{itemize}
\end{itemize}
\else
\fi

The second drawback \textbf{(2)} has seen little treatment in the
literature, except for the work of \citet{li2017provably} who proved
that the projection step of \citet{filippi2010parametric} could be
avoided by resorting to random initialization phases.
However, a careful examination of the required conditions shows that these
initialization phases can be prohibitively long to be deployed in scenarios
of practical interest.

  \ifbulletpoint
\begin{itemize}
	{\color{red}\item missing piece: non-stationarity}
	\end{itemize}
\else
\fi

The aforementioned improvements to the original GLB algorithm of
\citet{filippi2010parametric} were developed under a stationarity
assumption. However, non-stationary environments are ubiquitous
% omnipresent pour changer de l'abstract?
in real-world applications of contextual bandits. In the linear bandits
literature, this has motivated the development of adequate algorithms,
able to handle changes in the structure of the reward signal
\citep{cheung2019learning, russac2019weighted, zhao2020simple}.
\citet{russac2020algorithms} generalized such approaches to GLB,
but without addressing neither \textbf{(1)} nor
\textbf{(2)}. As a result, the practical relevance of their approach
remains questionable and the development of \emph{efficient} and
\emph{non-stationary} GLB algorithms stands incomplete.

  \ifbulletpoint
\begin{itemize}
	{\color{red}\item contributions: solve the two drawbacks under non-stationarity}
	\end{itemize}
\else
\fi

This paper aims at closing this gap. We study a broad family of GLB,
known as \emph{self-concordant} (which includes for instance the
logistic and Poisson bandits), in environments where the 
parameter is allowed to switch arbitrarily over time. Under this
setting, we answer \textbf{(1)} by providing a non-trivial extension of
the concentration results from \citet{faury2020improved}. We also leverage the self-concordance property to \emph{remove} the projection
step, henceforth overcoming \textbf{(2)}. This is made possible by an
improved characterization of the, possibly weighted,
 MLE in (self-concordant)
generalized linear models. Combined together, these two contributions lead to the
design of \emph{efficient} GLB algorithms, with improved regret bounds
and which do not require to solve hard (i.e. non-convex) optimization
programs. In doing so, we also answer the long-standing issue of
providing proper confidence regions centered around the pristine MLE in GLB.

%%% Local Variables:
%%% TeX-master: "main.tex"
%%% End:

\section{BACKGROUND}
\label{section:1}
\ifbulletpoint
{\color{red} $\hat\theta_t, g_t, \mathbf{V_t}, \mathbf{H_t}, \cm, \km $ +
 assumption: reward max, bounded decision set}
\else
\fi
\subsection{Setting and Assumptions}

At each time step, the environment provides a time-dependent action set
$\mathcal{A}_t$ and the agent plays a $d$-dimensional action
$a_t \in \mathcal{A}_t$. We will assume that the reward's distribution
belongs to a \textit{canonical exponential family} with respect to a
reference measure $\nu$, such that
$d \mathbb{P}_\theta(r |a) = \exp( r a^\top \theta - b(a^\top \theta)
+ c(r)) d\nu(r)$.  Here, the function $c(\cdot)$ is real-valued and
$b(\cdot)$ is assumed to be twice continuously differentiable. Thanks
to the properties of exponential families, $b$ is convex and can be
related to the function $\mu = \dot{b}$, itself referred to as
the \emph{inverse link} or \emph{mean} function. A key feature of this description is that
given a ground-truth parameter $\theta^\star$, selecting an action
$a_t$ at time $t$ yields a reward $r_{t+1}$ conditionally independent
on the past and such that
$
\mathbb{E}[r_{t+1} |a_t] = \mu(a_t^\top \theta^\star)
$.

The non-stationary nature of the considered environments is
characterized as follows: the bandit parameter $\theta^\star$ is
allowed to change in an arbitrary fashion up to $\Gamma_T$ times
within the horizon $T$.  In the following, $\theta^\star$ will
be indexed by $t$ to clearly exhibit its dependency w.r.t round $t$,
and the reward signal will follow
$$
\mathbb{E}[r_{t+1} |a_t] = \mu(a_t^\top \theta_t^\star) \;.
$$
The focus of this paper is the \textit{dynamic regret} defined as
$$
R_T = \sum_{t=1}^T \max_{a \in \mathcal{A}_t } \mu\big(a^\top \theta^\star_t\big)
- \mu\big(a_t^\top \theta^\star_t\big)\;.
$$
Note that in this setting, there is no fixed best arm, 
both due to the non-stationarity of the environment and to the fact that the action set $\mathcal{A}_t$ may vary with time. 
% checker la phrase du dessus.
We will work under the following assumptions.

\begin{ass}[Bounded actions and bandit parameters]
\label{ass:bounded_actions}
  $$\forall t \geq 1, \lVert \theta^\star_t \rVert_2 \leq S  \quad \text{and}
  \quad \forall a \in \mathcal{A}_t,  \lVert a \rVert_2 \leq 1\;.$$
\end{ass}
We define the admissible parameter space $\Theta = \big\{ \theta \in \mathbb{R}^d,
  \lVert \theta \rVert_2 \leq S\big\}$.
\begin{ass}[Bounded rewards]
  \label{ass:bounded_rewards}
  $$\exists m\in\mathbb{R}^+  \text{such that}  \;  \forall t \geq 1,  0 \leq r_t 
  \leq \rM \;.$$
\end{ass}
\begin{ass}
 \label{ass:link_fun}
 The mean function $\mu: \mathbb{R} \mapsto \mathbb{R}$ is continuously
 differentiable, Lipschitz with constant $k_\mu$ and such that
 $$
 c_\mu = \inf_{ \theta \in \Theta, \lVert a \rVert_2 \leq 1}
  \dot{\mu}\big(a^\top \theta\big) > 0 \;.
 $$
\end{ass}
 The quantity $c_\mu$ is crucial in the analysis, as it represents the
(worst case) sensitivity of the mean function.
Our last assumption differs from most of existing works as we focus
here on \emph{self-concordant} GLMs. This assumption on the curvature
of the mean function is rather mild, and covers for instance the
logistic and Poisson models.
\begin{ass}[Generalized self-concordance]
 \label{ass:SC}
 The mean function verifies $
|\ddot{\mu}| \leq \dot{\mu} \;.
$
\end{ass}
In order to estimate the unknown bandit parameter $\theta_t^\star$,
 we will adopt a \emph{weighted} regularized maximum-likelihood 
principle. Formally,
 we define $\hat\theta_t$ for $ \lambda>0$ and $\gamma\in(0,1]$ as the solution
 of the strictly convex program
\begin{equation}
\label{eq:MLE_D}
\hat\theta_t= \argmin_{\theta \in \mathbb{R}^d} - \sum_{s=1}^{t-1} \gamma^{t-1-s} \log
 \mathbb{P}_\theta(r_{s+1} | a_s) + \frac{\lambda}{2} \lVert \theta \rVert_2^2
\;.
\end{equation}
Equivalently, $\hat{\theta}_{t}$ may
be defined as the minimizer of
$- \sum_{s=1}^{t-1} \gamma^{-s} \log \mathbb{P}_\theta(r_{s+1} | a_s)
+ \frac{\lambda \gamma^{-(t-1)}}{2} \lVert \theta \rVert_2^2$, with
time-independent increasing weights $\gamma^{-s}$ and
time-varying regularization $\lambda \gamma^{-(t-1)}$, which is more handy
for analysis purposes, see~\citep{russac2019weighted}.
%\textcolor{red}{AG: pas sûr qu'on soit obligés de laisser ce paragraphe}

 %  \Longleftrightarrow
%   \sum_{s=1}^{t-1} \gamma^{t-1-s} (r_{s+1} - \mu(x_s^\top \theta)) x_s -
 % \lambda \theta = 0
 % The equivalence is obtained by
%  differentiating the left hand-side and using the concavity in $\theta$.

%We 	also define $g_t$ such that,
%  $$g_t(\theta) = \sum_{s=1}^{t-1} \gamma^{t-1-s} \mu(x_s^\top \theta) x_s + \lambda
%  \theta\;.$$
%Next, we define the different matrices of interest.
%$$
%\VV_t = \sum_{s=1}^{t-1} \gamma^{t-1-s}
%x_s x_s^\top  + \lambda/c_\mu \identity{d}
%\quad \textnormal{and} \quad
%\wVV_t = \sum_{s=1}^{t-1} \gamma^{2(t-1-s)}
%x_s x_s^\top  + \lambda/c_\mu \identity{d}\;.
%$$
%$$
%\HH_t(\theta) = \sum_{s=1}^{t-1} \gamma^{t-1-s} \dot{\mu}
%(x_s^\top \theta) x_s x_s^\top + \lambda \identity{d}
%\quad \textnormal{and} \quad
%\wHH_t(\theta) =  \sum_{s=1}^{t-1} \gamma^{2(t-1-s)} \dot{\mu}
%(x_s^\top \theta) x_s x_s^\top + \lambda \identity{d}\;.
%$$
% essayer d'inclure la différence avec GLM-UCB: notamment faire apparaître que eux font une projection et pas nous.
%$\HH_t$ is the Hessian of the weighted negative log-loss.
% The definitions when a sliding window $\tau$ is used  are similar.
% The weights are equal to 1 when $
%  \max(t- \tau, 1) \leq s \leq t-1$ and 0 otherwise. More details can be found in
%Appendix \ref{subsection:notations_SW_appendix}.

\subsection{Stationary GLB}
GLB were first considered in the seminal work of
\citet{filippi2010parametric} who proposed
$\tt GLM \mhyphen UCB$, an optimistic algorithm with
a regret upper bound of the form
$\tilde{\mathcal{O}}(\cm^{-1}d\sqrt{T})$. A key characteristic of
$\tt GLM \mhyphen UCB$ is a \emph{projection step}, used to map
the MLE onto the set of admissible parameters
$\Theta$. Formally, when the MLE $\hat\theta_t$ is not in $\Theta$, it needs to be replaced by
\begin{align}
	\tilde\theta_t = \argmin_{\theta\in\Theta}\left \lVert \sum_{s=1}^{t-1}\left[\mu\big(a_s^\top\theta \big) -
	 \mu\big(a_s^\top\hat\theta_t\big)\right]a_s\right\rVert_{\mathbf{V_t}^{-1}} 
\label{eq:filippi_proj}
\end{align}
where $\mathbf{V_t}$ is an invertible $d\times d$ square matrix.

With $\tt GLM \mhyphen UCB$, both the size of the confidence set (thus
 the exploration bonus) and the regret bound scale as
$\cm^{-1}$. 
% (thus the exploration bonus), and .... -> eviter deux répétitions de and.
However, this constant can be prohibitively large. In the
cases of the logistic and Poisson bandits, one has
$\cm^{-1}\geq e^{S}$, revealing an \emph{exponential} dependency on
$S$. If we consider the example of click prediction in online
advertising with the logistic GLB, $\cm^{-1}$ is of the order $10^3$,
corresponding to typical click rates of less than a percent.  

This
critical dependency was addressed by \citet{faury2020improved} for the
logistic bandit. They introduce $ \tt LogUCB1$ and $\tt LogUCB2$ for
which they respectively prove
$\widetilde{\mathcal{O}}(\cm^{-1/2}d\sqrt{T})$ and
$\widetilde{\mathcal{O}}(d\sqrt{T}+\cm^{-1})$ regret upper bounds. Their
analysis relies on the self-concordance
property of the logistic log-likelihood. 
Self-concordance offers a
refined way to control the curvature of the log-likelihood, and has
been used in batch statistical learning \citep{bach2010self} and
online optimization \citep{bach2013non} (see also \citep[Section 9.6]{Boyd:Vandenberghe:04} for a broader picture).
\ifcolor
{\color{purple}
\else 
\fi 
However, the analysis of \cite{faury2020improved} does
not use the self-concordance to its fullest and a projection step is still required, as detailed in Section~\ref{sec:discussion}.
\ifcolor
}
\else 
\fi

%how the use of the the self-concordance assumption 
%differs in the two works.
%makes it possible to analyze 
%the MLE without introducing any additional estimator.}

Since the mean function $\mu$ can be non-convex (as for example in the case of logistic regression), the projection step
defined in Equation~\eqref{eq:filippi_proj} generally involves the minimization of a non-convex function. Solving this program can be arduous and finding ways to bypass it is desirable. This was achieved by
\citet{li2017provably} using a \emph{burn-in phase} corresponding to  an
initial number of rounds during which the agent plays randomly. This
ensures that $\hat\theta_t$ stays in $\Theta$ for subsequent rounds and therefore
avoids the projection step. This technique was re-used in other recent
works, such as \citep{kveton2020randomized, zhou2019learning}. A major
drawback of this approach however is the length of this burn-in phase, which
typically grows with $\cm^{-2}$ \citep[Section
4.5]{kveton2020randomized}. In the previously cited example of
click-prediction, this would lead the agent to act randomly for
approximately $10^6$ rounds.

% {\color{purple} 
% This problem is particularly acute in abruptly changing environments, where long burning phases ensuring that $\hat{\theta}_t \in \Theta$ cannot be multiplied repeatedly. This is why finding alternatives where no projection is required is hardly avoidable in this particular setting.
% }

%\subsection{Non-stationarity: Challenges and Contributions}
\subsection{Forgetting in Non-Stationary Environments}

Motivated by the non-stationary nature of most real-life applications
of contextual bandits, a consequent theory for linear bandits in non-stationary
environments has been recently developed \citep{cheung2019hedging,
russac2019weighted,zhao2020simple}. 
\ifcolor
{\color{purple}
\else
\fi
We focus here on forgetting policies, 
a broader perspective is discussed in~Section \ref{sec:discussion}.
\ifcolor
} 
\else
\fi
In \citep{cheung2019hedging}, a
sliding window is used and the estimator is constructed based on the
most recent observations only. In \citep{russac2019weighted}
exponentially increasing weights are used to give more importance to
most recent observations. In \citep{zhao2020simple} the algorithm is
restarted on a regular basis. These contributions were generalized to
GLB by \cite{russac2020algorithms,cheung2019hedging, zhao2020simple}.
However, the approach of \cite{russac2020algorithms} still suffers
from the aforementioned limitations (dependency w.r.t. $\cm$ and need
for a projection step) while the analysis of both
\cite{cheung2019hedging} and \cite{zhao2020simple} are missing
key features of the problem at hand (see \cite[Section
1]{russac2020algorithms}).

The non-stationary nature of the problem rules out the use of 
burning phases as changes in the GLB parameter can lead $\hat\theta_t$
to leave $\Theta$, even when well initialized. This also accentuates
% increases plutot que accentuates?
the inconveniences brought by the projection step, as
$\hat\theta_t$ leaving $\Theta$ is more likely to happen. 
This is why finding alternatives without projection is even more attractive in this particular setting.
Furthermore, a
generalization of the improvements brought by \cite{faury2020improved}
to non-stationary world is missing, and it is unclear if the
dependency in $\cm$ can still be reduced in this harder setting.

\subsection{Contributions}

The present paper addresses these challenges, focusing on the use of
exponential weights to adapt to changes in the model.
%\footnote{We have also considered the use of a sliding-window, as in \cite{cheung2019hedging}, due
%to space limitation 
%all results and proofs pertaining to this way of handling non-stationarity are deferred to the appendix.}. 
First, we extend
in Theorem~\ref{thm:instantaneous_main} the Bernstein-like
tail-inequality of \cite[Theorem 1]{faury2020improved} to
\emph{weighted} self-normalized martingales. We then
leverage the self-concordance property (Assumption~\ref{ass:SC}) to
provide an improved characterization of the maximum-likelihood
estimator (Proposition~\ref{prop:deviation_main}). This allows to provide concentration
guarantees \emph{without} projecting $\hat\theta_t$ back to $\Theta$.
 Combining these results leads to the $\DGLM$ strategy
(Algorithm~\ref{alg:D-GLM}), which does not resort to a non-convex
projection step and enjoys an
$\tilde{\mathcal{O}}(\cm^{-1/3} d^{2/3} \Gamma_T^{1/3} T^{2/3})$
worst case regret upper bound (Theorem~\ref{th:regret_D_no_proj_main}).
\ifcolor
{ \color{purple}
\else
\fi
A $\mathcal{O}(\cm^{-1/2} \Delta^{-1} d \sqrt{\Gamma_T T})$ regret bound is also obtained 
(Theorem \ref{thm:fjqiqposejf}) under an additional
minimal gap $\Delta > 0$ assumption (Assumption \ref{ass:minimum_gap}).
\ifcolor
}
\else
\fi
A summary
of our contributions and comparison with prior work are given in
Table~\ref{tab:regret_comparison}.
\begin{table*}[t]
    \centering
    \begin{tabular}{|c||c|c|c|c|}
    \hline
         \textbf{Algorithm} &   \textbf{Setting} & \textbf{Projection} &
          \textbf{Regret Upper Bound}  \\
         \hline
         \begin{tabular}{c} $ \tt GLM \mhyphen UCB$ \\ \cite{filippi2010parametric}\end{tabular}
         &  \begin{tabular}{c} Stationary \\ GLM \end{tabular}
         & Non-convex
         & $\widetilde{\mathcal{O}}\left(
         	 {\color{black}\pmb{c_\mu^{-1}}}
         	\cdot d \cdot \sqrt{T}  \right)$ \\
         \hline
         \begin{tabular}{c} $ \tt LogUCB1 $\\ \cite{faury2020improved}\end{tabular}
         & \begin{tabular}{c} Stationary \\ Logistic\end{tabular}
         & Non-convex
         & $\widetilde{\mathcal{O}} \left(   {\color{black}\pmb{c_\mu^{-1/2}}}
         \cdot d \cdot \sqrt{T} \right)$
         \\
         \hline
          \begin{tabular}{c} $\tt  D\mhyphen GLUCB$ \\ \cite{russac2020algorithms} \end{tabular}
         & \begin{tabular}{c} Non-Stationary \\ GLM \end{tabular}
         & Non-convex
          & $\widetilde{\mathcal{O}}\left(  {\color{black}\pmb{c_\mu^{-1}}}
         \cdot d^{2/3} \cdot \Gamma_T^{1/3} \cdot T^{2/3}\right)$
         \\
         \hline
         \begin{tabular}{c} $\DGLM$ \\ \textbf{(this paper)} \end{tabular}
         & \begin{tabular}{c} Non-Stationary \\ GLM + SC + Ass.\hypersetup{linkcolor = black} \ref{ass:minimum_gap}
         \hypersetup{citecolor=blue}
         \end{tabular}
         & \textbf{\color{red} No projection}
          & $\widetilde{\mathcal{O}}\left(   {\color{red}\pmb{c_\mu^{-1/2}}}
         \cdot d \cdot \sqrt{\Gamma_T T} \right)$
         \\
         \hline
         \begin{tabular}{c} $\DGLM$ \\ \textbf{(this paper)} \end{tabular}
         & \begin{tabular}{c} Non-Stationary \\ GLM + SC \end{tabular}
         & \textbf{\color{red} No projection}
          & $\widetilde{\mathcal{O}}\left(   {\color{red}\pmb{c_\mu^{-1/3}}}
         \cdot d^{2/3} \cdot \Gamma_T^{1/3} \cdot T^{2/3} \right)$
         \\
         \hline
          \end{tabular}
    \caption{Comparison of regret guarantees for different algorithms in the GLM
     setting with respect to the degree of non-linearity $c_\mu$, the dimension $d$, the horizon $T$ and the number $\Gamma_T$ of abrupt changes. 
     In the table SC stands for
     self-concordant. \ifcolor { \color{purple} \else \fi Regret guarantees for $\SW$ are the same than for $\DGLM$.	
     \ifcolor     
     }
     \else \fi
     }
    \label{tab:regret_comparison}
\end{table*}

%%% Local Variables:
%%% TeX-master: "main.tex"
%%% End:

\section{ALGORITHM AND RESULTS}
% %Next, we state a concentration result with
%% time-dependent regularization parameters
%  %and increasing weights that is independent of $c_\mu$ and
%Next,  we fully characterize
%  the regret of $\DGLM$ even when the projection step is avoided.
\subsection{Algorithms}
In this section, we consider the abruptly changing environments defined in Section
\ref{section:1}. We propose two algorithms: $\DGLM$, which is based on discount factors, and
$\SW$ using a sliding window. 
Due to space limitation constraints, the pseudo-code of $\SW$ and the corresponding
theoretical results are reported in Appendix \ref{section:regret_SW_appendix}.
\ifcolor
}% but rely on similar ideas than for $\DGLM$.}
\else \fi
Associated with the weighed MLE defined in Equation~\eqref{eq:MLE_D}, define the
weighted design matrix as
\begin{equation}
 \label{eq:design_matrix_D_main}
 \VV_t = \sum_{s=1}^{t-1} \gamma^{t-1-s}
a_s a_s^\top  + \frac{\lambda}{c_\mu} \identity{d}
\;.
\end{equation}
The $\DGLM$ algorithm proceeds as follows. First, based on the
previous rewards and actions, $\hat{\theta}_t$ is computed. After
receiving the action set $\mathcal{A}_t$, the action $a_t$ is chosen
optimistically as the maximizer of the
current estimate $\mu(a^\top \hat{\theta}_t)$ of each arm's reward inflated by the confidence bonus
$\cm^{-1/2} \beta_T^\delta \lVert a \rVert_{\VV_t^{-1}}$. Finally, the
reward $r_{t+1}$ is received and the matrix $\VV_t$ is updated. The
expression of $\beta^\delta_T$ is a consequence of our novel concentration
result and is defined in Equation~\eqref{eq:beta_main}.
A pseudo-code of the algorithm is presented in
Algorithm~\ref{alg:D-GLM}.

There are two differences between $\DGLM$ and the algorithm
proposed in \citet{russac2020algorithms}. First,
we directly use $\hat\theta_t$ to make predictions about the arms'
performances,  whether it belongs to $\Theta$ or not. Second, the
exploration term scales as $\cm^{-1/2}$ (instead of $\cm^{-1}$), as in \cite{faury2020improved}. The latter has a direct impact on the regret-bound of $\DGLM$, to be stated below.

\begin{algorithm}[ht]
\caption{$\DGLM$}
   \label{alg:D-GLM}
\begin{algorithmic}
   \STATE {\bfseries Input:} Probability $\delta$, dimension $d$, regularization $\lambda$,
   upper bound for bandit parameters $S$, discount factor $\gamma$.
\STATE {\bfseries Initialize:} $\VV_0 = (\lambda/c_{\mu}) \identity{d}$, $\hat{\theta}_0 = 0_{\mathbb{R}^d}$.
   \FOR{$t=1$ {\bfseries to} $T$}
   \STATE Receive $\mathcal{A}_t$, compute $\hat{\theta}_t$ according to (\ref{eq:MLE_D})
   %\IF { $\hat{\theta}_t \in \Theta$}
   %\STATE $\beta_t^\delta = \beta_{t,1}^\delta$ with $\beta_{t,1}^\delta$ defined in Equation \eqref{eq:beta_1}
   %\ELSE
%   \STATE $\beta_t^\delta = \beta_{t,2}^\delta$ with $\beta_{t,2}^\delta$ defined in Equation \eqref{eq:beta_2}
 %  \ENDIF
   \STATE {\bfseries Play} $a_t = \argmax_{a \in \mathcal{A}_t} \mu(a^\top \hat{\theta}_t)  + \frac{\beta^\delta_T}{\sqrt{c_\mu}}
 \lVert a \rVert_{\VV_t^{-1}} $ with $\beta_{T}^\delta$ defined in Equation \eqref{eq:beta_main}
  \STATE {\bfseries Receive} reward $r_{t+1}$
  \STATE {\bfseries Update:} $\VV_{t+1} \leftarrow a_t a_t^{\top} + \gamma
  \VV_{t} + \frac{\lambda}{c_{\mu}} (1- \gamma) \identity{d}$
   \ENDFOR
\end{algorithmic}
%\label{alg:pc}
\end{algorithm}

\subsection{Regret Upper Bounds}
\label{sec:regret_upper_bound}
% \ifbulletpoint
% \begin{itemize}
% \item {\color{red}  Regret analysis without any projection step with the discount factors
%  model.
% }
% \end{itemize}
% \else
% \fi

We detail in this section the performance guarantees for $\DGLM$. Define
%\begin{align}
%\beta_{t,1}^\delta = 2 \km\sqrt{1 + 2S} ( \sqrt{\lambda}S + \rho_t^\delta)\;. \label{eq:beta_1}
%\end{align}
\begin{align}
 \label{eq:beta_main} 
\beta_{T}^\delta = \km \sqrt{\lambda} \left( 1 + \bar{S}  + \sqrt{\frac{1 + \bar{S}}{\lambda}}
\rho_T^\delta + \left( \frac{\rho_T^\delta}{\sqrt{\lambda}} \right)^2   \right)^{3/2 }
\end{align}
with
\begin{align}
\bar{S} = S + \frac{2S\km + m }{T \lambda (1- \gamma)}\;,
\end{align}
and where
\begin{align*}
    \begin{split}
        \rho_T^\delta &= \frac{\sqrt{\lambda}} {2 \rM} + \frac{2
\rM
  }{\sqrt{\lambda} }\log\left(\frac{T}{\delta}\right)
 + \frac{2 \rM}{\sqrt{\lambda}}
  d\log(2) \\
  & + \frac{d \rM
  }{ \sqrt{\lambda} }\log\left(1 + \frac{k_\mu(1- T^{-2})}{d \lambda (1- \gamma^2)}
  \right)
  \; .
    \end{split}
\end{align*}
The latter expression is a direct consequence of the concentration result presented in Theorem~\ref{thm:instantaneous_main}
below. The difference between $\bar{S}$ and $S$
is a bias term due to the non-stationarity.

%We conclude this section by a stronger asymptotic result pertaining to
%a more restricted class of problems. 
%Indeed, recent works such as
%\cite{chen2019new} obtain regret bounds that under the same assumption of
%non-stationarity scale as $\sqrt{T\Gamma_T}$. However these results
%correspond to the setting where one has a finite fixed set of feasible
%actions and the regret is defined with respect to the best
%policy in some finite class. 
%In contrast, Corollary~\ref{cor:zfoqiqdf}
%applies to the general setting where actions can change over time and
%the regret benchmark is the ground-truth of the environment. 
%
\ifcolor
{\color{purple}
\else \fi
Before stating our first theorem, we add an additional assumption on the
minimal gap. This assumption is discussed in Section \ref{sec:discussion}
and is only used in Theorem \ref{thm:fjqiqposejf}.
\ifcolor
}
\else
\fi
\begin{ass}
\label{ass:minimum_gap}
The reward gaps $\Delta_t = \min_{a \in \mathcal{A}_t, 
\mu( a^\top \theta^\star_t) < \mu(a_\star^\top \theta^\star_t)} \mu(a_\star^\top \theta^\star_t)
- \mu(a^\top \theta^\star_t)$ satisfies
$$
\forall t \leq T, \Delta_t \geq \Delta > 0\;.
$$
\end{ass}

%Under Assumption \ref{ass:minimum_gap}, one can obtain refined regret bounds.
%This result in particular applies to the case of a fixed finite set of
%actions. Under this assumption, we have the following theorem.

\begin{restatable}{thm}{thmproblem}%[Gap-dependent bound]
  \label{thm:fjqiqposejf}
  \ifappendix
Under Assumption \ref{ass:minimum_gap}, 
the regret of the $\DGLM$ algorithm is bounded for all $\gamma \in (1/2,1)$
with probability at least $1-\delta$ by 
\begin{align*}
    \begin{split}
        R_T &\leq C_1 \frac{\Gamma_T}{1- \gamma} + 
        C_2 \frac{1}{T(1-\gamma)^2 \Delta }
        +  
        C_3\frac{\beta^\delta_T \sqrt{dT}}{ \sqrt{c_\mu} \Delta} 
        \sqrt{T \log(1/\gamma) + \log \left(1 + \frac{1}{d \lambda (1- \gamma)} \right) }
        \\ 
        & + C_4 \frac{d (\beta_T^{\delta})^2}{\cm \Delta }
    \big( T \log(1/\gamma) + \log(1 + \frac{1}{d \lambda (1-\gamma)}) \big) \;,
    \end{split}
\end{align*}
where $C_1$, $C_2$, $C_3$, $C_4$ are universal constants independent of 
$\cm$, $\gamma$ with only logarithmic terms in $T$.

In particular, setting $\gamma = 1 - \frac{\sqrt{c_\mu \Gamma_T}} {d \sqrt{T}}$ and 
$\lambda = d\log(T)$ leads to
 $$
 R_T = \widetilde{\mathcal{O}}
 \big(\Delta^{-1} \cm^{-1/2} d \sqrt{ \Gamma_T T}\big)\;.
 $$
 \else
Under Assumption \ref{ass:minimum_gap}, 
the regret of the $\DGLM$ algorithm is bounded for all $\gamma \in (1/2,1)$
with probability at least $1-\delta$ by 
\begin{align*}
    \begin{split}
        R_T &\leq C_1 \frac{\Gamma_T}{1- \gamma} + 
        C_2 \frac{1}{T(1-\gamma)^2 \Delta }
        \\
        & +  
        C_3\frac{\beta^\delta_T \sqrt{dT}}{ \sqrt{c_\mu} \Delta} 
        \sqrt{T \log(1/\gamma) + \log \left(1 + \frac{1}{d \lambda (1- \gamma)} \right) }
        \\ 
        & + C_4 \frac{d (\beta_T^{\delta})^2}{\cm \Delta }
    \Big( T \log(1/\gamma) + \log\Big(1 + \frac{1}{d \lambda (1-\gamma)} \Big) \Big) \;,
    \end{split}
\end{align*}
where $C_1$, $C_2$, $C_3$, $C_4$ are universal constants independent of 
$\cm$, $\gamma$ with only logarithmic terms in $T$.

In particular, setting $\gamma = 1 - \frac{\sqrt{c_\mu \Gamma_T}} {d \sqrt{T}}$ 
and $\lambda = d \log(T)$ leads to
 $$
 R_T = \widetilde{\mathcal{O}}
 \big(\Delta^{-1} \cm^{-1/2} d \sqrt{ \Gamma_T T}\big)\;.
 $$
 \fi
\end{restatable}
\ifcolor
{\color{purple}
\else \fi

% \begin{rem}
There is a strong link between the cost of non-stationarity in the $K$-arm setting 
and the one observed in the more general GLB setting.
In the $K$-arm setting, any sub-optimal arm $i$ is played at most
$\mathcal{O}(\Delta_i^{-2}\log(T))$ times
(e.g \cite[Proposition 1.1]{munos2014bandits}),
whereas in any abruptly changing
environment, forgetting policies
play a sub-optimal arm $i$ at most
$\widetilde{\mathcal{O}}((\Delta_T(i))^{-2} \sqrt{\Gamma_T T})$ \citep{garivier2011upper}. 
$\Delta_T(i)$ is the minimum distance between the mean of the optimal arm and the mean of the suboptimal arm
$i$
over the entire time horizon.
For GLBs, in the stationary case \citet[Theorem 1]{filippi2010parametric} 
give a gap-dependent
bound on the regret scaling as $\mathcal{O}(\Delta^{-1} \cm^{-2} d^2 \log(T))$. 
Here, the bound of Theorem \ref{thm:fjqiqposejf} is of
order $\mathcal{O}(\Delta^{-1}  \cm^{-1/2} d \sqrt{\Gamma_T T} )$.
The reduced dependency in
$\cm$ in the latter bound is a direct consequence of the use of self-concordance.
Also note that when the inverse link function is the identity and the action 
set is the canonical
basis, our analysis recovers the results of \citet{garivier2011upper}.
%\end{rem}

We give an upper bound for the worst case regret of Algorithm~\ref{alg:D-GLM} in the following theorem; its proof is deferred to the appendix.
\ifcolor
}
\else
\fi
\begin{restatable}{thm}{thregretDnoprojmain}
\label{th:regret_D_no_proj_main}
\ifappendix
% in the appendix
The regret of the $\DGLM$ algorithm is bounded for all $\gamma \in (1/2,1)$
with probability at least $1-\delta$ by
\begin{equation*}
\begin{split}
R_T &\leq \frac{2 \log(T)}{1- \gamma}\Gamma_T + \frac{2 k_\mu( 2 S k_\mu + \rM)}{\lambda} \frac{1}{1- \gamma} \\
& \quad +
 \frac{2\beta^\delta_T}{\sqrt{c_\mu}} \sqrt{dT}  \sqrt{2\max \left(1,\frac{1}{\lambda}
 \right)}
\sqrt{T \log(1/\gamma) + \log \left(1 + \frac{1}{d \lambda (1- \gamma)} \right) } \;.
\end{split}
\end{equation*}
In particular, setting $\gamma = 
1-\left(\frac{\cm^{1/2}\Gamma_T}{dT}\right)^{2/3}$ and $\lambda = d \log(T)$ leads to
$$
R_T = \widetilde{\mathcal{O}}\big( \cm^{-1/3} d^{2/3} \Gamma_{T}^{1/3} T^{2/3}\big)\;.
$$
\else
% in the main paper
The regret of the $\DGLM$ algorithm is bounded for all $\gamma \in (1/2,1)$
with probability at least $1-\delta$ by
\begin{equation*}
\begin{split}
R_T &\leq C_1 \frac{\Gamma_T}{1- \gamma} \\
& +
 C_2\frac{\beta^\delta_T \sqrt{dT}}{\sqrt{c_\mu}} 
\sqrt{T \log\left(\frac{1}{\gamma}\right) + \log \left(1 + \frac{1}{d \lambda (1- \gamma)} \right) } \;,
\end{split}
\end{equation*}
where $C_1$ and $C_2$ are universal constants independent of $c_\mu$
and $\gamma$
with only logarithmic terms in $T$.

In particular, setting $\gamma = 
1-\left(\frac{\cm^{1/2}\Gamma_T}{dT}\right)^{2/3}$ 
and $\lambda = d \log(T)$
leads to
$$
R_T = \widetilde{\mathcal{O}}\big( \cm^{-1/3} d^{2/3} \Gamma_{T}^{1/3} T^{2/3}\big)\;.
$$
\fi
\end{restatable}
As in the linear case, this regret bound highlights the existence of
two mechanisms of different nature. The first term is due to
non-stationarity, the number of changes $\Gamma_T$ being multiplied by
$1/(1-\gamma)$, which is a rough measure of the forgetting time
induced by the exponential weights. The second term characterizes the rate at which the weighted MLE
$\hat\theta_t$ approaches $\theta_t^\star$. By balancing both terms, we
can characterize the asymptotic behavior of
the regret bound.

\ifcolor
{ \color{purple}
\else
\fi
In Theorem~\ref{th:regret_D_no_proj_main}, optimally tuning $\gamma$ yields the asymptotic worst case rate of 
$T^{2/3}$. This is similar to
the asymptotic rate achievable in the linear case with a different measure of non-stationarity \citep{russac2019weighted} and
the same dependency is attained with a sliding window for MDPs in abruptly changing environments
\citep{gajane2018sliding} and with restart factors \citep{auer2009near}. 
\ifcolor
}
\else \fi
\begin{rem}
  \label{rem:fizefzjapcd}
  The proof of Theorem~\ref{th:regret_D_no_proj_main} reveals that for
  rounds $t$ where $\hat{\theta}_t$ lies in $\Theta$, it is
  possible to obtain a (usually) tighter concentration result
  (depending on the values of $\lambda$ and $S$) by replacing
  $\beta_T^\delta$ with
  $\km\sqrt{1 + 2S} ( \sqrt{\lambda}S + \rho_T^\delta)$. This cannot
  be used to improve the result of
  Theorem~\ref{th:regret_D_no_proj_main}, as one doesn't know in
  advance for which rounds the condition will be satisfied, but
  this minor modification of Algorithm~\ref{alg:D-GLM} is most often advisable in practice.
  See Section~\ref{subrefined_explo} in Appendix for more details.
\end{rem}

%%% Local Variables:
%%% TeX-master: "main.tex"
%%% End:

\section{KEY ARGUMENTS}
In this section, we detail some key elements of our analysis. First,
we describe the concentration result in its most generic form.  Then,
we explain the main steps to derive the upper bound of the regret of
$\DGLM$.
\subsection{A Tail-Inequality for Self-Normalized Weighted Martingales}
To reduce the dependency in $c_\mu$, it is essential to take into
account the actual conditional variance of the generalized linear
model \citep{faury2020improved}. With exponentially increasing weights,
we also need time-dependent regularization parameters to avoid a
vanishing effect of the regularization
\citep{russac2019weighted}. Carefully combining these two elements
yields the following concentration result.

\begin{restatable}{thm}{instantaneousmain}
\label{thm:instantaneous_main}
Let $t$ be a fixed time instant. Let $\{\mathcal{F}_u\}_{u=1}^t$ be a
filtration. Let $ \{a_u\}_{u=1}^t$ be a stochastic process on $\mathbb{R}^d$
such that $a_u$ is $\mathcal{F}_u$ measurable and $\lVert a_u\rVert_2 \leq 1$. Let
$\{\epsilon_u\}_{u=2}^{t}$ be a martingale difference sequence such
that $\epsilon_{u+1}$ is $\mathcal{F}_{u+1}$ measurable. Assume that
the weights are non-decreasing, strictly positive and the time horizon
is known.  Furthermore, assume that conditionally on $\mathcal{F}_u$
we have $|\epsilon_{u+1}| \leq \rM$ a.s.  Let
$ \{\lambda_u \}_{u=1}^t$ be a deterministic sequence of
regularization terms and denote
$\sigma_t^2 = \mathbb{E} \left[ \epsilon_{t+1}^2 | \mathcal{F}_t
\right]$.

Let
$\widetilde{\HH}_t = \sum_{s=1}^{t-1} w_s^2 \sigma_s^2 a_s a_s^{\top}
+ \lambda_{t-1} \identity{d}$ \;and
\;$S_t = \sum_{s=1}^{t-1} w_s \epsilon_{s+1} a_s$, then for any
$\delta \in (0,1]$,
%\label{eq:weighted_deviation_instantaneous}
\begin{align*}
\begin{split}
%\mathbb{P} \left( 
\left\lVert S_{t}\right\rVert_{\widetilde{\HH}_{t}^{-1}
   } 
   &\geq \frac{\sqrt{\lambda_{t-1} }} {2 \rM w_{t-1}} + \frac{2 \rM
   w_{t-1}}{\sqrt{\lambda_{t-1}} }\log\left(\frac{\det(\widetilde{\HH}_{t})^{1/2}}{
   \delta  \lambda_t^{d/2}}\right) \\
   &
   + \frac{2 \rM w_{t-1}}{\sqrt{\lambda_{t-1} }} d\log(2)
    %\right) \leq \delta\;.
\end{split}
\end{align*}
with probability smaller than $\delta$.
\end{restatable}

\subsection{Upper Bounding the Regret of $\DGLM$}
In a non-stationary environment, each change in the parameter will
necessarily result in a number of rounds where the bias of the
weighted MLE estimator cannot be controlled. This gives rise to the
first term in the upper bound in Theorem
\ref{th:regret_D_no_proj_main}. To make this observation more
explicit, for $D\geq 1$, define
$\mathcal{T}(\gamma) = \{ 1 \leq t \leq T, \text{ such that }
\theta^{\star}_s = \theta^{\star}_t \text{ for } t- D \leq s \leq t-1
\}$ the set of time instants that are at least $D$ steps away from the
previous closest breakpoint.
Central in the analysis of weighted GLBs is the matrix
$$
\GG_t(\hat{\theta}_t, \theta^\star_t) =  \sum_{s=1}^{t-1} \gamma^{t-1-s}
\alpha(a_s, \hat{\theta}_t, \theta^\star_t)
	a_s a_s^\top + \lambda \identity{d},
	$$
where
$$
	\alpha(a_s, \hat{\theta}_t, \theta^\star_t)   =
	\int_{0}^1 \dot{\mu}(a_s^{\top}((1-v)\theta^\star_t + v\hat{\theta}_t))dv.
$$
As in the linear case, we define its analogue with squared 
exponential weights, 
$$
\wGG_t(\hat{\theta}_t, \theta^\star_t)  =  \sum_{s=1}^{t-1} \gamma^{2(t-1-s)}
\alpha(a_s, \hat{\theta}_t, \theta^\star_t)
	a_s a_s^\top + \lambda \identity{d}\;.$$
	
%by $\epsilon_{s+1} = r_{s+1} - \mu(a_s^\top \theta^\star_s)$.
We add the subscript $\tD$ to a quantity when the sum is for time
instants between $t-D$ and $t-1$. 
In this subsection, for space constraints, 
we will denote equivalently $\wGG_t(\hat{\theta}_t, \theta^\star_t)$ 
(resp. $\GG_t(\hat{\theta}_t, \theta^\star_t)$) by $\wGG_t$ (resp.
$\GG_t$).
As for
linear bandits, the exploration bonus is designed to mitigate the impact of prediction
errors.  We focus below on upper bounding the prediction error
in $\hat{\theta}_t$ defined as
$\Delta_t(a, \hat{\theta}_t) = | \mu(a^\top \hat{\theta}_t) -
\mu(a^\top \theta^\star_t)|$. The exact link between the regret and
this quantity is made explicit in Proposition
\ref{prop:delta_t_regret} in the appendix.
By defining $g_t(\theta) = \sum_{s=t-D}^{t-1} 
\gamma^{t-1-s} \mu(a_s^\top \theta) a_s
+ \lambda \theta$,
%and following the line of proof of \cite{faury2020improved},
when $t \in \mathcal{T}(\gamma)$ one can
upper bound the prediction error in $\hat{\theta}_t$.
%\begin{equation*}
%\begin{split}
%\Delta_t(a, \hat{\theta}_t) &\leq \frac{c \gamma^D} {1-\gamma} 
%+   k_\mu \Big( \underbrace{\lVert a \rVert_{\GG_t^{-1}(\hat{\theta}_t,
%\theta^\star_t)}}_{\circled{2} } 
% \\
%&
% \times 
% \underbrace{\lVert g_t(\hat{\theta}_t) - g_t(\theta^\star_t)
% + \lambda \theta^\star_t \rVert_{\wGG_{\tD}^{-1}(\hat{\theta}_t,
% \theta^\star_t) }}_{\circled{1}}\Big)
%  \;.
%\end{split}
%\end{equation*}
%%%%%
\begin{equation*}
\begin{split}
\Delta_t(a, \hat{\theta}_t) &\leq \frac{c \gamma^D} {1-\gamma} 
+   k_\mu  \underbrace{\lVert g_t(\hat{\theta}_t) - g_t(\theta^\star_t)
 \rVert_{\wGG_{\tD}^{-1}}}_{\circled{1}}
 \underbrace{\lVert a \rVert_{\GG_t^{-1}}}_{\circled{2} } 
\end{split}
\end{equation*}
%{\color{purple} Details can be found in Proposition \ref{prop:upper_delta_t_D_no_proj}
%of the  Appendix.}
The first term corresponds to the bias due to non-stationarity.
$\circled{1}$ is a measure of the deviation of $\hat{\theta}_t$ from $\theta^\star_t$
 adapted to the non-linear nature of the problem.
  Note that $g_t(\hat{\theta}_t) - g_t(\theta^\star_t)$ involves a martingale
  difference sequence (thanks to the optimality condition
  of the MLE) that can be controlled using Theorem
  \ref{thm:instantaneous_main}.  However, to bound $\circled{1}$ using
  Theorem \ref{thm:instantaneous_main} one needs to link
  the matrix $\wGG_{\tD}$ with $\wHH_{\tD}$ ,
 the self-concordance allows exactly to do this.

\paragraph{Self-Concordance} More precisely, the use of self-concordance
 offers a sharp relation (independent of $\cm$)
between the first derivative of the mean function evaluated at different points.
% between the mean function evaluated at $\hat{\theta}_t$ and
%$\theta^\star_t$, that only depends on the deviation from $\hat{\theta}_t$
%around $\theta^\star_t$. %the distance between the
%parameters.
Using Lemma \ref{lemma:self_concordance} reported in
Appendix~\ref{section:useful_results}, standard calculations yield:
\begin{equation}
\label{eq:SC_main_2}
\wGG_{\tD} \geq \big(1 + C + \frac{1}{\sqrt{\lambda}} \lVert g_t(\hat{\theta}_t) -
 g_t(\theta^\star_t)\rVert_{\wGG_{\tD}^{-1}} \big) \wHH_{\tD}
\end{equation}
%% ATTENTION BIEN DIRE QUE C'est avec theta^\star_t et \hat{\theta}_t
Note that Equation \eqref{eq:SC_main_2} involves the deviation
term that we want to control. Here, $C$ is a residual bias due
to the non-stationarity of the environment.
%{\color{purple}
%The use of self-concordance is also discussed in Section \ref{sec:discussion}
%}
%When $\hat{\theta}_t \in \Theta$, Eq.  \eqref{eq:SC_main_2} can be
%further lower bounded using
%$1 + |a_s (\hat{\theta}_t - \theta^\star_t)| \leq 1 + 2S$. Using this
%and summing for the different time instants, we obtain an easy link
%between $\wGG_{\tD}$ (resp. $\GG_t$) and $\wHH_{\tD}$ (resp. $\HH_t$).
%This specific bound can also be incorporated in the algorithm, as
%discussed in Remark~\ref{rem:fizefzjapcd}.

\paragraph{Better Characterization of the MLE}
By leveraging Equation \eqref{eq:SC_main_2} to bound the
deviation $g_t(\hat{\theta}_t) - g_t(\theta^\star_t)$
in the $\wGG_{\tD}^{-1}$-norm, one obtains an implicit equation.
Solving it leads to the following proposition.

%When $\hat{\theta}_t \notin \Theta$, the analysis is more
%involved. Using Eq. \eqref{eq:SC_main_2} makes it possible to control the
%deviation of $\hat{\theta}_t$ around $\theta^\star_t$.  By defining
%$\rho_t^\delta$ the high probability upper-bound from
%Theroem \ref{thm:instantaneous_main} with re-indexed time instants, we
%can prove the follwing.

\begin{restatable}{prop}{propdeviationmain}
\label{prop:deviation_main}
\ifappendix
For any $\delta \in (0,1]$, with probability higher than $1- \delta$, 
\begin{equation*}
\forall t \in \mathcal{T}(\gamma), \; 
 \lVert \gamma^{t-1} S_{\tD} \rVert_{\widetilde{\GG}_{\tD}^{-1}(\hat{\theta}_t, \theta_t^\star)}
\leq
\sqrt{1 + \bar{S}}   \rho_T^\delta +
\frac{1}{\sqrt{\lambda}} \left(\rho_T^\delta\right)^2  \;,
 \end{equation*}
where $\rho^\delta_T$ is defined in Equation \eqref{eq:rho_t_appendix}.
\else
 When $t \in \mathcal{T}(\gamma)$, the following holds,
\begin{equation*}
\lVert g_t(\hat{\theta}_t) - g_t(\theta^\star_t) \rVert_{\widetilde{\GG}_{\tD}^{-1}(\hat{\theta}_t, \theta_t^\star)}
\leq
\sqrt{1 + C}   \rho_T^\delta +
\frac{1}{\sqrt{\lambda}} \left(\rho_T^\delta\right)^2  \;,
 \end{equation*}
 where $C$ is a residual term due to non-stationarity.
 \fi
\end{restatable}

\begin{rem*}
In stark contrast with previously existing works
(see \cite[Proposition 1]{filippi2010parametric}),
deviations from the true parameter $\theta^\star_t$ are characterized
uniquely by the MLE (and not by its projected counterpart).
 This can be done whether $\hat{\theta}_t$ belongs to $\Theta$ or not and without any
 projection.
 This is not specific to the non-stationary nature of the
 problem
 but fundamentally relies on an improved analysis of the MLE.
 Similar guarantees can be obtained in any stationary environment.
See Section \ref{sec:discussion} for a more detailed comparison 
of the possible uses of the self-concordance property.
\end{rem*}
$\circled{1}$ can be  upper bounded using Proposition
\ref{prop:deviation_main}.
To upper bound $\circled{2}$ we use the following inequality.
\begin{equation}
% attention c'est bien \GG_{t}(\hat{\theta}_t, \theta_t^\star) dont on parle pour la ligne suivante 
\label{eq:SC_2_main_2}
\GG_{t} \geq \left( 1 + C + \frac{1}
{\sqrt{\lambda}}
\lVert
 g_t(\hat{\theta}_t) - g_t(\theta^\star_t)
\rVert_{\widetilde{\GG}_{\tD}^{-1}}\right)^{-1}
c_\mu \VV_t
\;.
\end{equation}
Combining Proposition \ref{prop:deviation_main} with Equation  \eqref{eq:SC_2_main_2}
gives the upper bound for $\circled{2}$.
Putting everything together, we obtain the form of $\beta_{T}^\delta$ given
in Equation~(\ref{eq:beta_main}).
The regret bound is then obtained by summing the exploration bonus for
the different time instants. Applying the so-called elliptical
lemma (see \cite[Chap. 19]{lattimore2019bandit}) and letting
$D =~ \log(T)/\log(1/\gamma)$ completes the proof.

\section{DISCUSSION}
\label{sec:discussion}
\ifcolor
{\color{purple}
\else \fi
\paragraph{Assumption on the Gaps.} 
%Here, we discuss the relevance and necessity of Assumption 
%\ref{ass:minimum_gap}.
Assumptions similar to our Assumption \ref{ass:minimum_gap} 
requiring a minimum gap are frequent in non-stationary bandits.
First, note that $\Delta$ is not required for the algorithm but only
for the theoretical analysis.
Second, similar assumptions can be found 
for $K$-arm bandits in several 
works to obtain the optimal $\widetilde{\mathcal{O}}(\sqrt{\Gamma_T T})$ regret bound. 
This is in particular the case for change-points detection methods: %only gives gap-dependent upper-bounds.
\cite[Corollary 1]{cao2018nearly} and
\cite[Corollary 4.3]{zhou2020near} 
is proved under an assumption on the minimal gap.
This remains true for forgetting strategies: the bound of \citet{garivier2011upper}
is gap-dependent,
\cite{trovo2020sliding} achieve a 
$\mathcal{O}(\Delta^{-1} \sqrt{T \Gamma_T})$ regret.
%and for cascading bandits, \cite{li2019cascading} obtain gap-dependent
%upper bounds as well.
More demanding, the LM-DSEE and SW-UCB\# algorithms from \cite{wei2018abruptly}
require the minimum gap as an input of the algorithm.
Generally speaking, none of those works provide an analysis when the minimum gap can depend on the time horizon $T$ and when
the mean of different arms can be arbitrarily close. We suspect that forgetting policies would obtain
a $\mathcal{O}(\Gamma_T^{1/3} T^{2/3})$ worst case dependency as in Theorem \ref{th:regret_D_no_proj_main}
and that changepoint detection methods are likely to fail in such a case.

\paragraph{Tightness of the Bound.}
For problems with a finite number of actions, \cite{auer2018adaptively} have developed 
an algorithm that does not require the knowledge of the number of breakpoints nor
assumption on the gaps.
This was extended to the $K$-arm setting by \cite{auer2019adaptively} and to the more
general contextual bandits by \cite{chen2019new}. Both works 
(\cite{auer2019adaptively, chen2019new})
achieve the optimal
$\widetilde{\mathcal{O}}(\sqrt{\Gamma_T T})$ regret bound.
Yet, their analysis 
does not apply to the GLB framework.
Furthermore, both works rely on replaying phases that are incompatible
with time-dependent action sets as considered here.
Additionally, in \citep{chen2019new} the regret is defined with respect to the best
policy in some finite class, whereas our results
apply to the general setting where actions can change over time and
the regret benchmark is the ground-truth of the environment.
The best lower-bound for forgetting policies in abruptly changing environments with time-dependent action sets remains unknown. 
While it is known that forgetting policies are minimax optimal when non-stationarity is measured through the so-called variational budget (see \cite{cheung2019learning, russac2019weighted}), whether such methods are
optimal in abruptly changing environments is unclear. 
Nonetheless, 
the bound obtained by \cite{garivier2011upper} in the $K$-arm setting yields a worst case regret bound that can be shown to be of order 
$\mathcal{O}( \Gamma_T^{1/3} T^{2/3})$ (see Appendix \ref{sec:worst_case_K}).

\paragraph{Knowledge of $\Gamma_T$}
Optimizing the choice of the forgetting parameter $\gamma$ (w.r.t. the regret bound) requires the knowledge of $\Gamma_T$. The Bandit over
Bandit (BOB) framework introduced by \cite{cheung2019learning} can be used
to circumvent this requirement. When the assumption \ref{ass:minimum_gap} is
satisfied, following the proof from \cite{cheung2019hedging} one would obtain a regret bound of order
$\widetilde{\mathcal{O}}(\Delta^{-1} d \cm^{-1/2} \sqrt{T \max(\Gamma_T,
T^{1/2})})$ (see \cite[Remark 2]{auer2019adaptively}). 
Similarly, in the absence of Assumption \ref{ass:minimum_gap} an upper bound of order
$\widetilde{\mathcal{O}}(\cm^{-1/3} d^{2/3} T^{2/3} 
\max(\Gamma_T, d^{-1/2} T^{1/4})^{1/3})$ can be achieved (see \cite[Theorem 4]{zhao2020simple}).

\paragraph{Self-Concordance} The analysis of \cite{faury2020improved} does not use self-concordance to its fullest. We present an
improved analysis valid in any stationary time frame, proving that a better treatment of the self-concordance
removes the need for the inconvenient projection. 
Informally, the self-concordance links $\mu(x^\top \hat{\theta}_t)$ to $
\mu(x^\top \theta^\star)$ without resorting to global bounds on $\dot\mu$ (e.g 
% rajouter un through ?    
    $k_\mu$ and $c_\mu$). In \cite{faury2020improved}, this takes the form of a Taylor-like expansion:
    \begin{align*}
       \mu(x^\top \theta_t) \leq \mu(x^\top \theta^\star) + \frac{\vert x^\top (\theta^\star-\theta_t)
       \vert}{1+2S}\dot\mu(x^\top \theta^\star)\;,
    \end{align*}
where $\theta_t$ is a projected version of $\hat\theta_t$ in $\Theta$. The
denominator of the r.h.s. is reminiscent of this projection step. We show here that a
finer analysis yields the following, more implicit but powerful bound:
% sur du powerful ici ?
%
    \begin{align*}
       \mu(x^\top \hat\theta_t) \leq \mu(x^\top \theta^\star) + \frac{\vert x^\top (\theta^\star-\hat
       \theta_t)\vert}{1+\vert x^\top (\theta^\star-\hat\theta_t)\vert}\dot\mu(x^\top \theta^\star)
       \;.
    \end{align*}
\begin{figure*}[hbt]
%\begin{figure*}[H]
\begin{subfigure}[t]{0.49\linewidth}
\centering
\includegraphics[width=\linewidth]{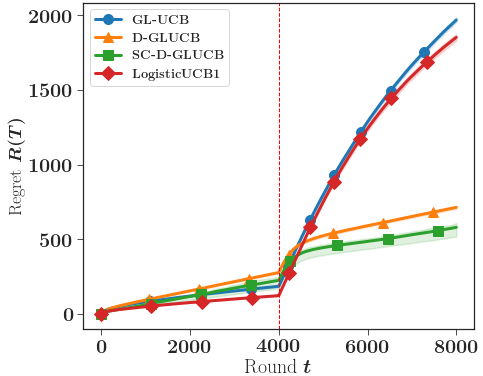}
\caption{$\cm^{-1} = 400$}
\label{fig:1}
\end{subfigure}
\hfill
\begin{subfigure}[t]{0.49\linewidth}
\centering
\includegraphics[width=\linewidth]{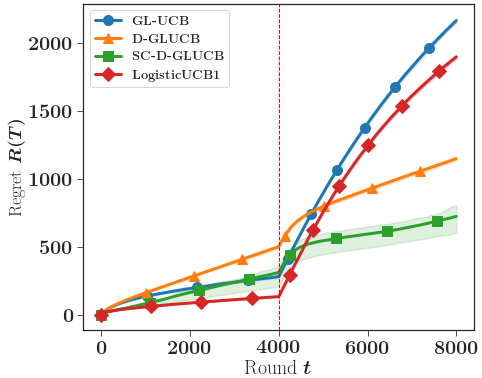}
\caption{$\cm^{-1} = 1000$}
\label{fig:2}
\end{subfigure}
\caption{Regret of the different algorithms in a 2D abruptly changing environment averaged on $200$ independent experiments and the $25\%$ associated quantiles.}
\label{fig:exps}
\end{figure*}

Note that when $\hat\theta_t\in\Theta$ (i.e there is no need for a projection), our
bound implies the one of \cite{faury2020improved}. The kind of relationship displayed in the above equation allows us
to derive a tail inequality for the deviation from $\hat\theta_t$ to $\theta^\star$
without projecting $\hat\theta_t$, 
by solving an implicit equation. 
We believe that this
new approach is of interest in other settings involving 
self-concordant GLBs. 
The self-concordance assumption (Assumption \ref{ass:SC})
is not particularly restrictive and goes beyond logistic
functions. 
Under the classical Assumption~\ref{ass:bounded_actions}
(i.e. bounded features) all GLMs are self-concordant 
(cf. Sec. 2 of \cite{JMLR:v15:bach14a}) 
with constants that depend on the link function.

%\paragraph{Stationary environments.} Our analysis and the ability to remove the
%projection step in \ref{alg:D-GLM} is valid independently of the non-stationarity of the environment. Obtaining this only requires the link function to be 
%self-concordant. Again, obtaining this is more attractive in abruptly changing 
%environments where a burning phase as in \cite{li2017provably} can not be used anymore to ensure $\hat{\theta}_t \in \Theta$.
\ifcolor
}
\else \fi
%%% Local Variables:
%%% TeX-master: "main.tex"
%%% End:

\section{EXPERIMENTS}

In this section, we illustrate the empirical performance of $\DGLM$ in
a simulated, abruptly changing environment with a logistic link function
$ \mu(x) = 1/(1+\exp(-x))$. In this two-dimensional problem, there is a
switch in the reward distribution at $t=4000$ (red dashed line on
Figure~\ref{fig:exps}).  

$\DGLM$ (Algorithm \ref{alg:D-GLM}) is compared with
$\tt GLM \mhyphen UCB$ from \cite{filippi2010parametric},
$\tt LogUCB1$ from \cite{faury2020improved} and with
$\tt D\mhyphen GLUCB$ from \cite{russac2020algorithms}.  $\DGLM$
(resp. $\tt D\mhyphen GLUCB$) is related with $\tt LogUCB1$
(resp. $\tt GLM \mhyphen UCB$) in the sense that the exploration terms
have the same scaling but the former incorporate the exponential
weights making it possible to adapt to changes. The average regret 
of the different policies together with their
central $50\%$ quantiles, averaged on 200 independent runs, are reported
in Figure~\ref{fig:exps} for two different parameter values.

%In both experiments there is a switch in the reward distribution at time $t = 4000$. 
In Fig.~\ref{fig:1}, $\theta^\star$ starts on the
circle of radius $S=6$ (corresponding to
$\cm^{-1} = \exp(S) \approx 400 $) with an angle of $2\pi/3$ and jumps
at $t=4000$ to an angle of $4\pi/3$. The experiment reported on
Fig.~\ref{fig:2} is identical with a radius $S=7$ corresponding to a
$\cm^{-1} \approx 1000$. As previously discussed, using such values of
$S$ is required in situation where the actions return binary rewards
with expected values in the range $10^{-3}$ -- $10^{-2}$, which is
typically the case in web advertising or recommendation applications.

For both experiments, at every time steps, $50$ randomly generated
actions in the unit circle are proposed to the learner. For $\DGLM$
and $\tt D\mhyphen GLUCB$ the asymptotically optimal choice of the
discount factors is used: $\gamma = 1 - (\Gamma_T/(d \times T))^{2/3}$
with $d=2$, $\Gamma_T = 2$ and $T= 8000$. To speed up the learning
that is hard with those values of $\cm$, all the algorithms have their
exploration bonus divided by 5.
% dans un environnement stationnaire c'est mieux de prendre en compte toutes les précédents action, rewards.

As expected, the algorithms tuned for non stationary situations
($\DGLM$, $\tt D\mhyphen GLUCB$) perform worse than their stationary
counterparts ($\tt LogUCB1$ and $\tt GLM \mhyphen UCB$) during the
first stationary phase. More precisely, with the choice made for
$\gamma$ the estimation of $\hat{\theta}_t$ for algorithms that use
exponential weights is roughly based on the
$1/(1- \gamma) \approx 400$ most recent observations.  In contrast,
$\tt LogUCB1$ and $\tt GLM \mhyphen UCB$ use all the observations from
the start to compute the MLE, which eventually leads to a more precise
estimation. Right after the change,
the bias caused by the non-stationarity results in
a significant increase in regret. Unweighted algorithms
are affected much more deeply by this phenomenon that will eventually
cause large losses in performance due to the persistence of obsolete
information.

The theoretical analysis of Section~\ref{sec:regret_upper_bound}
suggests that the advantage of $\DGLM$ is all the more significant in
strongly non-linear (large $\cm^{-1}$) non-stationary
environments. This is obvious in Figure~\ref{fig:exps}, particularly
when comparing Fig.~\ref{fig:1} and Fig.~\ref{fig:2}, which differ by
the range on which the logistic function is used for making reward
predictions. Note that, on average, for these two simulated scenarios
the fact that the MLE $\hat{\theta}_t$ does not belong to $\Theta$
happens for several hundred of rounds. All the algorithms except
$\DGLM$ would require non convex projection steps at these instants,
or equivalently, one should inflate $S$ (and thus $\cm^{-1}$) to ensure the
compliance of these algorithms with the associated theory. In
producing Figure~\ref{fig:exps}, this projection step was simply
bypassed, which provides an optimistic evaluation of the performance
of the competitors of $\DGLM$. Interestingly, the observation that the
dispersion of performance of $\DGLM$ is slightly higher than that of
$\tt D\mhyphen GLUCB$ can be traced back to the use of
Remark~\ref{rem:fizefzjapcd} in these simulations: $\DGLM$ adapts to
the events $\{\hat{\theta}_t\notin\Theta\}$ (rather than pretending that
these did not happen) and thus its performance is made somewhat
dependent on the actual occurrence of these events.

%%% Local Variables:
%%% TeX-master: "main.tex"
%%% End:

%\input{conclusion.tex}

%\subsubsection*{References}
%\clearpage
%\newpage

\bibliography{bib.bib}
\bibliographystyle{plainnat}

\clearpage
\newpage
\onecolumn
\appendix
\aistatstitle{Self-Concordant Analysis of Generalized Linear Bandits with
Forgetting: Supplementary Material}
\vspace{-5cm}
%\vspace{-15cm}
\appendixtrue
The Appendix is structured as follows.
In Section \ref{section:concentration_weights}, our new concentration
result for
self-normalized weighted martingales with time dependent
 regularization parameters is presented. In Section
\ref{subsec:application_weighted_n_s}, similar concentration results
 are established when a sliding window is used.
Section \ref{section:regret_D_appendix} studies the regret with discount factors
through our improved characterization of the MLE.
Section \ref{section:regret_SW_appendix} gives similar results with a sliding
 window.
Section \ref{section:useful_results} gathers some technical results,
 in particular the main properties resulting from the self-concordance assumption.
 Finally in Section \ref{sec:worst_case_K}, a worst case bound for a sliding window policy 
 in the $K$-arm setting is presented.

\section{TAIL-INEQUALITY FOR SELF-NORMALIZED WEIGHTED MARTINGALES}
\label{section:concentration_weights}
While keeping in mind our objective of obtaining a deviation inequality with 
exponentially increasing weights, we give more generic
 results under two assumptions on the weights.

\begin{ass}
\label{ass_T}
The time horizon $T$ is known in advance.
\end{ass}

\begin{ass}
\label{ass_increasing_weights}
The weights are deterministic, strictly positive and non-decreasing, i.e,
$$
\forall 1 \leq t \leq T, 0 < w_1 \leq w_t \leq w_{t+1} \leq w_T\;.
$$
\end{ass}

We recall the statement of the corresponding concentration result.

\instantaneousmain*

Theorem \ref{thm:instantaneous_main} is a
non-trivial extension of 
\citet[Theorem 1]{faury2020improved} 
allowing for the use of time-dependent regularization 
parameters and weights.
We now state several lemmas that are useful for establishing Theorem
\ref{thm:instantaneous_main}. 

\subsection{Useful Lemmas}
As a first step \textbf{we fix a time instant $t$}. Let $M_u^t(\xi)$ for $\xi \in
 \mathbb{R}^d$ and $1 \leq u \leq t$ be defined as
\begin{equation}
\label{eq:M_u}
M_u^t(\xi) = \exp\left(\frac{1}{\rM w_{t-1}} \, \xi^{\top} S_u - \frac{1}{ \rM^2 w^{2}_{t-1}}
 \xi^{\top} \wHH_u(0) \xi\right)\;,
\end{equation}
with  $S_u = \sum_{s=1}^{u-1} w_s \epsilon_{s+1} a_s $ and
$\wHH_u(0) = \sum_{s=1}^{u-1} w_s^2 \sigma_s^2
a_s a_s^{\top}$ where $\sigma_s^2 = \mathbb{E}
[\epsilon_{s+1}^2|
\mathcal{F}_s]$.

We prefer the notation $M_u^t$ to $M_u$ to clearly indicate the dependency on the
 weight $w_{t-1}$. When $u=t$, we prefer
the notation $M_t$ to $M_t^t$. For the entire appendix, we use the notation
$\mathcal{B}_2(d) = \{ a \in \mathbb{R}^d, \lVert a \rVert_2 \leq 1\}$.

\begin{lemma}
\label{lemma_instantaneous}
For all $\xi \in \mathcal{B}_2(d)$ and $2 \leq u \leq t$, under Assumption
\ref{ass_T} and \ref{ass_increasing_weights}, we have
$$
\mathbb{E} \left[ M_u^t(\xi) |\mathcal{F}_{u-1} \right] \leq M_{u-1}^t(\xi) \;. \quad
 \text{a.s}
$$
\end{lemma}
\begin{proof}
\begin{equation*}
\begin{split}
\mathbb{E} \left[ M_{u}^t(\xi) | \mathcal{F}_{u-1}\right] &= M_{u-1}^t(\xi) \exp\left(-
\frac{1}{\rM^2 w_{t-1}^2} \xi^{\top} w_{u-1}^2
 \sigma_{u-1}^2 a_{u-1} a_{u-1}^{\top} \xi\right) \\
 & \quad \times \mathbb{E} \left[ \exp\left(\frac{1}{\rM w_{t-1}} \xi^{\top} w_{u-1}
  \epsilon_{u} a_{u-1}
 \right) | \mathcal{F}_{u-1}\right]\;.
 \end{split}
\end{equation*}
The equality holds because $a_{u-1}$ is $\mathcal{F}_{u-1}$ measurable and $
\epsilon_{u-1}$ is $\mathcal{F}_{u-1}$
 measurable.
With $\tilde{\epsilon}_{u}= \epsilon_{u}/ \rM $ and $v =  \frac{w_{u-1}}{w_{t-1}}
 \xi^{\top} a_{u-1}$, the conditions of Lemma
\ref{lemma:upper_bound} (stated below) are met and we have,
\begin{equation*}
\mathbb{E} \left[ \exp\left(\frac{1}{\rM w_{t-1}} \xi^{\top} w_{u-1} \epsilon_{u} a_{u-1}
 \right) | \mathcal{F}_{u-1}\right] = \mathbb{E} \left[ \exp(v \tilde{\epsilon}_{u} )|
  \mathcal{F}_{u-1}\right]
 \leq 1 + \frac{v^2}{\rM^2} \sigma_{u-1}^2\;.
\end{equation*}
$|v| \leq 1$ holds because of Assumption \ref{ass_increasing_weights} and both  $\xi$
 and $a_{u-1} \in
 \mathcal{B}_2(d)$.
 Therefore,
 \begin{equation*}
 \begin{split}
 \mathbb{E} \left[ M_{u}^t(\xi) | \mathcal{F}_{u-1}\right] &\leq M_{u-1}^t(\xi) \exp\left(
 -\frac{1}{\rM^2 w_{t-1}^2} \xi^{\top} w_{u-1}^2
 \sigma_{u-1}^2 a_{u-1} a_{u-1}^{\top} \xi\right)\\
 & \quad \times \left( 1 + \frac{w_{u-1}^2}{\rM^2 w_{t-1}^2}  \sigma^2_{u-1} \xi^{\top}
  a_{u-1} a_{u-1}^{\top} \xi \right) \\
 &\leq M_{u-1}^t(\xi) \quad \text{(a.s)} \;,
 \end{split}
 \end{equation*}
 where the last inequality uses $1 + x  \leq \exp(x)$.
\end{proof}
Hence, for all $ 1\leq u \leq t $
and $\xi \in \mathcal{B}_2(d)$,
$\mathbb{E}\left[M_t(\xi)
\right] \leq  \mathbb{E}\left[M_u^t(\xi)
\right] \leq \mathbb{E}\left[M_1^t(\xi)
\right] = 1$.

For $ 1 \leq u \leq t$ we define,
\begin{equation}
\label{eq:bar_M_u}
\bar{M}_u^t = \int_{\xi} M_u^t(\xi) d h_u( \xi) \;.
\end{equation}
Here, $h_u$ is the density of an isotropic normal distribution of precision $\frac{2
 \lambda_{u-1} }{\rM^2 w_{t-1}^2}$ truncated on $\mathcal{B}_2(d)$. We will denote
  $N(h_u)$ its normalization constant.
\begin{lemma}
\label{lemma_int_instantaneous}
Let $t$ be a fixed time instant, for all $1 \leq u \leq t$, under assumptions \ref{ass_T}
 and \ref{ass_increasing_weights},
with $\{h_u\}_{u=1}^t$ the density of an isotropic normal distribution of precision $
\frac{2
\lambda_{u-1}}{\rM^2 w_{t-1}^2}$ truncated on $\mathcal{B}_2(d)$ we have,
$$
\mathbb{E}\left[ \bar{M}_u^t \right] \leq 1 \;.
$$
\end{lemma}
\begin{proof}
\begin{align*}
    \mathbb{E} \left[\bar{M}_u^t \right] &= \int_{\Omega}  \bar{M}_u^t  d \mathbb{P}(w)
    = \int_{\Omega} \left( \int_{\mathbb{R}^d}  M_u^t(\xi) d h_u(\xi) \right) d\mathbb{P}
    (w) \\
    &\leq \int_{\mathbb{R}^d} \left( \int_{\Omega} M_{u}^t(\xi)  d\mathbb{P}(w) \right)d
     h_u(\xi) \quad \text{(Fubini)} \\
    &\leq \int_{\mathbb{R}^d} \left( \int_{\Omega}  1  d\mathbb{P}(w) \right) d h_u(\xi)
     \quad \text{(Lemma \ref{lemma_instantaneous} + $h_u$ defined on $\mathcal{B}
     _2(d)$)} \\
    &\leq \int_{\mathbb{R}^d} d h_u(\xi) = 1\;. \quad \text{($h_u$ is a probability density
     function)}
\end{align*}
\end{proof}

\begin{rem}
Allowing time-dependent regularization parameters  is
essential in our analysis to avoid the vanishing effect
of the regularization with exponentially increasing weights
for example.
This is a fundamental difference with the deviation result provided in
 \citet{faury2020improved}. Furthermore, allowing
 the regularization parameters to be time-dependent comes at a cost here, we loose
 the property
$\mathbb{E}\left[ \bar{M}_u^t  | \mathcal{F}_{u-1} \right] \leq  \bar{M}_{u-1}^t $ that
 would hold with a
fixed regularization parameter (as in \cite{faury2020improved}).
In the linear bandit setting, this issue was discussed in Lemma 2 in
 \citet{russac2019weighted}.
  \end{rem}
In particular, applying Lemma \ref{lemma_int_instantaneous} for $u=t$ gives,
\begin{equation}
\label{eq_bar_t}
\mathbb{E}\left[ \bar{M}_t \right] = \mathbb{E}\left[ \bar{M}_t^t \right] \leq 1 \;.
\end{equation}
\begin{lemma}[Lemma 7 of \citet{faury2020improved}]
\label{lemma:upper_bound}
Let $\varepsilon$ be a centered random variable of variance $\sigma^2$ and such
 that $| \varepsilon | \leq 1$ almost surely. Then for all $v \in [-1,1]$,
$$
\mathbb{E} \left[ \exp(v \varepsilon) \right] \leq 1 + v^2 \sigma^2\;.
$$
\end{lemma}

\begin{rem}
\label{rem:instantaneous}
We stress out that $v \in [-1,1]$ is required for Lemma \ref{lemma:upper_bound} to
 hold. It has
 strong consequences
 in our setting with the weights as the normalization $1/w_{t-1}
 $ and $1/w_{t-1}^2$ in the definition of
  $M_u^t\;$ are needed to ensure that $v = (w_{u-1}/w_{t-1}) \xi^{\top}
   a_u$ that
   appears in the proof of Lemma
  \ref{lemma_instantaneous} will be smaller than 1.
 As
a consequence,
 the stopping trick presented in
   \cite{abbasi2011improved} can not be applied
  to $\bar{M}_u^t$ because of its dependency on $t$. For this reason, the deviation
    result presented in Theorem \ref{thm:instantaneous_main} is only valid for a fixed time
     instant $t$. To obtain a deviation
     result
  on the entire trajectory an union bound is required. %and the term $\log(1/\delta)$ is
%   replaced with $\log(T/\delta)$ in the
%   high  probability upper-bound.
\end{rem}
\subsection{Proof of Theorem \ref{thm:instantaneous_main}}
\label{subsection:proof_th}

The proof of this theorem follows the line of proof of \cite{faury2020improved}. The
 main differences are the time-dependent
regularization parameters and the presence of weights.
We recall that in Equation \eqref{eq:bar_M_u} $h_t$ is the density of an isotropic
 normal distribution of precision $\frac{2
 \lambda_{t-1}}{\rM^2 w_{t-1}^2}$ truncated on $\mathcal{B}_2(d)$ and denote
  $N(h_t)$ its normalization constant.

The following holds,
\begin{align}
\bar{M}_t = \frac{1}{N(h_t)} \int_{\mathbb{R}^d}  \mathds{1}\left[  \xi \in \mathcal{B}
_2(d)\right]
\exp\left( \frac{1}{\rM w_{t-1}} \xi^{\top} S_t - \frac{1}{\rM^2 w_{t-1}^2} \xi^{\top}
 \wHH_t \xi\right) d \xi \;.
\end{align}

Let $f_t : \mathbb{R}^d \mapsto \mathbb{R}$ be defined as $f_t(\xi) = \frac{1}{\rM
 w_{t-1}} \xi^{\top} S_t - \frac{1}{\rM^2
 w_{t-1}^2}
 \xi^{\top} \wHH_t \xi$. As a quadratic function, $f_t$ can be rewritten
  for $\xi^{\star}  = \argmax_{ \lVert \xi
  \rVert_2 \leq 1/2} f_t(\xi)$,
 $$
 f_t(\xi) = f_t(\xi^{\star}) + \nabla f_t(\xi^{\star})^{\top} (\xi- \xi^{\star}) + \frac{1}{2} (\xi -
 \xi^{\star})^{\top} \nabla^2 f_t (\xi^{\star}) (\xi - \xi^{\star})\;.
 $$
 Using  $ \forall \xi  \in \mathcal{B}_2(d), \, \,\nabla^2 f_t (\xi) = -\frac{2}{\rM^2 w_{t-1}
 ^2} \wHH_t$,
 \begin{align*}
 \bar{M}_t &= \frac{e^{f_t(\xi^{\star})}}{N(h_t)} \int_{\mathbb{R}^d}
\mathds{1} \left[ \lVert \xi \rVert_2 \leq 1 \right]\exp\left( \nabla f_t(\xi^{\star})^{\top}
(\xi - \xi^{\star}) - \frac{1}{\rM^2 w_{t-1}^2} \lVert \xi - \xi^{\star}
 \rVert_{\wHH_t}
^2
\right) d \xi \\
& = \frac{e^{f_t(\xi^{\star})}}{N(h_t)} \int_{\mathbb{R}^d}
\mathds{1} \left[ \lVert \xi  + \xi^{\star} \rVert_2 \leq 1 \right]\exp\left( \nabla
 f_t(\xi^{\star})^{\top}
\xi - \frac{1}{ \rM^2 w_{t-1}^2} \lVert \xi \rVert_{\wHH_t}
^2
\right) d \xi \\
& \geq \frac{e^{f_t(\xi^{\star})}}{N(h_t)} \int_{\mathbb{R}^d}
\mathds{1} \left[ \lVert \xi \rVert_2 \leq 1/2 \right]\exp\left( \nabla f_t(\xi^{\star})^{\top}
\xi - \frac{1}{\rM^2 w_{t-1}^2} \lVert \xi \rVert_{\wHH_t}
^2
\right) d \xi  \\
& \geq \frac{e^{f_t(\xi^{\star})} N(g_t)}{N(h_t)} \mathbb{E}_{\xi \sim g_t} \left[
 \exp\left(\nabla f_t(\xi^{\star})^{\top} \xi \right) \right] \;.
 \end{align*}

The second equality is obtained after a change of variable $\xi \mapsto \xi - \xi^{\star}$. In the
 last inequality, $g_t$ is the density of a $d$-dimensional normal
distribution with precision matrix $\frac{2}{\rM^2 w_{t-1}^2} \wHH_t$ truncated
on $\{ a \in \mathbb{R}^d, \lVert a \rVert_2 \leq 1/2\}$.
\begin{align*}
 \bar{M}_t & \geq \frac{e^{f_t(\xi^{\star})} N(g_t)}{N(h_t)} \exp\left( \mathbb{E}_{\xi \sim
  g_t} \left[
 \nabla f_t(\xi^{\star})^{\top} \xi \right]   \right)\;. \quad \textnormal{(Jensen's
  inequality)}
\end{align*}
$g_t$ is symmetric which implies $\mathbb{E}_{\xi \sim g_t} \left[\xi \right] = 0$.
Hence,
\begin{equation}
\label{eq:bar_M_t_lower}
 \bar{M}_t \geq \frac{e^{f_t(\xi^{\star})} N(g_t)}{N(h_t)} \;.
\end{equation}
Therefore,
\begin{align*}
\delta &\geq \mathbb{P} \left(\bar{M}_t \geq \frac{1}{\delta} \right) \quad
\text{(Equation \eqref{eq_bar_t} + Markov's Inequality)}  \\
&\geq  \mathbb{P} \left( f_t(\xi^{\star}) \geq
\log \left(\frac{1}{\delta } \right) +
\log \left(\frac{N(h_{t})}{N(g_{t})} \right) \right) \quad \textnormal{(Equation $
\eqref{eq:bar_M_t_lower}$)} \\
&=  \mathbb{P} \left( \max_{ \lVert \xi \rVert_2 \leq
1/2} f_t(\xi) \geq
\log \left(\frac{1}{\delta } \right) +
\log \left(\frac{N(h_{t})}{N(g_{t})} \right) \right) \\
&\geq  \mathbb{P} \left(f_t(\xi_0) \geq
\log \left(\frac{1}{\delta } \right) +
\log \left(\frac{N(h_{t})}{N(g_{t})} \right) \right)\;.
\end{align*}
In the last inequality $\xi_0$ is defined as $\xi_0 = \frac{\sqrt{\lambda_t}}{2} \frac{
\wHH_t^{-1} S_t}{\lVert S_t
\rVert_{\wHH_t^{-1}}}$, such that $\lVert \xi_0
\rVert_2 \leq 1/2$ holds. This can be seen by using $
\wHH_t  \geq \lambda_{t-1} \identity{d}$.
We also have,
$$
f_t(\xi_0) = \frac{1}{\rM w_{t-1}} \xi_0^{\top} S_t - \frac{1}{\rM^2 w_{t-1}^2} \xi_0^{\top}
\wHH_t \xi_0 = \frac{\sqrt{\lambda_{t-1} }}{2 \rM w_{t-1}}\lVert S_t
\rVert_{\wHH_t^{-1}} - \frac{\lambda_{t-1}}{4 \rM^2 w_{t-1}^2}\;.
$$
Therefore,
\begin{equation}
\mathbb{P} \left( \lVert S_t \rVert_{\wHH_t^{-1}} \geq
\frac{\sqrt{\lambda_{t-1} }}{2 \rM w_{t-1}} +  \frac{2 \rM w_{t-1}}{\sqrt{\lambda_{t-1}}}
\log(1
/\delta) +
\frac{2 \rM w_{t-1}}{\sqrt{\lambda_{t-1}}}\log \left(\frac{N(h_{t})}{N(g_{t})} \right)
\right) \leq \delta\;.
\end{equation}

We conclude using Proposition \ref{prop:upper_h_t_g_t}.
%The following proposition is an extension of Lemma 6 of \citet{faury2020improved} in
% the weighted case with time-
%dependent regularization parameters.
\begin{prop}
\label{prop:upper_h_t_g_t}
Let $h_t$ be the density of a $d$-dimensional isotropic normal distribution of
 precision $\frac{2 \lambda_{t-1}}{\rM^2 w_{t-1}^2}$ truncated on $\mathcal{B}_2(d)$.
Let $g_t$ be the density of a $d$-dimensional normal distribution with
 precision matrix $\frac{2}{\rM^2 w_{t-1}^2} \wHH_t$ truncated on $\{a \in
  \mathbb{R}^d, \lVert a \rVert_2 \leq 1/2 \}$.
The following inequality holds,
\begin{equation}
\log \left( \frac{N(h_{t})}{N(g_{t})}\right) \leq \log \left(
 \frac{\det(\wHH_t)}{
 \lambda_{t-1}^{d/2}}
\right) + d\log(2) \;.
\end{equation}
\end{prop}

\begin{proof}
\begin{align*}
N(h_t) &= \int_{\mathbb{R}^d} \mathds{1} \left[ \lVert \xi \rVert_2 \leq 1
\right] \exp \left(- \frac{1}{2} \frac{2 \lambda_{t-1}}{\rM^2 w_{t-1}^2} \lVert \xi
 \rVert_2^2 \right)
d \xi  \\
&= \left(\frac{m^2 w_{t-1}^2}{2 \lambda_{t-1}}\right)^{d/2} \int_{\mathbb{R}^d}
 \mathds{1} \left[ \lVert \xi \rVert_2 \leq \frac{\sqrt{2 \lambda_{t-1} } }{\rM w_{t-1} }
\right] \exp \left(- \frac{1}{2} \lVert \xi \rVert_2^2 \right)
d \xi \;.
\end{align*}

\begin{align*}
    N(g_t) &= \int_{\mathbb{R}^d} \mathds{1} \left[ \lVert \xi \rVert_2 \leq
    1/2 \right] \exp \left(- \frac{1}{2} \frac{2}{\rM^2 w_{t-1}^2} \xi^{\top}
    \wHH_t \xi \right) d \xi \\
    &= \frac{1}{\left|\det\left(\frac{\sqrt{2}}{\rM w_{t-1} } \wHH_t^{1/2}
    \right) \right| }  \int_{\mathbb{R}^d} \mathds{1} \left[
    \lVert \xi \rVert_2 \leq \frac{1}{2} \frac{\sqrt{2 \lambda_{t-1}}}{\rM w_{t-1}} \right]
     \exp \left(- \frac{1}{2} \lVert \xi
    \rVert_2^2\right)
    d \xi   \\
    & \geq \left(\frac{\rM^2 w_{t-1}^2}{2} \right)^{d/2}
    \det(\wHH_t)^{-1/2} \int_{\mathbb{R}^d} \mathds{1} \left[
    \lVert \xi \rVert_2 \leq \frac{1}{2} \frac{\sqrt{2 \lambda_{t-1} }}{\rM w_{t-1} } \right]
     \exp \left(- \frac{1}{2} \lVert \xi
    \rVert_2^2\right)
    d \xi \;.
\end{align*}
Therefore,
\begin{equation}
\label{eq:ball-d}
    \frac{N(h_t)}{N(g_t)} \leq \frac{\det(\wHH_t)}{\lambda_{t-1}^{d/2}}
    \underbrace{\frac{\int_{\mathbb{R}^d} \mathds{1} \left[ \lVert \xi \rVert_2 \leq
     \frac{\sqrt{2
     \lambda_{t-1} } }{\rM w_{t-1} }
\right] \exp \left(- \frac{1}{2} \lVert \xi \rVert_2^2 \right)
d \xi }{\int_{\mathbb{R}^d} \mathds{1} \left[ \lVert \xi \rVert_2 \leq \frac{1}{2}
\frac{\sqrt{2 \lambda_{t-1} } }{\rM w_{t-1} }
\right] \exp \left(- \frac{1}{2} \lVert \xi \rVert_2^2 \right)
d \xi }}_{R}\;.
\end{equation}
The last step consists in upper bounding
the ratio of the integrals $R$. Following, 
\cite[Lemma 6]{faury2020improved}, one gets $R = 2^d$.

%\begin{align*}
%R &= 1  + \frac{\int_{\mathbb{R}^d} \mathds{1} \left[
%\frac{1}{2} \frac{\sqrt{2 \lambda_{t-1} } }{\rM w_{t-1}} \leq
% \lVert \xi \rVert_2 \leq \frac{\sqrt{2 \lambda_{t-1} } }{\rM w_{t-1} }
%\right] \exp \left(- \frac{1}{2} \lVert \xi \rVert_2^2 \right)
%d \xi }{\int_{\mathbb{R}^d} \mathds{1} \left[ \lVert \xi \rVert_2 \leq \frac{1}{2}
%\frac{\sqrt{2 \lambda_{t-1} } }{\rM w_{t-1} }
%\right] \exp \left(- \frac{1}{2} \lVert \xi \rVert_2^2 \right)
%d \xi } \\
%& = 1 + \frac{\mathcal{V}_d\left(\frac{\sqrt{2 \lambda_{t-1} } }{\rM w_{t-1} } \right) -
% \mathcal{V}
%_d\left( \frac{1}{2}\frac{\sqrt{2 \lambda_{t-1}  } }{\rM w_{t-1}  } \right)}{\mathcal{V}_d
%\left( \frac{1}{2}
%\frac{\sqrt{2 \lambda_{t-1} } }{\rM w_{t-1} } \right) } \\
%& = 1 + \frac{1 - (1/2)^d}{(1/2)^d} \\
%& = 2^d \;,
%\end{align*}
%where $\mathcal{V}_d(r) \;\alpha \; r^d$ denotes the volume of the $d$-dimensional
% ball of radius $r$.

We conclude by using this equality in Equation \eqref{eq:ball-d} and applying
 the
 logarithm on both sides.
\end{proof}

\subsection{A Unifying Concentration Result for Discount Factors and Sliding-Window}
 \label{subsec:application_weighted_n_s}

In this section, we explain how Theorem \ref{thm:instantaneous_main} can be used
with self-concordant GLBs to obtain a concentration inequality that encapsulates the analysis for both discount-factors and the sliding-window.

%When the regularization term is omitted, the matrix $\wHH_t$ is denoted
%$\wHH_t(0)$.
%All the conditions of Theorem \ref{thm:instantaneous} are satisfied and we obtain
%for any $\delta \in (0,1]$,
%\begin{equation}
%\label{eq:concentration_discounts}
%\mathbb{P} \left( \left\lVert \gamma^{t-1} S_{t}\right\rVert_{\wHH_{t}^{-1}} \geq
% \frac{\sqrt{\lambda}} {2
% \rM} + \frac{2 \rM}{\sqrt{\lambda}}
% \log\left(\frac{\det(\wHH_{t})^{1/2}}{
%   \delta \lambda^{d/2}}\right)+\frac{2 \rM}{\sqrt{\lambda}}
%    d\log(2) \right) \leq \delta\;.
%\end{equation}

Up to now, we have stated the results in the most generic way. Actually, in our analysis
 we will use a weaker version of the concentration inequality 
 established in Theorem \ref{thm:instantaneous_main}. 
 
\begin{thm}
\label{thm:instantaneous_extension}
Let $t$ be a fixed time instant. Let $\{\mathcal{F}_u\}_{u=1}^t$ be a
filtration. Let $ \{a_u\}_{u=1}^t$ be a stochastic process on $\mathbb{R}^d$
such that $a_u$ is $\mathcal{F}_u$ measurable and $\lVert a_u\rVert_2 \leq 1$. Let
$\{\epsilon_u\}_{u=2}^{t}$ be a martingale difference sequence such
that $\epsilon_{u+1}$ is $\mathcal{F}_{u+1}$ measurable. Assume that
the weights are non-decreasing, positive and the time horizon
is known.  Furthermore, assume that conditionally on $\mathcal{F}_u$
we have $|\epsilon_{u+1}| \leq \rM$ a.s.  Let
$ \{\lambda_u \}_{u=1}^t$ be a deterministic sequence of
regularization terms and denote
$\sigma_t^2 = \mathbb{E} \left[ \epsilon_{t+1}^2 | \mathcal{F}_t
\right]$.

 Let $\wHH_{\tot} = \sum_{s=t- t_0}^{t-1} w_s^2 \sigma_s^2 a_s a_s^{\top} +
 \lambda_{t-1} \identity{d}$ and
 $S_{\tot} = \sum_{s=t-t_0}^{t-1} w_s \epsilon_{s+1} a_s$.

 Then for any $\delta \in (0,1]$,
\begin{align*}
\mathbb{P} \left( \left\lVert S_{\tot}\right\rVert_{\wHH_{\tot}^{-1}}
 \geq \frac{\sqrt{\lambda_{t-1}}} {2 \rM w_{t-1}} + \frac{2 \rM w_{t-1}
  }{\sqrt{\lambda_{t-1}} }\log\left(\frac{\det(\wHH_{\tot})^{1/2}}{
   \delta  \lambda_{t-1}^{d/2}}\right)+\frac{2 \rM w_{t-1}}{\sqrt{\lambda_{ t-1} }} d\log(2)
    \right) \leq \delta\;.
\end{align*}
\end{thm}
\begin{proof}
The arguments used to establish Theorem \ref{thm:instantaneous_extension} are
the same than for Theorem \ref{thm:instantaneous_main}. We only give the main term 
that differs from the proof of Theorem \ref{thm:instantaneous_main}.

With $t$ a fixed time instant,
for any $u$
such that $ t- t_0 \leq u
\leq t$,
$M_u^t$ is defined as
$$
M_u^t(\xi) = \exp\left(\frac{1}{\rM w_{t-1}} \, \xi^{\top} S_{t-t_0:u} -
 \frac{1}{ \rM^2 w_{t-1}^2 }
 \xi^{\top} \sum_{s=t-t_0}^{u-1} w_s^2 a_s a_s^\top \xi\right)\;,
 $$
with $S_{t-t_0:u} = \sum_{s=t-t_0}^{u-1} w_s \epsilon_{s+1} a_s$.
Following the steps of the proof of Theorem \ref{thm:instantaneous_main} with these slight 
differences gives the result.
\end{proof}

\paragraph{Discount Factors} Let $t_0=D$ be the equivalent of the sliding window
length with exponential weights,
$w_t = \gamma^{-
t}$ and $\lambda_t = \lambda \gamma^{-2t}$ for
$0 < \gamma < 1$. Even when
$\gamma$ depends on $T$, the weights satisfy the assumptions \ref{ass_T} and
\ref{ass_increasing_weights}. We can obtain:

\begin{cor}[Concentration result with discount factors]
\label{cor:concentration_Discounts}
Under the same assumption than Theorem \ref{thm:instantaneous_extension},
when defining $\wHH_{\tD} = \sum_{s= t- D}^{t-1} \gamma^{2(t-1-s)} \dot{\mu}(a_s^{\top} \theta^{\star}_s)
 a_s a_s^{\top} + \lambda  \identity{d}$ and $S_{\tD}= \sum_{s=t-D}^{t-1} \gamma^{-s} \epsilon_{s+1} a_s$.
For any $\delta \in (0,1]$,
\begin{align*}
\mathbb{P} \left( \left\lVert \gamma^{t-1} S_{\tD}\right\rVert_{\wHH_{\tD}^{-1}}
 \geq \frac{\sqrt{\lambda}} {2 \rM} + \frac{2 \rM
  }{\sqrt{\lambda} }\log\left(\frac{\det(\wHH_{\tD})^{1/2}}{
   \delta  \lambda^{d/2}}\right)+\frac{2 \rM}{\sqrt{\lambda}} d\log(2)
    \right) \leq \delta\;.
\end{align*}
\end{cor}
\paragraph{Sliding Window} With $t_0 = \tau$ the length of the sliding window,  with the weights 
satisfying $w_t = 1$ 
for $ t- \tau \leq s \leq t-1 $
 and $\lambda_t = \lambda$, we have:
\begin{cor}[Concentration result with a sliding window]
\label{cor:concentration_SW}
Under the same assumption than Theorem \ref{thm:instantaneous_extension},
when defining $\HH_t = \sum_{s= \max(1, t- \tau)} ^{t-1} \dot{\mu}(a_s^{\top} \theta^{\star}_s)
 a_s a_s^{\top} + \lambda  \identity{d}$ and $S_{t}= \sum_{s=\max(1, t- \tau)}^{t-1} \epsilon_{s+1} a_s$.
For any $\delta \in (0,1]$,
\begin{align*}
\mathbb{P} \left( \left\lVert S_{t} \right\rVert_{\HH_{t}^{-1}}
 \geq \frac{\sqrt{\lambda}} {2 \rM} + \frac{2 \rM
  }{\sqrt{\lambda}} 
  \log\left(\frac{\det(\HH_{t})^{1/2}}{
   \delta  \lambda^{d/2}}\right)+\frac{2 \rM}{\sqrt{\lambda}} d\log(2)
    \right) \leq \delta\;.
\end{align*}
\end{cor}

\section{REGRET ANALYSIS WITH DISCOUNT FACTORS}
\label{section:regret_D_appendix}

In this section we detail the regret analysis of $\DGLM$.
First we recall the main notation.
 
\subsection{Notation}

For any  $\theta \in \mathbb{R}^d$,
\begin{equation}
\label{eq:H_t_tilde_theta_D}
\widetilde{\HH}_t(\theta) = \sum_{s=1}^{t-1} \gamma^{2(t-1- s)} \dot{\mu} (a_s^\top
 \theta) a_s a_s^\top + \lambda \identity{d}\;.
\end{equation}
\begin{equation}
\label{eq:H_t_theta_D}
\HH_t(\theta) = \sum_{s=1}^{t-1} \gamma^{t-1- s} \dot{\mu} (a_s^\top
 \theta) a_s a_s^\top + \lambda \identity{d}\;.
\end{equation}
\begin{equation}
    \mbold{\widetilde{V}}_t = \sum_{s=1}^{t-1} \gamma^{2(t-1-s)} a_s a_s^{\top} +
     \frac{\lambda}
    {c_{\mu}} \identity{d}\;.
\end{equation}
\begin{equation}
    \mbold{V}_t = \sum_{s=1}^{t-1} \gamma^{t-1-s} a_s a_s^{\top} + \frac{\lambda}
    {c_{\mu}} \identity{d}\;.
\end{equation}
\begin{equation}
    \gt(\theta) = \sum_{s=1}^{t-1} \gamma^{t-1-s}\mu(a_s^{\top} \theta) a_s + \lambda
     \theta\;.
\end{equation}
\begin{equation}
S_t = \sum_{s=1}^{t-1} \gamma^{-s} \epsilon_{s+1} a_s \;.
\end{equation}
For any $\theta_1, \theta_2 \in \mathbb{R}^d$,
\begin{align*}
    \alpha(a,\theta_1,\theta_2) = \int_{0}^1 \dot{\mu}( v a^{\top} \theta_2+(1-v)
    a^\top  \theta_1)dv\;. \\
    \GG_t(\theta_1,\theta_2) =  \sum_{s=1}^{t-1} \gamma^{t-1-s}\alpha(a_s,\theta_1,
    \theta_2)a_s a_s^\top + \lambda \identity{d} \;.
\end{align*}
\begin{equation}
\label{eq:G_t_tilde_D}
\wGG_t(\theta_1,\theta_2) =  \sum_{s=1}^{t-1} \gamma^{2(t-1-s)}\alpha(a_s,\theta_1,
\theta_2)a_s a_s^\top + \lambda
  \identity{d} \;.
\end{equation}
Let $\widetilde{\HH}_t$ be defined as
\begin{equation}
    \widetilde{\HH}_t = \sum_{s=1}^{t-1} \gamma^{2(t-1-s)}\dot{\mu}(a_s^{\top}
     \theta^{\star}_s) a_s a_s^{\top} +
     \lambda \identity{d} \;.
\end{equation}
Let us define $\mathcal{T}(\gamma)$ as
\begin{equation}
    \mathcal{T}(\gamma) = \{ 1 \leq t \leq T, \text{such that}\,  \forall
   s,  \,  t- D \leq s \leq
     t-1,
     \theta^{\star}_s =
    \theta^{\star}_t\}\;.
\end{equation}
\begin{rem*}
$t \in \mathcal{T}(\gamma)$ when $t$ is a least $D$ steps away from
the closest previous breakpoint. On the contrary to the analysis with the sliding
 window (see Appendix \ref{section:regret_SW_appendix})
the bias does not completely cancel out when we are far enough from a breakpoint.
\end{rem*}
$D$ is an analysis parameter and will be specified later in the different theorems.
For the entire section we will use the notation $t-D:t$ when the sum concerns time
 instants $s$ such that $t-D \leq s \leq t-1$.
In the weighted setting, we construct an estimator based on a weighted penalized 
log-likelihood. $\hat{\theta}_t$ is defined as the unique maximizer of
$$
\sum_{s=1}^{t-1} \gamma^{t-1-s} \log \mathbb{P}_{\theta}(r_{s+1} | a_s)
- \frac{\lambda}{2} \lVert \theta \rVert_2^2\;.
$$
By using the definition of the GLM and thanks to the concavity of this equation in $
\theta$, $\hat{\theta}_t$ is the unique solution of
$$
\sum_{s=1}^{t-1} \gamma^{t-1-s} (r_{s+1} - \mu(a_s^\top \theta))a_s - \lambda \theta 
= 0\;.
$$
This can be summarized with
\begin{align}
\label{eq:char_MLE_D}
    \gt(\hat{\theta}_t) &= \sum_{s= 1}^{t-1}  \gamma^{t-1-s} r_{s+1} a_s
    = \gamma^{t-1} S_t + \sum_{s=1}^{t-1} \gamma^{t-1-s} \mu(a_s^\top \theta^\star_s) a_s\;.
\end{align}

\subsection{Analysis of the Regret of $\DGLM$}
\label{subsection:regret_D_no_proj}
In this section, we present the main ideas to obtain an analysis of the regret 
of the $\DGLM$ algorithm
when the
projection step is avoided. 

We define 
\begin{equation}
\label{eq:rho_t_appendix}
\rho_T^\delta = \left( \frac{\sqrt{\lambda}} {2 \rM} + \frac{2
\rM
  }{\sqrt{\lambda} }\log\left(\frac{T}{\delta}\right) +
 \frac{d \rM
  }{ \sqrt{\lambda} }\log\left(1 + \frac{k_\mu(1- \gamma^{2D})}{d \lambda (1- \gamma^2)}
  \right)+\frac{2 \rM}{\sqrt{\lambda}}
  d\log(2) \right) \;,
\end{equation}

and also, 
\begin{equation}
\label{eq:bar_S_appendix}
\bar{S} = S + \frac{\gamma^D(2S \km + \rM)}{\lambda(1- \gamma)}\;.
\end{equation}

The expression of $\rho_T^\delta$  and $\bar{S}$ given here coincide
with the expression in the main paper when $D = \log(T)/\log(1/\gamma)$.
$\rho_T^\delta$ is defined such that thanks to Corollary \ref{cor:concentration_Discounts} 
with high probability for all $t$ in $\mathcal{T}(\gamma)$, $\lVert \gamma^{t-1} S_{\tD} \rVert_{\wHH_{\tD}^{-1}}
\leq \rho^{\delta}_T$ holds.

The next result uses the self-concordance 
to relate the first derivative of the link function evaluated 
at different points. This relation is independent of $\cm$ 
and only depends on the distance between the parameters.
\begin{prop}
\label{prop:alpha_discount}
When $\hat{\theta}_t$ is the maximum likelihood as defined in Equation \eqref{eq:MLE_D} 
and
$t \in \mathcal{T}(\gamma)$, we have
$$
\alpha(a, \theta^\star_t, \hat{\theta}_t) \geq \left(1 +  \bar{S} + \frac{1}{\sqrt{\lambda}}
\lVert \gamma^{t-1} S_{\tD}
\rVert_{\widetilde{\GG}_{\tD}^{-1}(\theta^\star_t, \hat{\theta}_t)}
\right)^{-1} \dot{\mu}(a^\top \theta^\star_t) \;,
$$
where $\bar{S}$ is defined in Equation \eqref{eq:bar_S_appendix}.
\end{prop}
\begin{proof}
In the proof, we will replace the notation $\widetilde{\GG}_{\tD}(\theta^\star_t, \hat{\theta}_t)$
with $\widetilde{\GG}_{\tD}$ and $\widetilde{\GG}_{t}(\theta^\star_t, \hat{\theta}_t)$
with $\widetilde{\GG}_{t}$ but also $\GG_{t}(\theta^\star_t, \hat{\theta}_t)$
with $\GG_{t}$.
Using Lemma \ref{lemma:self_concordance} we have,
 $$
\alpha(a, \theta^\star_t, \hat{\theta}_t) \geq \left( 1 + \left|a^\top (\hat{\theta}_t - \theta^\star_t)
\right| \right)^{-1} \dot{\mu}(a^\top \theta^\star_t)\;.
$$
Combining this with the mean value theorem gives
$$
\alpha(a, \theta^\star_t, \hat{\theta}_t) \geq \left( 1 + \left|a^\top \GG_t^{-1
}\left(\gt(\hat{\theta}_t)- \gt(\theta^\star_t)
\right) \right| \right)^{-1} \dot{\mu}(a^\top \theta^\star_t)\;.
$$
Next, it is possible to upper bound
 $ |a^\top \GG_t^{-1}\left(\gt(\hat{\theta}_t)- \gt(\theta^\star_t) \right)|$ using the
triangle inequality and Equation~\eqref{eq:char_MLE_D}.
\begin{equation*}
\begin{split}
\left|a^\top \GG_t^{-1}\left(\gt(\hat{\theta}_t)- \gt(\theta^\star_t) 
\right) \right| 
&\leq
\underbrace{ \left|a^\top \GG_t^{-1} \sum_{s=1}^{t-1}  \gamma^{t-1-s} (\mu(a_s^\top
 \theta^\star_s) - \mu(a_s^\top \theta^
\star_t)) a_s
         \right| }_{b_{1,t}(a)} \\
& \quad + \underbrace{ \left|a^\top \GG_t^{-1}
\left( - \lambda \theta^\star_t + \sum_{s=1}^{t-D-1}
 \gamma^{t-1-s} \epsilon_{s
+1} a_s
       \right) \right|}_{b_{2,t} (a)}  
         \\& \quad 
         + \underbrace{ \left|a^\top \GG_t^{-1}  \gamma^{t-1} S_{\tD} \right|}_{b_{3,t}(a)}
\end{split}
\end{equation*}
The first term is controlled as follows,
\begin{align*}
b_{1,t}(a) &= 
\left|a^\top \GG_t^{-1} \sum_{s=1}^{t-1}  \gamma^{t-1-s} (\mu(a_s^\top \theta^\star_s)
 - \mu(a_s^\top \theta^\star_t))
 a_s  \right| 
       \\ & 
        \leq \lVert a \rVert_{\GG_t^{-1}} 
   \lVert \sum_{s=1}^{t-1} \gamma^{t-1-s} (\mu(a_s^\top \theta^\star_s) - \mu(a_s^\top \theta^
    \star_t)) a_s  \rVert_{\GG_t^{-1}}  \quad \text{(Cauchy-Schwarz ineq.)}
    \\
    & \leq \frac{1}{\sqrt{\lambda}}
   \lVert \sum_{s=1}^{t-D-1} \gamma^{t-1-s} (\mu(a_s^\top \theta^\star_s) - \mu(a_s^\top \theta^
    \star_t)) a_s  \rVert_{\GG_t^{-1}}   \quad (\GG_t \geq \lambda \identity{d} \; \text{and} \; t \in \mathcal{T}(\gamma))
    \\
 & \leq \frac{1}{\lambda} \sum_{s=1}^{t-D-1} \gamma^{t-1-s} |\alpha(a_s, \theta^\star_s, \theta^\star_t)|
  \times |a_s^{\top}
 (\theta^\star_t- \theta^\star_s)| \times \lVert a_s \rVert_2   \quad (\text{Triangle ineq.} + \GG_t \geq \lambda \identity{d})
 \\
  & \leq \frac{2S \km}{\lambda} \sum_{s=1}^{t-D-1} \gamma^{t-1-s} \quad
  (\theta^\star_s \; \text{and} \; \theta^\star_t \in \Theta)
   \\
 & \leq   \frac{2 S k_\mu}{\lambda} \frac{\gamma^{D}}{1- \gamma} \;. 
\end{align*}
Using similar arguments, one can upper bound $b_{2,t}(a)$.
\begin{align*}
b_{2,t}(a) &= \left|a^\top \GG_t^{-1} \left(
-\lambda \theta^\star_t + \sum_{s=1}^{t-D-1} \gamma^{t-1-s} \epsilon_{s+1} a_s
       \right) \right| \\
    & \leq S + \lVert \sum_{s=1}^{t-D-1} \gamma^{t-1-s} \epsilon_{s+1} a_s  \rVert_{\GG_t^{-2}} \\
    & \leq S + \frac{\rM}{\lambda} \frac{\gamma^{D}}{1- \gamma}\;.
    \quad (|\epsilon_{s+1} | \leq \rM )
\end{align*}
Before upper bounding, $b_{3,t}(a)$, we need the following relation. 

When $ 0< \gamma<1$,
$\gamma^{2(t-1-s)} \leq \gamma^{t-1-s}$ for $s$ smaller than $t-1$ which implies
\begin{equation}
\label{eq:link_G_t_G_t_tilde}
\forall \theta_1, \theta_2 \in \mathbb{R}^d, \;  \widetilde{\GG}_t(\theta_1, \theta_2) \leq \GG_t(\theta_1, \theta_2)\;.
\end{equation}
We have,
\begin{align*}
b_{3,t}(a) &= |a^\top \GG_t^{-1} \widetilde{\GG}_t^{1/2} \widetilde{\GG}_t^{-1/2}
 \gamma^{t-1} S_{\tD}   | \\
& \leq \lVert a \rVert_{\GG_t^{-1} \widetilde{\GG}_t \GG_t^{-1}}  \lVert \gamma^{t-1}
 S_{\tD} \rVert_{\widetilde{\GG}^{-1}_t} \quad \text{(Cauchy-Schwarz ineq.)} \\
& \leq \lVert a \rVert_{\GG_t^{-1}} \lVert \gamma^{t-1} S_{\tD} \rVert_{\widetilde{\GG}
^{-1}_t} \quad
\textnormal{(Equation \eqref{eq:link_G_t_G_t_tilde})} \\
& \leq \frac{1}{\sqrt{\lambda}} \lVert \gamma^{t-1} S_{\tD} \rVert_{\widetilde{\GG
}_{\tD}^{-1}} \;.  \quad (\GG_t \geq \lambda \identity{d})
\end{align*}
By combining all the results we have,
\begin{align*}
\alpha(a, \theta^\star_t, \hat{\theta}_t) \geq \left(1 + \bar{S} + \frac{1}
{\sqrt{\lambda}}
\lVert \gamma^{t-1} S_{\tD} \rVert_{\widetilde{\GG}_{\tD}^{-1}}
\right)^{-1} \dot{\mu}(a^\top \theta^\star_t) \;.
\end{align*}
\end{proof}
\begin{cor}
\label{cor:S_t_G_t_S_t_H_t}
When $\hat{\theta}_t$ is the maximum likelihood as defined in Equation \eqref{eq:MLE_D},
and
$t \in \mathcal{T}(\gamma)$, we have
$$
\widetilde{\GG}_{\tD}(\theta^\star_t, \hat{\theta}_t) \geq \left( 1 
+ \bar{S}
 + \frac{1}{\sqrt{\lambda}}
\lVert \gamma^{t-1} S_{\tD} \rVert_{\widetilde{\GG}_{\tD}^{-1} (\theta^\star_t,
 \hat{\theta}_t) }\right)^{-1}
\widetilde{\HH}_{\tD}
\;.
$$
\end{cor}
This proposition establishes a useful link between $\widetilde{\GG}_{\tD}
(\theta^\star_t, \hat{\theta}_t)$ and
$\widetilde{\HH}_{\tD}$. 
\begin{proof}
Thanks to Proposition \ref{prop:alpha_discount},
$$
\alpha(a_s, \theta^\star_t, \hat{\theta}_t) \geq \left(1 + \bar{S} + \frac{1}
{\sqrt{\lambda}}
\lVert \gamma^{t-1} S_{\tD} \rVert_{\widetilde{\GG}_t^{-1}(\hat{\theta}_t, \theta^\star_t)}
\right)^{-1} \dot{\mu}(a_s^\top \theta^\star_t) \;.
$$
Therefore,
\begin{equation*}
\begin{split}
\sum_{s=t- D}^{t-1} \gamma^{2(t-1-s)} \alpha(a_s, \theta^\star_t, \hat{\theta}_t) a_s
 a_s^\top
&\geq
\left(1 + \bar{S} + \frac{1}{\sqrt{\lambda}}
\lVert \gamma^{t-1} S_{\tD} \rVert_{\widetilde{\GG}_t^{-1}(\hat{\theta}_t, \theta^\star_t)}
\right)^{-1} \\
& \quad \times \sum_{s= t- D}^{t-1} \gamma^{2(t-1-s)} \dot{\mu}(a_s^\top \theta^
\star_t) a_s a_s^\top\;.
\end{split}
\end{equation*}
We obtain the announced result by using $\theta^\star_s = \theta^\star_t$ for $ t-
 D \leq s \leq t-1$ because
$ t \in \mathcal{T}(\gamma)$ and by adding the regularization terms.
\end{proof}
Using Proposition \ref{prop:alpha_discount} and
Corollary \ref{cor:S_t_G_t_S_t_H_t},
 we can now prove Proposition \ref{prop:deviation_main}. The proposition
establishes an upper bound for the deviation of the MLE (through $\gamma^{t-1} S_{\tD}$) that only depends
on $\rho_T^\delta$ the high probability upper bound obtained using 
Corollary \ref{cor:concentration_Discounts}. 
\propdeviationmain*
\begin{rem*}
Here, note that the left-hand side is controlled under the norm 
$\wGG_{\tD}^{-1}(\hat{\theta}_t, \theta^\star_t)$, whereas the right 
hand side is the consequence of the upper bound
 of the same term controlled in the $\wHH_{\tD}^{-1}$-norm 
 (Corollary \ref{cor:concentration_Discounts}).
  Linking those two matrices independently from $\cm$ 
  is not-straightforward. 
  The self-concordance is the key ingredient to obtain this bound.
\end{rem*}
\begin{proof}
Applying Corollary \ref{cor:S_t_G_t_S_t_H_t},
\begin{align*}
\lVert \gamma^{t-1} S_{\tD} \rVert_{\widetilde{\GG}^{-1}_{\tD}(\hat{\theta}_t, \theta^\star_t  )}^2
 \leq \left( 1 + \bar{S}  +
\frac{1}{\sqrt{\lambda}}  \lVert \gamma^{t-1} S_{\tD} \rVert_{\widetilde{\GG}^{-1}
_{\tD}(\hat{\theta}_t, \theta^\star_t )} \right)
\lVert \gamma^{t-1} S_{\tD} \rVert_{\widetilde{\HH}^{-1}_{\tD}}^2\;.
\end{align*}
Let $X = \lVert \gamma^{t-1} S_{\tD} \rVert_{\widetilde{\GG}^{-1}_{\tD}(\hat{\theta}_t, \theta^\star_t)}$, 
it gives the
 following constraint,
\begin{align*}
\forall X, \,
{ \color{red}X^2} - \frac{1}{\sqrt{\lambda}}  \lVert \gamma^{t-1} S_{\tD} \rVert_{\widetilde{\HH}^{-1}
_{\tD}}^2  
{ \color{red} X} -
\left( 1 + \bar{S} \right)  \lVert \gamma^{t-1} S_{\tD}
\rVert_{\widetilde{\HH}^{-1}_{\tD}}^2 \leq 0\;.
\end{align*}
Solving this polynomial inequality yields
$$
\lVert \gamma^{t-1} S_{\tD} \rVert_{\widetilde{\GG}_t^{-1}(\theta^\star_t, \hat{\theta}_t)}
\leq
\frac{1}{\sqrt{\lambda}} \lVert  \gamma^{t-1} S_{\tD} \rVert_{\widetilde{\HH}^{-1}
_{\tD}}^2
+   \sqrt{1 + \bar{S} }   \lVert  \gamma^{t-1} S_{\tD}
 \rVert_{\widetilde{\HH}^{-1}_{\tD}}\;.
$$
The result is then obtained by applying Corollary \ref{cor:concentration_Discounts}.
\end{proof}

\begin{cor}
\label{cor:G_t_V_t_D}
When $\hat{\theta}_t$ is the maximum likelihood as defined in Equation \eqref{eq:MLE_D}
and
$t \in \mathcal{T}(\gamma)$, we have
$$
\GG_{t}(\theta^\star_t, \hat{\theta}_t) \geq \left( 1 + \bar{S} + \frac{1}
{\sqrt{\lambda}}
\lVert \gamma^{t-1} S_{\tD} \rVert_{\widetilde{\GG}_{\tD}^{-1} (\theta^\star_t,
 \hat{\theta}_t) }\right)^{-1}
c_\mu \VV_t
\;.
$$
\end{cor}

\begin{proof}
Similar to the proof of Corollary \ref{cor:S_t_G_t_S_t_H_t}.
\end{proof}

In the next proposition, we give an upper bound for $\Delta_t(a, \hat{\theta}_t)$ the prediction error in $
\hat{\theta}_t$ which is directly connected to the instantaneous regret. 

Here, $\beta_T^\delta$ is defined as in the main paper in Equation \eqref{eq:beta_main} but we replace 
$\rho_T^\delta$ and $\bar{S}$ with the expressions stated Equation \eqref{eq:rho_t_appendix} and
 \eqref{eq:bar_S_appendix}.
\begin{prop}
\label{prop:upper_delta_t_D_no_proj}
 For any $\delta \in (0, 1]$, with probability higher
  than $1-\delta$,
\begin{equation*}
\forall t \in \mathcal{T}(\gamma), \;
\Delta_t(a, \hat{\theta}_t) \leq \frac{k_\mu}{\lambda} \frac{\gamma^D}{1-\gamma}(2 S \km + \rM) +
\frac{\beta_T^\delta}{\sqrt{c_\mu}} 
\lVert a \rVert_{\VV_t^{-1}}\;.
\end{equation*}
\end{prop}
\begin{proof}
We  denote $\GG_t = \GG_t(\theta^\star_t, \hat{\theta}_t)$ and we have,
\begin{align*}
    \Delta_t(a, \hat{\theta}_t) &= |\mu(a^{\top} \theta^{\star}_t) -
    \mu(a^{\top} \hat{\theta}_t)| \\
    & \leq k_{\mu} | a^{\top} (\theta^{\star}_t - \hat{\theta}_t) | \\
    & = k_{\mu} | a^{\top} \GG_t^{-1}
     (\gt(\theta^{\star}_t) - \gt(\hat{\theta}_t)) |
     \quad \textnormal{(Mean-Value Theorem)}
     \\
    &= k_{\mu} \left| a^{\top} \GG_t^{-1}
     \left( \sum_{s=1}^{t-1} \gamma^{t-1-s} (\mu(a_s^\top
\theta^\star_t) - \mu(a_s^\top \theta^\star_s))a_s     + \lambda \theta^\star_t -
 \gamma^{t-1} S_t
     \right) \right| \;.
\end{align*}
In the last equality, we have used the characterization of the MLE (Equation \eqref{eq:char_MLE_D}).
\begin{equation*}
\begin{split}
    \Delta_t(a, \hat{\theta}_t) &\leq k_\mu  \underbrace{ \left|a^\top \GG_t^{-1}
    \sum_{s=1}^{t-1} \gamma^{t-1-s} (\mu(a_s^\top
\theta^\star_t) - \mu(a_s^\top \theta^\star_s))a_s \right| }_{c_{1,t}(a)} \\
& \quad + k_\mu \underbrace{ \left|a^\top \GG_t^{-1}\sum_{s=1}^{t- D-1} \gamma^{t-1-s}
 \epsilon_{s+1} a_s \right|}_{c_{2,t}
(a)}  
%\\& \quad 
+ k_\mu \underbrace{\left|a^\top \GG_t^{-1} \left(
\gamma^{t-1} S_{\tD} -
 \lambda
\theta^\star_t \right) \right|}_{c_{3,t}(a)}\;.
\end{split}
\end{equation*}
We will bound the different terms.

$c_{1,t}(a)$ can be bounded like $b_{1,t}(a)$ in the proof of Proposition
\ref{prop:alpha_discount}.
\begin{align*}
c_{1,t}(a) \leq \frac{2S k_\mu}{\lambda} \frac{\gamma^{D}}{1-\gamma}\;.
\end{align*}
$c_{2,t}(a)$ can be bounded like $b_{2,t}(a)$ in the proof of the same proposition.
\begin{align*}
c_{2,t}(a) \leq \frac{m}{\lambda} \frac{\gamma^{D}}{1-\gamma}\;.
\end{align*}
The last term requires more work. $\widetilde{\GG}_t(\theta^\star_t, \hat{\theta}_t)$
 will be denoted
$\widetilde{\GG}_t$ for simplicity.
\begin{align*}
c_{3,t}(a) &= \left|a^\top \GG_t^{-1} \left(
\gamma^{t-1} S_{\tD} - \lambda
\theta^\star_t \right) \right| 
=  \left|a^\top \GG_t^{-1} \widetilde{\GG}_t^{1/2} \widetilde{\GG}_t^{-1/2}
\left( \gamma^{t-1} S_{\tD} - \lambda
\theta^\star_t \right) \right|  \\
& \leq \lVert a \rVert_{\GG_t^{-1} \widetilde{\GG}_t \GG_t^{-1}} \lVert
 \gamma^{t-1} S_{\tD}
 - \lambda \theta^\star_t
\rVert_{\widetilde{\GG}^{-1}_t} \\
&\leq \lVert a \rVert_{\GG_t^{-1}} \lVert \gamma^{t-1} S_{\tD} - \lambda
 \theta^\star_t
\rVert_{\widetilde{\GG}^{-1}_t}  \quad \text{(Equation \eqref{eq:link_G_t_G_t_tilde})}
\\
&\leq \lVert a \rVert_{\GG_t^{-1}}  \left( \sqrt{\lambda} S + \lVert \gamma^{t-1} S_{\tD}
 \rVert_{\widetilde{\GG}^{-1}_t} \right) 
\quad (\wGG_t \geq \lambda \identity{d} \;
\textnormal{and Assumption \ref{ass:bounded_actions}} )
 \\
& \leq \frac{\lVert a \rVert_{\VV_t^{-1}}}{\sqrt{c_\mu}} \sqrt{1 + \bar{S} +
 \frac{1}{\sqrt{\lambda}}
\lVert \gamma^{t-1} S_{\tD} \rVert_{\widetilde{\GG}_{\tD}^{-1} }} \left( \sqrt{\lambda} S
 + \lVert \gamma^{t-1} S_{\tD}
 \rVert_{\widetilde{\GG}^{-1}_{\tD}} \right)\;.
\end{align*}
In the last inequality we used Corollary \ref{cor:G_t_V_t_D}.
The next step consists in upper bounding  $\lVert \gamma^{t-1} S_{\tD}
 \rVert_{\widetilde{\GG}^{-1}_{\tD}}$ with
Proposition \ref{prop:deviation_main} and
to combine this with the high probability upper bound from Corollary \ref{cor:concentration_Discounts}.
Therefore, with probability higher than $1-\delta$,
\begin{align*}
c_{3,t}(a) &\leq \frac{\lVert a\rVert_{\VV_t^{-1}}}{\sqrt{c_\mu}} 
\sqrt{1 + \bar{S} + \sqrt{\frac{1 + \bar{S}}{\lambda}} \rho_T^\delta  + \frac{1}{\lambda} (\rho_T^\delta)^2      }   
 \left( \sqrt{\lambda} S + \lVert \gamma^{t-1} S_{\tD} \rVert_{\widetilde{\GG}^{-1}
 _{\tD}} \right) \\
 & \leq \frac{ \sqrt{\lambda} }{\sqrt{c_\mu}} \lVert a\rVert_{\VV_t^{-1}}
\sqrt{1 + \bar{S} + \sqrt{\frac{1 + \bar{S}}{\lambda}} \rho_T^\delta  + \frac{1}{\lambda} (\rho_T^\delta)^2      }   
 \left( S+ \sqrt{\frac{1 + \bar{S}}{\lambda}} \rho_T^\delta  + \frac{1}{\lambda} (\rho_T^\delta)^2  \right) \\
 & \leq \frac{\sqrt{\lambda}}{\sqrt{c_\mu}} \lVert a\rVert_{\VV_t^{-1}}
 \left(
 1 + \bar{S} + \sqrt{\frac{1 + \bar{S}}{\lambda}} \rho_T^\delta  + \frac{1}{\lambda} (\rho_T^\delta)^2 
 \right)^{3/2} \;.
\end{align*}
\end{proof}
The first term of the right hand side of Proposition \ref{prop:upper_delta_t_D_no_proj}
is a bias term resulting from the non-stationarity of the environment. The second term 
results from the concentration results we have established in Section \ref{section:concentration_weights}
combined with the self-concordance assumption.

With $\beta_T^\delta$ defined in Equation \eqref{eq:beta_main}, the algorithm $\DGLM$ selects the action at time
$t$ as follows,
\begin{align}
a_t &= \argmax_{a \in \mathcal{A}_t} \left( \mu(a^\top \hat{\theta}_t)  + \frac{\beta^\delta_T}{\sqrt{c_\mu}} \lVert a
 \rVert_{\VV_t^{-1}} +  \frac{k_\mu}
{\lambda} \frac{\gamma^D}{1-\gamma}(2 S \km + \rM) \right)  \notag
\\
&=
 \argmax_{a \in \mathcal{A}_t} \left( \mu(a^\top \hat{\theta}_t)  + \frac{\beta^\delta_T}{\sqrt{c_\mu}}
 \lVert a \rVert_{\VV_t^{-1}}  \right)\;. \label{eq:choosing_action}
\end{align}

Note that the bias term is independent of the action.
Nevertheless, this term will appear in the upper bound for the regret. Equation \eqref{eq:choosing_action}
explains how the actions are chosen in Algorithm \ref{alg:D-GLM}.

We can now give the main theorem.
\thregretDnoprojmain*

\begin{proof}
Using Proposition \ref{prop:upper_delta_t_D_no_proj}, we obtain a high probability upper bound for $\Delta_t(a,
 \hat{\theta}_t)$. We recall that the exploration bonus of $\DGLM$ is defined as,
$$
\frac{1}{\sqrt{c_\mu}} \beta_T^\delta
\lVert a_t \rVert_{\VV_t^{-1}} 
+
 \frac{k_\mu}{\lambda}   \frac{\gamma^D}{1-\gamma}(2 S \km + \rM)
\;.
$$
Furthermore, the estimator used by $\DGLM$ is the MLE 
$\hat{\theta}_t$ as defined in Equation \eqref{eq:MLE_D},
all the conditions required 
for applying Proposition \ref{prop:delta_t_regret} are met.
Hence when $t \in \mathcal{T}(\gamma)$,
$$
r_t \leq 
 \frac{2}{\sqrt{c_\mu}} \beta_T^\delta
\lVert a_t \rVert_{\VV_t^{-1}}
+
  \frac{2k_\mu}{\lambda}  \frac{\gamma^D}{1-\gamma}(2 S \km + \rM)
\;.
$$
The dynamic regret can then be upper bounded by, 
\begin{align*}
R_T &= \sum_{t=1}^T r_T = \sum_{t \in \mathcal{T}(\gamma)} r_t + \sum_{t \notin
 \mathcal{T}(\gamma)} r_t 
 \leq \Gamma_T D + \sum_{t \in \mathcal{T}(\gamma)} r_t \\
& \leq \Gamma_T D +  \frac{2k_\mu}{\lambda}  \frac{\gamma^D}{1-\gamma}(2 S \km + \rM) T + 
\frac{2\beta_T^\delta}{\sqrt{c_\mu}} \sum_{t \in
 \mathcal{T}(\gamma)} \lVert a_t \rVert_{\VV_t^{-1}} \\
& \leq \Gamma_T D + \frac{2 k_\mu}{\lambda}  \frac{\gamma^D}{1-\gamma}(2 S \km + \rM) T 
+ \frac{2 \beta_T^\delta}{\sqrt{c_\mu}} \sqrt{T}
\sqrt{\sum_{t \in \mathcal{T}(\gamma)}\lVert a_t \rVert_{\VV_t^{-1}}^2 } \quad \textnormal{(Cauchy-Schwarz ineq.)} \\
& \leq \Gamma_T D + \frac{2 k_\mu}{\lambda}  \frac{\gamma^D}{1-\gamma}(2 S \km + \rM) T 
+ \frac{2 \beta_T^\delta}{\sqrt{c_\mu}} \sqrt{T}
\sqrt{\sum_{t=1}^T\lVert a_t \rVert_{\VV_t^{-1}}^2 } \\
& \leq  \Gamma_T D + 
\frac{2 k_\mu}{\lambda}  \frac{\gamma^D}{1-\gamma}(2 S \km + \rM) T 
+ \frac{2 \beta^\delta_T}{\sqrt{c_\mu}} \sqrt{T} \sqrt{2\max\left(1,\frac{1}{\lambda}
\right)
        \log \left( \frac{\det(\VV_{T+1})}{\gamma^{dT} \lambda^d } \right) }\;.
\end{align*}

The last inequality uses Lemma \ref{lemma:ellipticalpotential_Discount}.
Next, we use Corollary \ref{corollary:inequality_determinant_V}
 to upper bound the determinant,
$$
\frac{\det(\VV_{T+1})}{ \gamma^{dT}\lambda^d} \leq \gamma^{-dT} \left( 1 + 
\frac{1- \gamma^T}{\lambda d (1- \gamma)} \right)^d \;.
$$
Applying the logarithm function on both sides yields
\begin{equation*}
\begin{split}
R_T &\leq \Gamma_T D + \frac{2 k_\mu( 2 S k_\mu + \rM)}{\lambda} \frac{\gamma^D}{1- \gamma}T \\
& \quad +
 \frac{2\beta^\delta_T}{\sqrt{c_\mu}} \sqrt{dT}  \sqrt{2\max \left(1,\frac{1}{\lambda}
 \right)}
\sqrt{T \log(1/\gamma) + \log \left(1 + \frac{1}{d \lambda (1- \gamma)} \right) } \;.
\end{split}
\end{equation*}
With the additional constraint $1/2 < \gamma < 1$, by setting $D = \log(T)/\log(1/\gamma)$, 
noticing that $0< 1/\gamma-1 <1$ and using $\log(1+x) \geq x/2 \text{ for } 0 < x < 1$, we have
\begin{align*}
\log(1/\gamma) &= \log(1 + 1/\gamma-1) 
\geq \frac{1-\gamma}{2\gamma}\;.
\end{align*}
Therefore, we have $D \leq \frac{2 \gamma \log(T)}{1- \gamma}$.

By properly balancing the bias term due to the non-stationarity and the rate at which
 the weighted MLE approaches the true 
bandit parameter, the asymptotic behavior of $\DGLM$ can be characterized as follows:
%\asymptotic*
By setting $\gamma = 
1-\left(\frac{\cm^{1/2}\Gamma_T}{dT}\right)^{2/3}$ and
$\lambda = d\log(T)$, we have:
\begin{itemize}
\item $\frac{2 \log(T)}{1- \gamma}\Gamma_T$ scales as $\widetilde{\mathcal{O}}(\cm^{-1/3}d^{2/3} \Gamma_T^{1/3}
T^{2/3})$.
\item $\frac{2 k_\mu( 2 S k_\mu + \rM)}{\lambda} \frac{1}{1- \gamma}$
scales as $\widetilde{\mathcal{O}}(\cm^{-1/3}d^{2/3} \Gamma_T^{-2/3} T^{2/3})$.
\item $\frac{2\beta^\delta_T}{\sqrt{c_\mu}} \sqrt{dT}  \sqrt{2\max \left(1,\frac{1}{\lambda}
 \right)}
\sqrt{T \log(1/\gamma) + \log \left(1 + \frac{1}{d \lambda (1- \gamma)} \right) }$ 
scales as $\frac{1}{\sqrt{\cm}} d T \sqrt{\log(1/\gamma)}$ when omitting logarithmic factors
and constant terms.
\end{itemize}
Using $- \log(1-x) \leq \frac{x-1}{x} \text{ for }  0 \leq x <1$, we also have
\begin{align*}
\sqrt{\log(1/\gamma)} &= \sqrt{- \log(1- (1-\gamma))}
\leq \sqrt{\frac{1- \gamma}{\gamma}} 
 \leq \sqrt{2(1- \gamma)}\;.
\end{align*}
$\sqrt{\log(1/\gamma)}$  scales as $\widetilde{\mathcal{O}}(\cm^{1/6} d^{-1/3} \Gamma_T^{1/3}  T^{-1/3})$. 
Hence scales $\cm^{-1/2} d T \sqrt{\log(1/\gamma)}$ as
$\widetilde{\mathcal{O}}(\cm^{-1/3} d^{2/3} \Gamma_T^{1/3}  T^{2/3})$.
Combining the different terms concludes the proof.
\end{proof}

Using Assumption \ref{ass:minimum_gap}, we can obtain refined regret bounds.

\subsection{Gap-Dependent Bound}
\thmproblem*
\begin{proof}
First note that for any suboptimal action 
$a \in \mathcal{A}_t$, 
$$
\mu(a_{\star,t}^\top \theta^\star_t) - \mu(a^\top \theta^\star_t) \geq \Delta\;.
$$
This implies
\begin{equation}
\label{eq:lien_r_t_delta}
r_t =  \mu(a_{\star,t}^\top \theta^\star_t) 
- \mu(a_t^\top \theta^\star_t) \leq \frac{\left(\mu(a_{\star,t}^\top \theta^\star_t) 
- \mu(a_t^\top \theta^\star_t)\right)^2}{\Delta} = \frac{r_t^2}{\Delta}\;.
\end{equation}
Using Proposition \ref{prop:delta_t_regret} one has,
$$
r_t \leq 
 \frac{2}{\sqrt{c_\mu}} \beta_T^\delta
\lVert a_t \rVert_{\VV_t^{-1}}
+
  \frac{2k_\mu}{\lambda}  \frac{\gamma^D}{1-\gamma}(2 S \km + \rM)
\;.
$$
This implies in particular,
\begin{equation}
\label{eq:r_t_2}
r_t^2 \leq \underbrace{\frac{4}{\cm} (\beta_T^\delta)^2 \lVert a_t \rVert_{\VV_t^{-1}}^2}_{r_{1,t}} 
+ \underbrace{\frac{4 \km^2}{\lambda^2} 
\frac{\gamma^{2D}}{(1-\gamma)^2}(2S\km + \rM)^2}_{r_{2,t}}
+  \underbrace{\frac{8 \km}{\lambda} \frac{\beta_T^\delta}{\sqrt{\cm}} \frac{\gamma^D}{1-\gamma}(2S \km + \rM) 
\lVert a_t \rVert_{\VV_t^{-1}}}_{r_{3,t}}.
\end{equation}
The dynamic regret can then be upper bounded by, 
\begin{align*}
R_T &= \sum_{t=1}^T r_T = \sum_{t \in \mathcal{T}(\gamma)} r_t + \sum_{t \notin
 \mathcal{T}(\gamma)} r_t 
  \leq \Gamma_T D + \sum_{t \in \mathcal{T}(\gamma)} (\mu(a_{\star,t}^\top \theta^\star_t) 
- \mu(a_t^\top \theta^\star_t))
\\
& \leq \Gamma_T D + \frac{1}{\Delta}\sum_{t \in \mathcal{T}(\gamma)} r_t^2\;.
\quad (\text{Equation } \eqref{eq:lien_r_t_delta})
\end{align*}
By applying Equation \eqref{eq:r_t_2}, the regret can be separated in 4 different terms.

When summing for the different time instants $r_{1,t}$ becomes
\begin{align*}
 \sum_{t=1}^T  r_{1,t}
& \leq 
 \frac{8}{\cm} (\beta_T^\delta)^2
  \max\left(1,\frac{1}{\lambda}
\right)
        \log \left( \frac{\det(\VV_{T+1})}{\gamma^{dT} \lambda^d } \right)
          \quad (\text{Lemma 
\ref{lemma:ellipticalpotential_Discount}}) \\
& \leq  \frac{8d}{\cm} (\beta_T^\delta)^2 
\max\left(1,\frac{1}{\lambda}
\right)
\left(
T \log(1/\gamma) + \log \left(1 + \frac{1}{d \lambda (1- \gamma)} \right)
\right)\;.
\quad \text{(Corollary \ref{corollary:inequality_determinant_V})}
\end{align*}
For $r_{2,t}$, we have
\begin{align*}
\sum_{t=1}^T r_{2,t} \leq \frac{4 \km^2}{\lambda^2} 
\frac{\gamma^{2D} T}{(1-\gamma)^2}(2S\km + \rM)^2 \;.
\end{align*}
Furthermore, $r_{3,t}$ is treated as follows:
\begin{align*}
\sum_{t=1}^T r_{3,t} & \leq  \frac{8 \km}{\lambda} \frac{\beta_T^\delta}{\sqrt{\cm}} 
\frac{\gamma^D}{1-\gamma}(2S \km + \rM) 
\sum_{t=1}^T \lVert a_t \rVert_{\VV_t^{-1}} \\
& \leq \frac{8 \km}{\lambda} \frac{\beta_T^\delta}{\sqrt{\cm}} \frac{\gamma^D}{1-\gamma}(2S \km + \rM) 
\sqrt{T} \sqrt{ \sum_{t=1}^T \lVert a_t \rVert^2_{\VV_t^{-1}}} \\
& \leq  \frac{8 \km \beta_T^\delta}{\lambda \sqrt{\cm}} \frac{\gamma^D}{1-\gamma}(2S \km + \rM) 
%\sqrt{dT}  
\sqrt{2dT\max \left(1,\frac{1}{\lambda}
 \right)}
\sqrt{T \log\left(\frac{1}{\gamma}\right) + \log \left(1 + \frac{1}{d \lambda (1- \gamma)} \right) }.
\end{align*}
When $\lambda = d \log(T)$, $D = \frac{\log(T)}{\log(1/\gamma)}$ and 
$\gamma = 1 - \frac{\sqrt{\cm \Gamma_T}}{ d \sqrt{T}}$, we can upper bound the different terms 
following the proof of Theorem \ref{th:regret_D_no_proj_main}.

With those choices, 
\begin{enumerate}
\item $\Gamma_T D $ scales as $\widetilde{\mathcal{O}}(\cm^{-1/2} d \Gamma_T^{1/2} T^{1/2})$
\item $\sum_{t=1}^T r_{1,t}$ scales as $\widetilde{\mathcal{O}}(\cm^{-1/2} d \Gamma_T^{1/2} T^{1/2})$
\item $\sum_{t=1}^T r_{2,t}$ scales as $\widetilde{\mathcal{O}}(\cm^{-1}  \Gamma_T^{-1})$
\item $\sum_{t=1}^T r_{3,t}$ scales as $\widetilde{\mathcal{O}}(d^{1/4} \cm^{-3/4} \Gamma_T^{-1/4} T^{1/4})$
\end{enumerate}
Keeping the highest order term in $T$ and dividing by $\Delta$ yields the announced result.
\end{proof}

\subsection{Refined Exploration Bonus when $\hat{\theta}_t \in \Theta$}

\label{subrefined_explo}

As briefly explained in Remark \ref{rem:fizefzjapcd} in the main paper, when the MLE 
is an admissible parameter ($\hat{\theta}_t \in \Theta$) it is possible to obtain a usually tighter concentration result.
In this section, we explain exactly how this can be done. Note that this improvement is mostly useful for the 
design of the algorithm
and has no impact on the regret guarantees.

We define 
\begin{equation}
\label{eq:beta_improved}
\bar{\beta}_T^\delta = \km \sqrt{1 + 2S} \left(\sqrt{\lambda} S + \rho_T^\delta \right)\;,
\end{equation}
where $\rho_T^\delta$ is defined in Equation \eqref{eq:rho_t_appendix}.
\begin{prop}
\label{proposition:beta_improved}
 For any $\delta \in (0, 1]$,
 with probability higher
  than $1-\delta$,
\begin{equation*}
\forall t  \in \mathcal{T}(\gamma) \; s.t \; \hat{\theta}_t \in \Theta, \,
\Delta_t(a, \hat{\theta}_t) \leq \frac{k_\mu}{\lambda} 
\frac{\gamma^D}{1-\gamma}(2 S \km + \rM) +
\frac{\bar{\beta}_T^\delta}{\sqrt{c_\mu}} 
\lVert a \rVert_{\VV_t^{-1}}\;.
\end{equation*}
\end{prop}
\begin{proof}
We use the notation $\GG_t$ (respectively $\wGG_t$) instead of  
$\GG_t(\theta^\star_t, \hat{\theta}_t)$ 
(respectively $\wGG_t(\theta^\star_t, \hat{\theta}_t)$).
Following the same steps as for the proof of Proposition \ref{prop:upper_delta_t_D_no_proj},
one gets
\begin{align*}
    \Delta_t(a, \hat{\theta}_t) &\leq
    \frac{\km}{\lambda} \frac{\gamma^D}{1-\gamma} (2S\km + \rM) + \km | a^\top \GG_t^{-1} 
    (\gamma^{t-1} S_{\tD} - \lambda \theta^\star_t)| \\
    & \leq
        \frac{\km}{\lambda} \frac{\gamma^D}{1-\gamma} (2S\km + \rM)  + 
     \lVert a \rVert_{\GG_t^{-1} \widetilde{\GG}_t \GG_t^{-1}} \lVert
 \gamma^{t-1} S_{\tD}
 - \lambda \theta^\star_t
\rVert_{\widetilde{\GG}^{-1}_t} \\
& \leq
        \frac{\km}{\lambda} \frac{\gamma^D}{1-\gamma} (2S\km + \rM)  + 
     \lVert a \rVert_{\GG_t^{-1}} \lVert
 \gamma^{t-1} S_{\tD}
 - \lambda \theta^\star_t
\rVert_{\widetilde{\GG}^{-1}_t}\;. \quad (\text{Equation \eqref{eq:link_G_t_G_t_tilde}})
\end{align*}
Here, with the additional assumption $\hat{\theta}_t \in \Theta$, the self-concordance can be used to obtain an easier
relation between $\wGG_t$ and $\wHH_t$ as stated in Lemma \ref{lemma:boundGtbyHt_D}.
\begin{align*}
    \Delta_t(a, \hat{\theta}_t) &\leq
       \frac{\km}{\lambda} \frac{\gamma^D}{1-\gamma} (2S\km + \rM)  + 
    \sqrt{1 + 2S} \lVert a \rVert_{\GG_t^{-1}} \lVert
 \gamma^{t-1} S_{\tD}
 - \lambda \theta^\star_t
\rVert_{\wHH^{-1}_t}  
\quad \text{(Lemma \ref{lemma:boundGtbyHt_D})}
\\ 
& \leq   \frac{\km}{\lambda} \frac{\gamma^D}{1-\gamma} (2S\km + \rM)  + 
    \sqrt{1 + 2S} \lVert a \rVert_{\GG_t^{-1}} \lVert
 \gamma^{t-1} S_{\tD}
 - \lambda \theta^\star_t
\rVert_{\wHH^{-1}_{\tD}}\;.
\end{align*}
The last inequality uses $\wHH_{\tD} \leq \wHH_t$.
Now by applying Corollary \ref{cor:concentration_Discounts}, $\Delta_t(a, \hat{\theta}_t)$ can be further 
upper bounded.
\begin{align*}
    \Delta_t(a, \hat{\theta}_t) &\leq
     \frac{\km}{\lambda} \frac{\gamma^D}{1-\gamma} (2S\km + \rM)  + 
      \sqrt{1 + 2S} \lVert a \rVert_{\GG_t^{-1}}  \left( \sqrt{\lambda} S + \rho_T^\delta \right)\;.
\end{align*}
The final step consists in using $\GG_t = \GG_t(\theta^\star_t, \hat{\theta}_t) \geq \cm \VV_t$ 
which holds because both $\hat{\theta}_t$ and $\theta^\star_t$ are in $\Theta$.
\end{proof}

Consequently, when $\hat{\theta}_t \in \Theta$, the action $a_t$ at time $t$ can be 
chosen according to:
\begin{align}
a_t &= \argmax_{a \in \mathcal{A}_t} \left( \mu(a^\top \hat{\theta}_t)  + \frac{\bar{\beta}^\delta_T}{\sqrt{c_\mu}} \lVert a
 \rVert_{\VV_t^{-1}} +  \frac{k_\mu}
{\lambda} \frac{\gamma^D}{1-\gamma}(2 S \km + \rM) \right)  \notag
\\
&=
 \argmax_{a \in \mathcal{A}_t} \left( \mu(a^\top \hat{\theta}_t)  + \frac{\bar{\beta}^\delta_T}{\sqrt{c_\mu}}
 \lVert a \rVert_{\VV_t^{-1}}  \right)\;. \label{eq:choosing_action_admissible}
\end{align}
\section{REGRET ANALYSIS WITH A SLIDING WINDOW}
In the main paper only the analysis with discount factors is discussed. 
However as in the linear bandit literature, the analysis with exponential weights and a sliding 
window share similarities, in particular they have the same form of guarantees for the regret. 
For the sake of completeness, we give a detailed analysis of the results achievable with a sliding window.
\label{section:regret_SW_appendix}

\subsection{Notation}
\label{subsection:notations_SW_appendix}
Let us first introduce the main notations. For any value of $\theta \in \mathbb{R}^d$,  we define,
\begin{equation}
\label{eq:H_t_SW}
    \HH_t(\theta) = \sum_{s=\max(1,t- \tau)}^{t-1} \dot{\mu}(a_s^{\top} \theta) a_s a_s^{\top} +
     \lambda \identity{d}\;.
\end{equation}
\begin{equation}
   \VV_t = \sum_{s=\max(1,t- \tau)}^{t-1} a_s a_s^{\top} + \frac{\lambda}
    {c_{\mu}} \identity{d}\;.
\end{equation}
\begin{equation}
    g_t(\theta) = \sum_{s=\max(1,t- \tau)}^{t-1} \mu(a_s^{\top} \theta) a_s + \lambda \theta\;.
\end{equation}
\begin{equation}
S_t = \sum_{s=\max(1, t- \tau)}^{t-1}  \epsilon_{s+1} a_s \;.
\end{equation}
For any $\theta_1, \theta_2 \in \mathbb{R}^d$,
\begin{align*}
    \alpha(a,\theta_1,\theta_2) = \int_{0}^1 \dot{\mu}(va^{\top} \theta_2+(1-v)
    a^\top  \theta_1)dv\;.
\end{align*}
\begin{equation}
\label{eq:G_t_SW}
\GG_t(\theta_1,\theta_2) =  \sum_{s=\max(1,t - \tau) }^{t-1} \alpha(a_s,\theta_1,\theta_2)a_s a_s^\top + \lambda
 \identity{d}\;.
\end{equation}
Let $\HH_t$ be defined as
\begin{equation}
	\label{eq:H_t_good_sliding}
    \HH_t = \sum_{s=\max(1,t- \tau)}^{t-1} \dot{\mu}(a_s^{\top} \theta^{\star}_s) a_s a_s^{\top} +
     \lambda \identity{d} \;.
\end{equation}
Let us define $\mathcal{T}(\tau)$ as
\begin{equation}
    \label{t_tau}
    \mathcal{T}(\tau) = \{ 1 \leq t \leq T, \forall s, \text{such that} \,  t- \tau \leq s \leq t-1,
     \theta^{\star}_s =
    \theta^{\star}_t\}\;.
\end{equation}
$t \in \mathcal{T}(\tau)$ when $t$ is a least $\tau$ steps away from
the closest previous breakpoint. 
When focusing on time instants in
 $\mathcal{T}(\tau)$ the bias due to non-stationarity disappears.
In the sliding window setting, we construct an estimator based on a truncated penalized log-likelihood. 
In this section, $\hat{\theta}_t$ is defined as the unique maximizer of
\begin{equation}
\label{eq:MLE_sliding_window}
\sum_{s=\max(1, t-\tau)}^{t-1} \log \mathbb{P}_{\theta}(r_{s+1} | a_s)
- \frac{\lambda}{2} \lVert \theta \rVert_2^2\;.
\end{equation}
By using the definition of the GLM and thanks to the concavity of this equation in $
\theta$, $\hat{\theta}_t$ is the unique solution of
$$
\sum_{s=\max(1, t- \tau)}^{t-1} (r_{s+1} - \mu(a_s^\top \theta))a_s - \lambda \theta 
= 0\;.
$$
This can be summarized with
\begin{align*}
    g_t(\hat{\theta}_t) &= \sum_{s= \max(1, t- \tau)}^{t-1} r_{s+1} a_s
    =  S_t + \sum_{s=\max(1, t- \tau)}^{t-1} \mu(a_s^\top \theta^\star_s) a_s\;.
\end{align*}

\subsection{Algorithm}
The $\SW$ algorithm proceeds as follows. First, based on the $\tau$ last rewards and actions,
$\hat{\theta}_t$ is computed using Equation \eqref{eq:MLE_sliding_window}.
Then, after receiving the action set $\mathcal{A}_t$ the action $a_t$ is chosen optimistically.
Finally, by proposing this action a reward $r_{t+1}$ is received and the design matrix is updated. 
The pseudo code of $\SW$ is reported in Algorithm \ref{alg:SW-GLM}.
\begin{algorithm}[H]
\caption{$\SW$}
   \label{alg:SW-GLM}
\begin{algorithmic}
   \STATE {\bfseries Input:} Probability $\delta$, dimension $d$, regularization $\lambda$,
   upper bound for bandit parameters $S$, sliding window $\tau$.
\STATE {\bfseries Initialize:} $\VV_0 = (\lambda/c_{\mu}) \identity{d}$, $\hat{\theta}_0 = 0_{\mathbb{R}^d}$.
   \FOR{$t=1$ {\bfseries to} $T$}
   \STATE Receive $\mathcal{A}_t$, compute $\hat{\theta}_t$ according to (\ref{eq:MLE_sliding_window})
   %\IF { $\hat{\theta}_t \in \Theta$}
   %\STATE $\beta_t^\delta = \beta_{t,1}^\delta$ with $\beta_{t,1}^\delta$ defined in Equation \eqref{eq:beta_1}
   %\ELSE
%   \STATE $\beta_t^\delta = \beta_{t,2}^\delta$ with $\beta_{t,2}^\delta$ defined in Equation \eqref{eq:beta_2}
 %  \ENDIF
   \STATE {\bfseries Play} $a_t = \argmax_{a \in \mathcal{A}_t} \mu(a^\top \hat{\theta}_t)  + \frac{\beta^\delta_t}{\sqrt{c_\mu}}
 \lVert a \rVert_{\VV_t^{-1}} $ with $\beta_{t}^\delta$ defined in Equation \eqref{beta_t_SW}
  \STATE {\bfseries Receive} reward $r_{t+1}$
  \STATE {\bfseries Update:} 
  \IF {$t < \tau$}
  \STATE $\VV_{t+1} \leftarrow a_t a_t^{\top} +
  \VV_{t}$
  \ELSE
  \STATE $\VV_{t+1} \leftarrow a_t a_t^{\top} - a_{t- \tau} a_{t-\tau}^{\top} +
  \VV_{t}$
  \ENDIF
   \ENDFOR
\end{algorithmic}
\end{algorithm}

\subsection{Analysis of the Regret of $\SW$}
In Section \ref{section:regret_D_appendix}, the self-concordance
is the key tool to obtain an analysis without using a projection step. 
In the next proposition, we link the matrix $\GG_t(\hat{\theta}_t, \theta^\star_t)$
 with $\HH_t(\theta^\star_t)$ independently from $\cm$.

\begin{prop}
\label{qlzkdlqkzdl}
When $\hat{\theta}_t$ is the maximum likelihood estimator as defined in Equation
\eqref{eq:MLE_sliding_window} and $t \in \mathcal{T}(\tau)$, we have:
$$
\alpha(a, \theta^\star_t, \hat{\theta}_t) \geq \left( 1+ S + \frac{1}{\sqrt{\lambda}} 
\lVert S_t \rVert_{\GG_t^{-1}(\theta^\star_t, \hat{\theta}_t)}\right)^{-1} \dot{\mu}(a^\top \theta^\star_t)\;.
$$
\end{prop}
Note that the main difference with Proposition \ref{prop:alpha_discount} is that $\bar{S}$ is now replaced by $S$.
This is due to the fact that the bias disappears when using a sliding window for $t \in \mathcal{T}(\tau)$.
\begin{proof}
Thanks to Lemma \ref{lemma:self_concordance}, we have:
\begin{align*}
    \alpha(a,\theta^\star_t,\hat\theta_t) &
    \geq \left(1+\left\vert
    a^{\top}(\theta^\star_t-\hat\theta_t)\right\vert\right)^{-1}\dot{\mu}(a^{\top}\theta^\star_t)
    \\
    & \geq \left(1+\left\vert
    a^{\top} \GG_t^{-1}(\theta^\star_t, \hat{\theta}_t)
    (g_t(\theta^\star_t)-g_t(\hat\theta_t)) \right\vert\right)^{-1}
    \dot{\mu}(a^{\top}\theta^\star_t)
	 \quad \textnormal{(Mean-Value Theorem)}
    \\
    &\geq \left(1+\left\lVert a\right\rVert_{
    \GG_t^{-1}(\theta^\star_t,\hat\theta_t)}\left\lVert
    g_t(\theta^\star_t) - g_t(\hat\theta_t) \right\rVert_{\GG_t^{-1}(\theta^\star_t,\hat\theta_t)}
    \right)^{-1
    }\dot{\mu
    }(a^\top \theta^\star_t) \quad \textnormal{(Cauchy–Schwarz)}\\
     &\geq \left(1+  \lambda^{-1/2}  \left\lVert
    g_t(\theta^\star_t) - g_t(\hat\theta_t) \right\rVert_{\GG_t^{-1}(\theta^\star
    _t,\hat\theta_t)}
    \right)^{-1
    }\dot{\mu
    }(a^\top \theta^\star_t) \quad \textnormal{($\GG_t(\theta^\star_t,\hat\theta_t)
    \geq \lambda \identity{d}$)}\\
    &\geq \left(1+\LL \lambda^{-1/2}\left\lVert S_t -
    \lambda\theta^\star_t\right\rVert_{\GG_t^{-1}
    (\theta^\star_t,\hat\theta_t)}\right)^{-1}\dot{\mu}
    (a^\top \theta^\star_t) \quad \text{($t \in \mathcal{T}(\tau)$)}\\
    &\geq \left(1+\LL S+ \LL \lambda^{-1/2}\left\lVert
    S_t\right\rVert_{\GG_t^{-1}(\theta^\star_t,\hat\theta_t)}\right)^{-1}\dot{\mu}(a^\top \theta^\star_t)
    \;.
\end{align*}
\end{proof}
\begin{cor}
\label{prop:self_concord_SW}
When $\hat{\theta}_t$ is the maximum likelihood estimator as
defined in Equation \eqref{eq:MLE_sliding_window},
 when $t \in \mathcal{T}(\tau)$ and
$\HH_t$ is defined in Equation \eqref{eq:H_t_good_sliding}, we have,
$$
 \GG_t(\theta^*_t,\hat{\theta}_t)
  \geq \left(1+  S + \frac{1}{\sqrt{\lambda}} \left\lVert
    S_t\right\rVert_{\GG_t^{-1}(\theta^*_t,\hat\theta_t)
    }\right)^{-1} \HH_t\;.
$$
Furthermore,
$$
\forall t \leq T, \, \lVert S_t \rVert_{\GG_t^{-1}(\theta^\star_t, \hat{\theta}_t)} \leq
 \sqrt{1+ \LL S}\left\lVert S_t\right\rVert_{\HH_t^{-1}
    }
    +
\frac{1}{\sqrt{\lambda}}
    \left
    \lVert S_t\right\rVert_{\HH_t^{-1}}^2 
   \;.
$$
\end{cor}
\begin{proof}
Using Proposition \ref{qlzkdlqkzdl}
and summing for time instants $s$ such that $\max(1,t-\tau) \leq s \leq t-1$,
\begin{align*}
\sum_{s=t-\tau}^{t-1} \alpha(a_s, \theta^\star_t, \hat{\theta}_t) a_s a_s^{\top}
& \geq
\left(1+  S +  \lambda^{-1/2}\left\lVert
    S_t\right\rVert_{\GG_t^{-1}(\theta^\star_t,\hat\theta_t)}\right)^{-1}
\sum_{s=t-\tau}^{t-1} \dot{\mu}(a_s^{\top} \theta^\star_s) a_s a_s^\top\;.
\end{align*}
Where we use $\theta^\star_s = \theta^\star_t$ for 
$ t-\tau \leq s \leq t-1$ thanks to the assumption $t\in \mathcal{T}(\tau)$. 
The next step
consists in adding the regularization term on both sides.
 Note that $ \left(1+ S +  \lambda^{-1/2}\left\lVert
    S_t\right\rVert_{\GG_t^{-1}(\theta^\star_t,\hat\theta_t)}\right)
     \lambda \geq \lambda$
     and obtain,
    $$
    \GG_t(\theta^\star_t,\hat{\theta}_t)
  \geq \left(1+  S +  \lambda^{-1/2}\left\lVert
    S_t\right\rVert_{\GG_t^{-1}(\theta^\star_t,\hat\theta_t)
    }\right)^{-1} \HH_t\;.
    $$ 
This in turn implies,
\begin{align*}
    &\left\lVert S_t\right\rVert_{\GG_t^{-1}(\theta^\star_t,\hat\theta_t)}^2 \leq \left(1+ S +
 \lambda^{-1/2}\left\lVert S_t\right\rVert_{\GG_t^{-1}
 (\theta^\star_t,\hat\theta_t)}\right
 )\left
    \lVert S_t\right\rVert_{\HH_t^{-1}}^2
    \\
    &\Longleftrightarrow \left\lVert S_t\right\rVert_{\GG_t^{-1}(\theta^\star_t,\hat
    \theta_t)}^2 - \LL \lambda^{-1/2}\left\lVert S_t\right\rVert_{\HH_t^{-1}}
    ^2 \left\lVert S_t\right\rVert_{\GG_t^{-1}(\theta^\star_t,\hat\theta_t)} - (1+ \LL S)
    \left\lVert S_t\right\rVert_{\HH_t^{-1}}^2 \leq 0\;.
\end{align*}

Solving this polynomial inequality (in $\left \lVert S_t\right\rVert_{\GG_t^{-1}
(\theta^\star_t,
\hat\theta_t)}$) finally gives,
\begin{align*}
    \left\lVert S_t\right\rVert_{\GG_t^{-1}(\theta^\star_t,\hat\theta_t)} \leq 
\sqrt{1+ \LL S}\left\lVert S_t\right\rVert_{\HH_t^{-1}}
+    
    \frac{1}{\sqrt{\lambda}}  \left \lVert S_t\right\rVert_{\HH_t^{-1}}^2   
    \;.
\end{align*}
\end{proof}
Using this technique, we have established an explicit link between
 $\GG_t(\theta^\star_t,\hat\theta_t)$ and $\HH_t$ without the need to project $
\hat{\theta}_t$ on $\Theta$ when $t \in \mathcal{T}(\tau)$.

We define
\begin{equation}
\label{eq:rho_SW}
\rho_t^\delta = \left( \frac{\sqrt{\lambda}} {2 \rM} + \frac{2
\rM
  }{\sqrt{\lambda} }\log\left(\frac{T}{\delta}\right) +
 \frac{d \rM
  }{ \sqrt{\lambda} }\log\left(1 + \frac{k_\mu \min(t, \tau)}{d \lambda}
  \right)+\frac{2 \rM}{\sqrt{\lambda}}
  d\log(2) \right)\;,
\end{equation}
and
\begin{equation}
\label{beta_t_SW}
\beta_{t}^\delta = \km \sqrt{\lambda} \left( 1 + S  + \sqrt{\frac{1 + S}{\lambda}}
\rho_t^\delta + \left( \frac{\rho_t^\delta}{\sqrt{\lambda}} \right)^2   \right)^{3/2 }\;.
\end{equation} 
In the next proposition, we give an upper bound for $\Delta_t(a, \hat{\theta}_t)$.
\begin{prop}
For any $\delta \in (0,1]$, with probability higher
than $1-\delta$,
$$
\forall t \in \mathcal{T}(\tau), \;
\Delta_t(a, \hat{\theta}_t) \leq \frac{\beta_t^\delta}{\sqrt{c_\mu} }
\lVert a \rVert_{\VV_t^{-1}} \;.
$$
\end{prop}
\begin{proof}
\begin{align*}
    \Delta_t(a, \hat{\theta}_t) &= |\mu(a^{\top} \theta^{\star}_t) -
    \mu(a^{\top} \hat{\theta}_t)| 
    \leq k_{\mu} | a^{\top} (\theta^{\star}_t - \hat{\theta}_t) | \\
    & = k_{\mu} | a^{\top} \GG_t^{-1}(\theta^{\star}_t, \hat{\theta}_t)
     (g_t(\theta^{\star}_t) - g_t(\hat{\theta}_t)) |
     \quad \textnormal{(Mean-Value Theorem)}
     \\
    & \leq k_{\mu} \lVert a \rVert_{\GG_t^{-1}(\theta^{\star}_t, \hat{\theta}_t)}
    \lVert g_t(\theta^{\star}_t) - g_t(\hat{\theta}_t) \rVert_{
    \GG_t^{-1}(\theta^{\star}_t, \hat{\theta}_t)} \quad \text{(Cauchy-Schwarz ineq.)}
    \\
    & \leq k_{\mu} \lVert a \rVert_{\GG_t^{-1}(\theta^{\star}_t, \hat{\theta}_t)}
    \lVert S_t - \lambda \theta^\star_t \rVert_{
    \GG_t^{-1}(\theta^{\star}_t, \hat{\theta}_t)} \;. \quad (t \in \mathcal{T}(\tau))
\end{align*}

We can use Corollary \ref{prop:self_concord_SW} to link 
$\lVert a \rVert_{\GG_t^{-1}
(\theta^\star_t, \hat{\theta}_t)}$ with $\lVert a \rVert_{\HH_t^{-1}}$.
\begin{align*}
    \Delta_t(a, \hat{\theta}_t) &\leq k_\mu   \sqrt{1+S+\frac{1}{\sqrt{\lambda}}
    \left\lVert
    S_t\right\rVert_{\GG_t^{-1}(\theta^\star_t,\hat\theta_t)
    }}\lVert a \rVert_{\HH_t^{-1}} \left(\sqrt{\lambda} S +  \lVert S_t
     \rVert_{\GG_t^{-1}(\theta^\star_t, \hat{\theta}_t)} \right) \\
     &\leq k_\mu  \sqrt{\lambda} \sqrt{1+S+\frac{1}{\sqrt{\lambda}}
    \left\lVert
    S_t\right\rVert_{\GG_t^{-1}(\theta^\star_t,\hat\theta_t)
    }}\lVert a \rVert_{\HH_t^{-1}} \left(S +  \frac{1}{\sqrt{\lambda}}\lVert S_t
     \rVert_{\GG_t^{-1}(\theta^\star_t, \hat{\theta}_t)} \right)\\
     &\leq k_\mu  \sqrt{\lambda} \left(1+S+\frac{1}{\sqrt{\lambda}}
    \left\lVert
    S_t\right\rVert_{\GG_t^{-1}(\theta^\star_t,\hat\theta_t)
    }\right)^{3/2}
    \lVert a \rVert_{\HH_t^{-1}} \;.
\end{align*}
Then, using Corollary \ref{prop:self_concord_SW} we can upper bound 
$\lVert S_t \rVert_{\GG_t^{-1}(\theta^\star_t, \hat{\theta}_t)}$ with a combination of terms depending
on
$\lVert S_t \rVert_{\HH_t^{-1}}$. Recall that Corollary \ref{cor:concentration_SW}
 gives with probability higher than $1-\delta$ ,
for all $t$ in $\mathcal{T}(\tau)$, 
$\lVert S_t \rVert_{\HH_t^{-1}} \leq \rho_t^\delta$.
\begin{align*}
\Delta_t(a, \hat{\theta}_t) \leq \km \sqrt{\lambda} 
\left(
1+ S + \sqrt{\frac{1 + S}{\lambda}} \rho_t^\delta
+ \frac{1}{\lambda} (\rho_t^\delta)^2
 \right)^{3/2} \lVert a \rVert_{\HH_t^{-1}}\;.
\end{align*}
The proof is completed using $\HH_t \geq \cm \VV_t$, which
 holds thanks to Assumption \ref{ass:bounded_actions} on the bandit parameters.
\end{proof}
Finally, we give an upper bound for the regret enjoyed by $\SW$.
\begin{thm}
\label{thm:regret_wost_case_SW}
The regret of the $\SW$ algorithm is bounded
with probability at least $1-\delta$ by,
\begin{equation*}
R_T \leq \Gamma_T \tau + 
 \frac{2\beta^\delta_T}{\sqrt{c_\mu}}\sqrt{dT}  
 \sqrt{\lceil T/\tau \rceil   }
 \sqrt{2\max \left(1,\frac{1}{\lambda}
 \right)}
\sqrt{\log \left(1 + \frac{\tau}{d \lambda} \right) } \;,
\end{equation*}
where $\beta_t^\delta$ is defined in Equation \eqref{beta_t_SW}.
\end{thm}
\begin{proof}
The proof essentially follows the steps of the proof of Theorem
 \ref{th:regret_D_no_proj_main}. The main difference is that $\beta_t^\delta$
 from Equation \eqref{beta_t_SW} is used and the elliptical lemma is different
  because the design matrix is designed with a sliding window instead of weights.
  
Applying Proposition \ref{prop:delta_t_regret} when $t \in \mathcal{T}(\tau)$, with
probability higher than $1-\delta$,
\begin{equation}
\label{kzjnnjn}
r_t \leq 
 \frac{2}{\sqrt{c_\mu}} \beta_t^\delta
\lVert a_t \rVert_{\VV_t^{-1}}
\;.
\end{equation}
The dynamic regret can then be upper bounded by, 
\begin{align*}
R_T &= \sum_{t=1}^T r_T = \sum_{t \in \mathcal{T}(\tau)} r_t + \sum_{t \notin
 \mathcal{T}(\tau)} r_t 
\leq \Gamma_T \tau + \sum_{t \in \mathcal{T}(\tau)} r_t \\
& \leq \Gamma_T \tau + 
\frac{2\beta_T^\delta}{\sqrt{c_\mu}} \sum_{t \in
 \mathcal{T}(\tau)} \lVert a_t \rVert_{\VV_t^{-1}} \quad
\text{(Equation \eqref{kzjnnjn})} \\
& \leq \Gamma_T \tau 
+ \frac{2 \beta_T^\delta}{\sqrt{c_\mu}} \sqrt{T}
\sqrt{\sum_{t \in \mathcal{T}(\tau)}\lVert a_t \rVert_{\VV_t^{-1}}^2 } \quad \textnormal{(Cauchy-Schwarz ineq.)} \\
& \leq \Gamma_T \tau
+ \frac{2 \beta_T^\delta}{\sqrt{c_\mu}} \sqrt{T}
\sqrt{\sum_{t=1}^T\lVert a_t \rVert_{\VV_t^{-1}}^2 } \\
& \leq  \Gamma_T \tau + 
+ \frac{2 \beta^\delta_T}{\sqrt{c_\mu}} \sqrt{dT} 
 \sqrt{\lceil T/\tau \rceil   }
\sqrt{2\max\left(1,\frac{1}{\lambda}
\right)
        \log \left(1 + \frac{\tau}{d \lambda} \right) }\;. \quad \text{(Lemma \ref{lemma:ellipticalpotential_SW})}
\end{align*}
\end{proof}

\begin{cor}[Asymptotic bound]
 \label{corollary:asympt_regret_no_proj_SW}
If $\Gamma_T$ is known, by choosing
   $\tau = \left(\frac{dT}{ \cm^{1/2} \Gamma_T}\right)^{2/3}$
   and $\lambda = d \log(T)$,
   the regret of $\SW$ scales as 
   $$ R_T = \widetilde{\mathcal{O}}(\cm^{-1/3}
   d^{2/3} \Gamma_T^{1/3} T^{2/3})\;. $$
If $\Gamma_T$ is unknown, by choosing
   $\tau = \left(\frac{dT}{\cm^{1/2}}\right)^{2/3}$,
   the regret of $\SW$ scales as
   $$ R_T = \widetilde{\mathcal{O}}(\cm^{-1/3}
   d^{2/3} \Gamma_T T^{2/3}) \;.$$
\end{cor}
\begin{proof}
When $\Gamma_T$ is known, 
we set
$\lambda = d\log(T)$ and $\tau = \left(\frac{d T}{\sqrt{\cm} \Gamma_T}\right)^{2/3}$.
With those choices, 
\begin{enumerate}
\item $\beta_T^\delta$ scales as $\sqrt{d\log(T)}$.
\item $\Gamma_T \tau $ scales as $\widetilde{\mathcal{O}}(\cm^{-1/3} d^{2/3}
 \Gamma_T^{2/3} T^{2/3})$.
\item $\frac{\beta_T^\delta}{ \sqrt{\cm}} \sqrt{T} \sqrt{d \frac{T}{\tau}}$ 
scales as  $\widetilde{\mathcal{O}}(\cm^{-1/3} d^{2/3} \Gamma_T^{1/3} T^{2/3})$.
\end{enumerate}
The proof is similar when $\Gamma_T$ is unknown.
\end{proof}
When the reward gaps are bounded from below 
we can obtain the following gap-dependent upper bound:
\begin{thm}
  \label{thm:fjqiqposejf_SW}
Under Assumption \ref{ass:minimum_gap}, when
setting $\tau = \frac{d \sqrt{T}}{\sqrt{\cm \Gamma_T}}$ the regret of the $\SW$ algorithm
satisfies:
 $$
 R_T = \widetilde{\mathcal{O}}
 \big(\Delta^{-1} \cm^{-1/2} d \sqrt{ \Gamma_T T}\big)\;.
 $$
\end{thm}
\begin{proof}
First note that for any suboptimal action 
$a \in \mathcal{A}_t$, 
$$
\mu(a_{\star,t}^\top \theta^\star_t) - \mu(a^\top \theta^\star_t) \geq \Delta\;.
$$
This implies
\begin{equation}
\label{eq:lien_r_t_delta_SW}
r_t =  \mu(a_{\star,t}^\top \theta^\star_t) 
- \mu(a_t^\top \theta^\star_t) \leq \frac{\left(\mu(a_{\star,t}^\top \theta^\star_t) 
- \mu(a_t^\top \theta^\star_t)\right)^2}{\Delta} = \frac{r_t^2}{\Delta}\;.
\end{equation}
Using Proposition \ref{prop:delta_t_regret} one has,
$$
r_t \leq 
 \frac{2}{\sqrt{c_\mu}} \beta_t^\delta
\lVert a_t \rVert_{\VV_t^{-1}}
\;.
$$
The dynamic regret can then be upper bounded by, 
\begin{align*}
R_T &\leq \Gamma_T \tau + \frac{1}{\Delta}\sum_{t \in \mathcal{T}(\tau)} r_t^2
\quad (\text{Equation } \eqref{eq:lien_r_t_delta_SW}) \\
&\leq \Gamma_T \tau + \frac{4 (\beta_T^\delta)^2}{ \cm \Delta} 
\sum_{t=1}^T \lVert a_t \rVert_{\VV_t^{-1}}^2 \\
&\leq \Gamma_T \tau + \frac{8  (\beta_T^\delta)^2}{\cm \Delta} 
 \max\left(1, \frac{1}{\lambda} \right) d \lceil T/\tau \rceil 
 \log\left(1 + \frac{\tau}{\lambda d} \right)\;. \quad \text{(Lemma \ref{lemma:ellipticalpotential_SW})}
\end{align*}
We set
$\lambda = d\log(T)$ and $\tau = \frac{d \sqrt{T}}{\sqrt{\cm \Gamma_T}}$.
With those choices,
\begin{enumerate}
\item  $\beta_T^\delta$ scales as $\sqrt{d\log(T)}$.
\item  $\Gamma_T \tau $ scales as
 $\widetilde{\mathcal{O}}(\cm^{-1/2} d \Gamma_T^{1/2} T^{1/2})$.
 \item $\frac{(\beta_T^\delta)^2}{\cm} d \frac{T}{\tau}$ scales as 
 $\widetilde{\mathcal{O}}(\cm^{-1/2} d \Gamma_T^{1/2} T^{1/2})$.
\end{enumerate}
Dividing by $\Delta$ yields the announced result.
\end{proof}
When $\hat{\theta}_t$ is in $\Theta$ it is also possible with a sliding window to
 obtain a usually better concentration result. This discussion is not reported here, 
 but can be easily adapted from Proposition \ref{proposition:beta_improved}.
%\clearpage
%\newpage
\section{USEFUL RESULTS}
\label{section:useful_results}
\subsection{Self-Concordant Properties}
\label{subsection:SC}
In this section we state the main properties and lemma that can be obtained with
the
self-concordance assumption. 
\begin{lemma}[Lemma 9 in \cite{faury2020improved}]
\label{lemma:self_concordance}
For any $z_1,z_2\in\mathbb{R}$, we have the following inequality
$$
    \dot{\mu}(z_1)\frac{1-\exp(-\vert z_1-z_2\vert)}{\vert z_1-z_2\vert} \leq \int_{0}^1
     \dot{\mu}(z_1+v(z_2-z_1))dv \leq  \dot{\mu}(z_1)\frac{\exp(\vert z_1-z_2\vert)-1}
     {\vert z_1-z_2\vert}\;.
$$
Furthermore,
$$
    \int_{0}^1 \dot{\mu}(z_1+v(z_2-z_1))dv \geq \dot{\mu}(z_1)(1+\vert z_1-
    z_2\vert)^{-1}\;.
$$
\end{lemma}
Thanks to the self-concordance property we have an interesting relation between
$\GG_t(\theta_1, \theta_2)$ and $\HH_t(\theta_1)$ or $\HH_t(\theta_2)$ when both
$\theta_1$ and $\theta_2 \in \Theta$. This relation is made explicit in the next
lemma.
\begin{lemma}[Self-concordance and sliding window]
\label{lemma:boundGtbyHt_SW}
For all $\theta_1,\theta_2\in\Theta$, with $\GG_t$ defined in Equation \eqref{eq:G_t_SW} and $\HH_t$
defined in Equation \eqref{eq:H_t_SW} the following inequalities hold
\begin{align*}
    \GG_t(\theta_1,\theta_2) \geq (1+2 \LL S)^{-1}\HH_t(\theta_1)\;, \quad
    \GG_t(\theta_1,\theta_2) \geq (1+2 \LL S)^{-1}\HH_t(\theta_2)\;.
\end{align*}
\end{lemma}
\begin{proof}
Applying Lemma \ref{lemma:self_concordance}, for any $\theta_1, \theta_2 \in \mathbb{R}^d$,
$$
\alpha(a, \theta_1, \theta_2) \geq \frac{\dot{\mu}(a^{\top} \theta_1)}{1+ |a^{\top}
(\theta_1- \theta_2)|} \quad \text{and} \quad
\alpha(a, \theta_1, \theta_2) \geq \frac{\dot{\mu}(a^{\top} \theta_2)}{1+ |a^{\top}
(\theta_1- \theta_2)|}\;.
$$
Furthermore, if $\theta_1$ and $\theta_2 \in \Theta$, then
$$
|a^{\top} (\theta_1 - \theta_2) | \leq 2 \LL S\;.
$$
\end{proof}
\begin{lemma}[Self-concordance and discount factors]
\label{lemma:boundGtbyHt_D}
For all $\theta_1,\theta_2\in\Theta$, with $\widetilde{\HH}_t(\theta_1)$ defined in Equation
\eqref{eq:H_t_tilde_theta_D} and $\widetilde{\GG}_t(\theta_1, \theta_2)$ defined in Equation \eqref{eq:G_t_tilde_D}
the following inequalities hold:
\begin{align*}
    \widetilde{\GG}_t(\theta_1,\theta_2) \geq (1+2\LL S)^{-1} \widetilde{\HH}_t(\theta_1)\;, \quad
    \widetilde{\GG}_t(\theta_1,\theta_2) \geq (1+2\LL S)^{-1}  \widetilde{\HH}_t(\theta_2)\;.
\end{align*}
\end{lemma}
\begin{proof}
Same arguments than for Lemma \ref{lemma:boundGtbyHt_SW}
\end{proof}
\subsection{Determinant Inequalities}
\begin{prop}[Determinant inequality]
\label{ineq_det_DLINUCB}
Let $(\lambda_t)_t$ be a deterministic sequence of regularization parameters. Let $
\HH_t  = \sum_{s=1}^{t-1} w_s^2 \sigma_s^2 a_s  a_s^{\top} +
 \lambda_{t-1} \identity{d}$.
Under the Assumption \ref{ass:bounded_actions} and $\forall t, \sigma_t^2
 \leq k_\mu$,
 the following holds
$$
\det(\HH_t) \leq  \left(\lambda_{t-1} + \frac{ \km \sum_{s=1}^t w_s^2
}{d}
\right)^d\;.
$$
\end{prop}
\begin{proof}
\begin{align*}
\det(\HH_t) &= \prod_{i=1}^d l_i \quad (l_i \textnormal{ are the
 eigenvalues})
%  \\&
  \leq \left(\frac{1}{d}\sum_{i=1}^d l_i \right)^d  \quad (\textnormal{AM-GM inequality})\\
&\leq \left(\frac{1}{d} \textnormal{trace}(\HH_t) \right)^d
\leq \left(\frac{1}{d} \sum_{s=1}^{t-1} w_s^2 \sigma_s^2 \textnormal{trace}(a_s
 a_s^{\top}) + \lambda_{t-1}\right)^d \\
&\leq \left(\frac{1}{d} \sum_{s=1}^{t-1} w_s^2 \sigma_s^2\Vert a_s \Vert_2^2 +
 \lambda_{t-1} \right)^d
\leq \left(\lambda_{t-1} + \frac{ k_\mu}{d} \sum_{s=1}^{t-1} w_s^2 \right)^d\;.
\end{align*}
\end{proof}
\begin{cor}
\label{corollary:inequality_determinant}
In the specific case where the weights are given by $w_t = \gamma^{-t}$ with $0<
  \gamma <1$,  under the same assumptions than Proposition \ref{ineq_det_DLINUCB},
  with $\wHH_t = \sum_{s=t- t_0}^{t-1} \gamma^{2(t-1-s)} \sigma_s^2 a_s a_s^\top + \lambda
   \identity{d}$, one has
\end{cor}
$$
\det(\wHH_t) \leq
 \left(\lambda + \frac{k_\mu
 (1 - \gamma^{2t_0})}{d(1-\gamma^2)} \right)^d \;.
$$
\begin{cor}
\label{corollary:inequality_determinant_V}
In the specific case where the weights are given by $w_t = \gamma^{-t}$ with $0<
  \gamma <1$, under Assumption \ref{ass:bounded_actions}
  with $\VV_t = \sum_{s=1}^{t-1} \gamma^{t-1-s} a_s a_s^\top + \lambda
   \identity{d}$, one has
   $$
   \det(\VV_t) \leq \left( \lambda + \frac{1- \gamma^{t-1}}{d(1- \gamma)}\right)^d   \;.
   $$
\end{cor}
\begin{cor}
\label{corollary:inequality_determinant_SW}
In the specific case where the weights are given by $w_t = 1$ when $t\geq t-\tau$
and $0$ before. 
With $\HH_t = \sum_{s= \max(1,t- \tau)}^{t-1} \sigma_s^2 a_s a_s^\top + \lambda
  \identity{d}$, one has
\end{cor}
$$
\det(\HH_t) \leq
 \left(\lambda + \frac{k_\mu \min(t, \tau)}{d}
 \right)^d \;.
$$
\subsection{Elliptical Lemma}
The following lemma is a version of the Elliptical Lemma when
discount factors are used.
 It comes from Proposition 4 in \citep{russac2019weighted} 
 and is stated here for the sake of completeness.
\begin{lemma}[Elliptical potential with discount factors (based on Proposition 4 in \citet{russac2019weighted})]
    Let $\{a_s\}_{s=1}^\infty$ a sequence in $\mathbb{R}^d$ such that $\lVert
    a_s \rVert_2 \leq 1$ for all $s \in \mathbb{N}$, and  let $\lambda$
    be a non-negative scalar. For $t\geq 1$ define $\VV_t  = \sum_{s=1}^{t-1} \gamma^{t-1-s} a_s a_s^\top+ 
	\lambda \identity{d}$, the following inequality holds
    \begin{align*}
        \sum_{t=1}^{T}  \left\lVert a_t\right\rVert_{\VV_t^{-1}}^2
        \leq 
        2\max\left(1,\frac{1}{\lambda}\right)
        \log \left( \frac{\det( \VV_{T+1})}{\lambda^d \gamma^{dT}} \right) \;.
    \end{align*}
    %  &=
    %    \sum_{t=1}^{T} \gamma^{-(t-1)} \left\lVert a_t\right\rVert_{\mbold{W}_t^{-1}}^2 \\
     %   & \leq 2\max\left(1,\frac{1}{\lambda}\right)
     %   \log \left( \frac{\det( \mbold{W}_{T+1})}{\lambda^d} \right)=
\label{lemma:ellipticalpotential_Discount}
\end{lemma}

\begin{proof}
In the proof we introduce the matrix 
$\mbold{W}_t  = \sum_{s=1}^{t-1} \gamma^{-s} a_s a_s^\top+ \gamma^{-
    (t-1)}\lambda \identity{d}$ such that $\VV_t = \gamma^{t-1}\mbold{W}_t$. We have,
\begin{align*}
\mbold{W}_t &= \sum_{s=1}^{t-1} \gamma^{-s} a_s a_s^\top + \gamma^{-(t-1)} \lambda \identity{d}  \\
& = \gamma^{-(t-1)} a_{t-1} a_{t-1}^\top + \sum_{s=1}^{t-2} \gamma^{-s} a_s a_s^\top + \gamma^{-(t-2)} \lambda 
\identity{d}
+  \gamma^{-(t-1)} \lambda \identity{d}  -  \gamma^{-(t-2)} \lambda \identity{d}  \\
& =  \gamma^{-(t-1)} a_{t-1} a_{t-1}^\top + \gamma^{-(t-1)} (1- \gamma) \lambda \identity{d} + \mbold{W}_{t-1} \\
&\geq \gamma^{-(t-1)} a_{t-1} a_{t-1}^\top + \mbold{W}_{t-1} 
\geq \mbold{W}_{t-1}^{1/2} (\identity{d} + \gamma^{-(t-1)} \mbold{W}_{t-1}^{-1/2} a_{t-1} a_{t-1}^\top 
\mbold{W}_{t-1}^{-1/2}) \mbold{W}_{t-1}^{1/2}\;.
\end{align*}
This implies,
\begin{align*}
\det( \mbold{W}_{t+1}) &\geq \det( \mbold{W} _{t}) \det
 \left( \identity{d} + (\gamma^{-t/2} \mbold{W}_{t}^{-1/2} a_{t}) (\gamma^{-t/2} \mbold{W}_{t}^{-1/2}
a_{t})^{\top} \right) \\
&\geq \det(  \mbold{W}_{t}) \left( 1+ 
\gamma^{-t} \lVert a_{t}\rVert_{ \mbold{W}_{t}^{-1}}^2 \right) 
\quad (\det( \identity{d} + x x^\top)= 1
+ \lVert x \rVert_2^2 )\;.
\end{align*}
This in turn gives,
\begin{align*}
\frac{\det( \mbold{W}_{T+1})}{\det( \mbold{W}_1)} =
 \prod_{t=1}^T \frac{\det( \mbold{W}_{t+1})}{\det( \mbold{W}_{t})} \geq
 \prod_{t=1}^T \left( 1+ \gamma^{-t} \lVert a_{t}\rVert_{ \mbold{W} _{t}^{-1}}^2 \right)\;.
\end{align*}
Taking the logarithm on both sides gives:
\begin{align*}
\log \left(\frac{\det( \mbold{W}_{T+1})}{\lambda^d} \right) &\geq
 \sum_{t=1}^T \log(1 + \gamma^{-t} \lVert a_t \rVert^2_{ \mbold{W}_t^{-1}}) 
\geq \sum_{t=1}^T \log(1 + \gamma^{-(t-1)} \lVert a_t \rVert^2_{ \mbold{W}_t^{-1}}) \\
& \geq \sum_{t=1}^T \log \left( 1 + \frac{\gamma^{-(t-1)}   \lVert a_t \rVert^2_{\mbold{W}_t^{-1}}}
{\max\left(1, \frac{1}{\lambda} \right)}
\right)\; .
\end{align*}
Next, by using $\mbold{W}_t \geq \gamma^{-(t-1)} \lambda \identity{d}$, we see that
$$
\gamma^{-(t-1)} \lVert a_t \rVert^2_{\mbold{W}_t^{-1}} \leq \frac{1}{\lambda} \;.
$$
Which ensures that 
$$ 0 \leq  \frac{\gamma^{-(t-1)}   \lVert a_t \rVert^2_{\mbold{W}_t^{-1}}}
{\max\left(1, \frac{1}{\lambda} \right)} \leq 1\;.$$
Finally, with $\log(1+ x) \geq x/2$ valid when $ 0 \leq x \leq 1$. We get,
 \begin{align*}
 \log \left(\frac{\det(\mbold{W}_{T+1})}{\lambda^d} \right) 
 \geq \frac{1}{2\max\left(1, \frac{1}{\lambda} \right)}
  \sum_{t=1}^T \gamma^{-(t-1)} \lVert a_t \rVert^2_{\mbold{W}_t^{-1}} \;.
 \end{align*}
\end{proof}
The following lemma is a version of the Elliptical Lemma when
a sliding window is used and can be
 extracted from \cite[Proposition 9]{russac2019weighted}.
The proof is included here for the sake of completeness.
\begin{lemma}[Elliptical potential with sliding window (Proposition 9 in
 \cite{russac2019weighted})]
    Let $\{a_s\}_{s=1}^\infty$ a sequence in $\mathbb{R}^d$ such that $\lVert
    a_s \rVert_2 \leq 1$ for all $s \in \mathbb{N}$, and  let $\lambda$
    be a non-negative scalar. For $t\geq 1$ define $\VV_t  = \sum_{s=\max(1, t- \tau)
    }^{t-1} a_s a_s^\top+\lambda \identity{d}$. The following inequality holds:
    $$
        \sum_{t=1}^{T} \left\lVert a_t\right\rVert_{\VV_t^{-1}}^2 \leq 2 d \max
        \left(1,\frac{1}
        {\lambda} \right)
         \ceil{T/ \tau} \log\left( 1+ \frac{\tau}{\lambda d} \right)\; .
   $$
\label{lemma:ellipticalpotential_SW}
\end{lemma}
\begin{proof}
We start by rewriting the sum as follows.
\begin{align*}
\sum_{t=1}^T  \lVert a_t \rVert_{\VV_{t}^{-1}}^2 =
 \sum_{k=0}^{\ceil{T/\tau}-1} \sum_{t= k \tau + 1}^{(k+1)\tau}  \lVert a_t
 \rVert_{\VV_{t}^{-1}}^2\;.
\end{align*}
For the $k$-th block of length $\tau$ we define the matrix
$\mbold{W}_t^{(k)} = \sum_{s= k \tau +1}^{t-1} a_s a_s^{\top} + \lambda \identity{d}$.
We also have $\forall t \in [\![ k\tau + 1, (k+1) \tau]\!], \VV_t \geq \mbold{W}_t^{(k)}$ as every term in 
$\mbold{W}_t^{(k)}$ is
contained in $\VV_t$ and the extra-terms in $\VV_t$ correspond
 to positive definite matrices. 
$$
\sum_{k=0}^{\ceil{T/\tau}-1} \sum_{t= k \tau + 1}^{(k+1)\tau}
\lVert a_t\rVert_{\VV_{t}^{-1}}^2
   \leq \sum_{k=0}^{\ceil{T/\tau}-1} \sum_{t= k \tau + 1}^{(k+1)\tau}
\lVert a_t\rVert_{(\mbold{W}_{t}^{(k)})^{-1}}^2\;.
$$
Furthermore, $\forall t \in[\![ k \tau + 1 , (k+1) \tau]\!] $
 we have,
$$
\det(\mbold{W}_{t+1}^{(k)}) = \det( \mbold{W}_{t}^{(k)}) \left( 1 +
 \lVert a_t \rVert_{( \mbold{W}_{t}^{(k)})^{-1}}^2 \right)\;.
$$
With positive definitive matrices whose determinants
are strictly positive,
 this implies that
$$
\frac{\det(\mbold{W}_{(k+1)\tau +1}^{(k)})}{\det( \mbold{W}_{k \tau + 1}^{(k)})} =
\prod_{t=k\tau +1}^{(k+1) \tau }
 \frac{\det( \mbold{W}_{t+1}^{(k)})}{\det( \mbold{W}_{t}^{(k)})} =
 \prod_{t=k\tau +1}^{(k+1) \tau }  \left(1 +
  \lVert a_t\rVert_{ ( \mbold{W}_{t}^{(k)})^{-1}}^2\right) \;.
$$
By definition we have $\mbold{W}_{k \tau+1}^{(k)} = \sum_{t=k \tau +1 }^{k \tau}
a_t a_t^\top + \lambda \identity{d} = \lambda \identity{d}$.
\begin{align*}
\log \left( \frac{\det \left(\mbold{W}_{(k+1)\tau + 1}^{(k)}   \right)}{\lambda^d} \right)
&= \sum_{t= k \tau + 1 }^{(k+1)\tau} \log \left(1 +
 \lVert a_t \rVert_{( \mbold{W}_{t}^{(k)})^{-1}}^2  \right) \\
 & \geq \sum_{t= k \tau + 1 }^{(k+1)\tau} \log \left(1 + \frac{1}{\max(1, 1/\lambda)}
 \lVert a_t \rVert_{( \mbold{W}_{t}^{(k)})^{-1}}^2  \right) \; .
\end{align*}
In the next step we use,  $\forall 0 \leq x \leq 1, \log(1+ x) \geq x/2$.
 \begin{align*}
 \log \left( \frac{\det \left( \mbold{W}_{(k+1)\tau + 1}^{(k)}   \right)}{\lambda^d} \right)
& \geq \frac{1}{2 \max(1, 1/\lambda)}
\sum_{t= k \tau +1}^{(k+1) \tau} \lVert a_t \rVert_{( \mbold{W}_{t}^{(k)})^{-1}}^2\;.
 \end{align*}
 By summing, over the different blocks, we obtain
 \begin{align*}
 \sum_{k=0}^{\ceil{T/\tau}-1} \sum_{t= k \tau + 1}^{(k+1)\tau}
\lVert a_t\rVert_{\VV_{t}^{-1}}^2
  & \leq \sum_{k=0}^{\ceil{T/\tau}-1} \sum_{t= k \tau + 1}^{(k+1)\tau}
\lVert a_t\rVert_{( \mbold{W}_{t}^{(k)})^{-1}}^2 \\
& \leq 2 \max(1, 1/ \lambda)  \sum_{k=0}^{\ceil{T/\tau}-1}  \log \left( \frac{\det
 \left( \mbold{W}_{(k+1)\tau + 1}^{(k)}   \right)}{\lambda^d} \right)\;.
 \end{align*}

 Then, we upper bound $\det( \mbold{W}_{(k+1)\tau + 1}^{(k)})$ using similar arguments
 than for Corollary \ref{corollary:inequality_determinant_SW},
$$
\det( \mbold{W}_{(k+1)\tau+1}^{(k)}) \leq \left(  \lambda + \frac{\tau}{d}  \right)^d.
$$
Applying the logarithm function on both sides concludes the proof.
\end{proof}

\subsection{Link Between $\Delta_t$ and the Instantaneous Regret}
For any optimistic algorithm, even in a non-stationary environment the
instantaneous regret can be directly related to $\Delta_t(a, \theta)$
defined as
$$
\Delta_t(a,\theta) = |\mu(a^{\top} \theta) - \mu(a^{\top} \theta^\star_t)| \;.
$$
\begin{prop}[Based on Lemma 14 in \cite{faury2020improved}]
\label{prop:delta_t_regret}
Consider any optimistic algorithm in a possibly non-stationary environment
 such that the exploration bonus for action $a$ at time $t$ is defined by $\beta_t(a)$.
 Let $\theta_t$ be the estimator used at time $t$ by the algorithm to compute the
  UCB, i.e.
 $UCB_t(a) = \mu(a^\top \theta_t) + \beta_t(a)$.
 Under the assumption $\Delta_t(a, \theta_t) \leq \beta_t(a)$, the following inequality
 holds
 $$
 r_t \leq 2 \beta_t(a_t)\;.
 $$
\end{prop}
\begin{proof}
Let $a_{t,\star} = \argmax_{a \in \mathcal{A}_t} \mu(a^{\top} \theta^{\star}_t)$
\begin{align*}
r_t & = \mu(a_{t,\star}^{\top} \theta^{\star}_t) - \mu(a_t^{\top} \theta^\star_t) 
 \leq  |\mu(a_{t, \star}^\top \theta^{\star}_t)  -  \mu(a_{t, \star}^\top \theta_t) | +
 \mu(a_{t, \star}^\top \theta_t) -
\mu(a_t^{\top} \theta_t) + |\mu(a_t^{\top} \theta_t)  -  \mu(a_t^{\top} \theta^
\star_t) |
 \\
 &= \Delta_t(a_t, \theta_t) + \Delta_t(a_{t, \star}, \theta_t) +  \mu(a_{t, \star}^{\top}
 \theta_t) -
\mu(a_t^{\top} \theta_t)
\\
 &= \Delta_t(a_t, \theta_t) + \Delta_t(a_{t, \star}, \theta_t)+  \mu(a_{t, \star}^{\top}
 \theta_t) +
  \beta_t(a^\star_t)-
\mu(a_t^{\top} \theta_t) - \beta_t(a_t) + \beta_t(a_t) - \beta_t(a^\star_t)\;.
\end{align*}
For any optimistic algorithm with an exploration bonus of $\beta_t(.)$ and such that
 the
upper confidence bound of the action $a$ at time $t$ is given by
$\mu(a^\top \theta_t) + \beta_t(a)$, by definition for all $a \in \mathcal{A}_t$
$$
\mu(a^\top \theta_t) + \beta_t(a) \leq \mu(a_t^\top \theta_t) + \beta_t(a_t)\;.
$$
In particular, this is also true for the action $a_{t, \star}$. Therefore, plugging this
 inequality
 in the expression of the instantaneous regret gives
\begin{align*}
r_t \leq \Delta_t(a_t, \theta_t) + \Delta_t(a_{t, \star}, \theta_t) + \beta(a_t) - \beta(a^
\star_t)\;.
\end{align*}
Under the additional assumption that $\Delta_t(a, \theta) \leq \beta_t(a)$, we
 obtain the announced result.
\end{proof}
This proposition shows that any improvement in an upper bound
of $\Delta_t(a, \theta_t)$ will result in an improvement of the regret, as
long as the exploration bonus satisfies the assumption stated in the proposition.

\section{ON THE WORST CASE REGRET IN THE $K$-ARM SETTING}
\label{sec:worst_case_K}

In this section, we build upon the analysis from \cite{garivier2011upper} to provide a worst case 
regret bound for the sliding window policy in the $K$-arm setting. Even if a proper lower bound is missing, 
the results we provide here suggest that in some cases sliding window policies can suffer a regret 
of order $\mathcal{O}(\Gamma_T^{1/3} T^{2/3})$ 
in the simpler $K$-arm setting. In particular, this would mean that the $T^{2/3}$ dependency is not a 
sub-optimality from our setting but can already be seen for forgetting policies in the non-contextual setting.
Worst-case regret bounds (i.e. gap independent) for forgetting policies  in non-stationary environments have seen little treatment in the literature.

\paragraph{Setting.} The setting considered in this section is the one from \cite{garivier2011upper}.
At each time $t$, the player chooses an arm $I_t \in \{1,...,K\}$ based on the previous rewards and actions. 
Upon selecting $I_t$ a reward $X_t(I_t)$ is observed. We consider abruptly changing environments as in other
sections, where the distribution of the rewards remains constant during phases and changes at unknown time
instants. At time $t$, the arm $i$ has a mean reward $\mu_t(i)$. As before, $\Gamma_T$ denote the number of
abrupt changes in the reward distributions before time $T$. 
Following the notation from \cite{trovo2020sliding}, we denote the $\Gamma_T$ breakpoints
$\mathcal{B}= \{ b_1, ..., b_{\Gamma_T} \}$. We can associate $\Gamma_T$ stationary phases 
$\{ \phi_1,..., \phi_{\Gamma_T} \}$ with these breakpoints, where
 $\phi_i = \{  t \in \{1,..., T \} \textnormal{ s.t } b_{i -1} \leq t < b_i  \}$ and $b_0=1$.
It is further assumed that for all arms and all time instants the means of the reward distributions lie in $[0,B]$.
In this section the focus is on the forgetting policy
using a sliding window but the same arguments can be used with exponentially increasing weights.

\paragraph{Improving the problem dependent bound.} In \citep[Theorem 2]{garivier2011upper}, 
the number of times the arm $i$ is played before time $T$ while being sub-optimal is upper bounded in
 expectation as
\begin{align}
\label{eq:garivier_utile}
\mathbb{E}\left[ N_T(i) \right] \leq \frac{C(\tau)}{ (\Delta{\mu_T(i}))^2 } \frac{T \log(\tau)}{\tau}
+ \tau \Gamma_T + \log^2(\tau) \;,
\end{align}
where
$$
\Delta \mu_T(i) = \min \{ \mu_t(i_t^\star) - \mu_t(i): t \in \{1,...,T \}, \mu_t(i) < \mu_t(i_t^\star) \} \;.
$$
This result has a worst case flavor in the sense that $\Delta \mu_T(i)$ is the minimum distance between the
 mean
of the optimal arm and the mean of the $i$-th arm when $i$ is sub-optimal
over the entire time horizon. We obtain a less pessimistic
bound by decomposing the regret into the $\Gamma_T$ different stationary phases
and upper-bounding the number of times a sub-optimal arm is drawn 
in each of these phases $\phi$.
The upper-bound naturally depends on $\Delta^{\phi}_i$, 
the difference between the mean of the 
optimal arm and the $i$-th arm in the $\phi$-th stationary phase rather than 
$\Delta \mu_T(i)$. This is of utmost importance as for some phases 
$\Delta^{\phi}_i$ can be significantly larger than $\Delta \mu_T(i)$.

During the $\phi$-th stationary phase, let $\mu_i^\phi$ denote the mean of
the $i$-th arm and $N^\phi_i$ denote the number of times the arm $i$ is selected. 
The regret can be decomposed as follows:

\begin{equation}
\label{eq:reg_K_arm_phase}
\mathbb{E}\left[ R_T \right] =
 \sum_{t=1}^T (\mu_t^\star - \mu_t(i_t)) = 
 \sum_{i=1}^K \sum_{\phi = 1}^{\Gamma_T} \Delta_i^\phi 
\mathbb{E}[N^\phi_i] \;.
\end{equation}

\paragraph{A worst-case bound.} The bound from Equation \eqref{eq:garivier_utile}
is problem dependent and depends explicitly on the minimum gap. 
It is interesting to study the worst case regret. In particular when 
$\Delta \mu_T(i)$ goes to $0$ the upper bound from Equation \eqref{eq:garivier_utile} 
becomes uninformative. At the same time, 
with a small gap $\Delta_i^\phi$ the cost of selecting the $i$-th arm rather than the optimal one 
diminishes. The trade-off between these two opposite effects is made explicit in the following result.

\begin{thm}
\label{th:upper_bound_problem_indep}
The worst case regret of the sliding window policy from \citep{garivier2011upper}, can be upper-bounded by
$$
\mathbb{E}[R_T] \leq C_1 \sqrt{K} \frac{T}{\sqrt{\tau}} + C_2 \sqrt{K} \tau \Gamma_T + C_3 K \frac{T}{\tau} \;,
$$
with $C_1$, $C_2$ and $C_3$ universal constants that depends only on the logarithm of $\tau$.

In particular, setting $\tau = \frac{T^{2/3}}{ K^{1/3}\Gamma_T^{2/3}} $ yields: 
$$
\mathbb{E}[R_T] = \widetilde{\mathcal{O}}( K^{2/3} \Gamma_T^{1/3} T^{2/3}) \;.
$$
\end{thm}
\begin{proof}
\begin{align*}
\mathbb{E}[ R_T] &= \sum_{i=1}^K \sum_{\phi = 1}^{\Gamma_T} \Delta_i^\phi
\mathbb{E}[N_i^\phi]
= \sum_{i, \phi: \Delta_i^\phi > \Delta} \Delta_i^\phi
\mathbb{E}[N_i^\phi]  + 
 \sum_{i, \phi: \Delta_i^\phi \leq \Delta} \Delta_i^\phi
\mathbb{E}[N_i^\phi] \\
 & \leq \sum_{i, \phi: \Delta_i^\phi > \Delta} \Delta_i^\phi
\mathbb{E}[N_i^\phi]  + \Delta \sum_{i=1}^K \sum_{\phi = 1}^{\Gamma_T}
\mathbb{E}[N_i^\phi] 
\leq \sum_{i, \phi: \Delta_i^\phi > \Delta} \Delta_i^\phi
\mathbb{E}[N_i^\phi]  + \Delta T \;.
\end{align*}
The next step consists in upper bounding the expected number of times
 the arm $i$ is selected in the $\phi$-th phase.
We recall that $N_i^\phi$ is defined as 
$$
N_i^\phi = \sum_{t \in \phi} \mathds{1}(I_t = i \ne i_t^\star) = \sum_{t= b_{\phi-1}}
^{b_\phi} \mathds{1}(I_t = i \ne i_t^\star)\;.
$$
We introduce $N_t(\tau, i) = \sum_{s=t- \tau +1}^{t} \mathds{1}(I_s = i)$, the number of times the arm $i$ was selected in the $\tau$ steps preceding  $t$.
We have the following:
\begin{align*}
N_i^\phi &= \sum_{t= b_{\phi-1}}
^{b_{\phi-1} + \tau -1 } \mathds{1}(I_t = i \ne i_t^\star) + 
\sum_{t= b_{\phi-1} + \tau}
^{b_{\phi}} \mathds{1}(I_t = i \ne i_t^\star) 
\leq 
\tau + \sum_{t= b_{\phi-1} + \tau}
^{b_{\phi}} \mathds{1}(I_t = i \ne i_t^\star) \\
& \leq 
\tau 
+ \sum_{t= b_{\phi-1} + \tau}
^{b_{\phi}} \mathds{1}(I_t = i \ne i_t^\star, N_t(\tau, i) \leq A_i^\phi)
+ \sum_{t= b_{\phi-1} + \tau}
^{b_{\phi}} \mathds{1}(I_t = i \ne i_t^\star, N_t(\tau, i) > A_i^\phi )  \;.
\end{align*}
The first term can be bounded using \cite[Lemma 1]{garivier2011upper} that is restated here.
\begin{lemma}[Lemma 1 in \citep{garivier2011upper}]
\label{lemma:garivier_1}
Let $i \in \{1,...,K \}$. For any positive integer $\tau$ and  any positive $m$,
$$
\sum_{t=K+1}^{T} \mathds{1}(I_t = i, N_t(\tau, i) \leq m) \leq \lceil T/ \tau \rceil m\;.
$$
\end{lemma}
Lemma \ref{lemma:garivier_1} can be adapted to our setting and by introducing 
$T^\phi$ the length of the $\phi$-th stationary phase, one has:
$$
\sum_{t= b_{\phi-1} + \tau}
^{b_{\phi}} \mathds{1}(I_t = i \ne i_t^\star, N_t(\tau, i) \leq A_i^\phi) \leq 
\lceil T^\phi/ \tau \rceil A_i^\phi\;.
$$
This in turn gives,
\begin{align*}
N_i^\phi \leq  \tau + \lceil T^\phi / \tau \rceil A_i^\phi 
+ \sum_{t= b_{\phi-1} + \tau}
^{b_{\phi}} \mathds{1}(I_t = i \ne i_t^\star, N_t(\tau, i) >  A_i^\phi ) \;.
\end{align*}
We recall that the upper confidence bound for the sliding-window strategy has the following form in the $K$ arm setting \citep{garivier2011upper}:
$$
UCB_i(t) = \bar{X}_t(\tau, i) + c_t(\tau, i ) \;,
$$
with 
$$
\bar{X}_t(\tau, i)  = \frac{1}{N_t(\tau, i)} \sum_{s =t- \tau +1}^{t} X_s(i) 
\mathds{1}(I_s = i)  \quad \text{and } \quad 
c_t(\tau, i) = B \sqrt{  \frac{\xi \log(\min(t, \tau))}{N_t(\tau,i)}  }\;.
$$
Following the same arguments than \cite{garivier2011upper}  when the event 
$\{ I_t = i \ne i_t^\star , N_t(\tau,i) > A_i^\phi \}$ holds, at least one of the three following events $E_1, E_2, E_3$ must be true where:
$$
E_1 = \{  \bar{X}_t(\tau, i) > \mu_t(i) + c_t(\tau, i) \} \quad \text{the case where $\mu_t(i)$ is over-estimated.}
$$
$$
E_2 =  \{  \bar{X}_t(\tau, i_t^\star) < \mu_t^\star - c_t(\tau, i_t^\star) \} 
\quad \text{the case where the best arm at time $t$ is under-estimated.}
$$
$$
E_3=  \{ \mu_t^\star - \mu_t(i) \leq  2 c_t(\tau, i), N_t(\tau, i) > A_i^\phi \} 
\quad \text{the case where the means are too close to each others.}
$$
From now on, we set 
$$
A_i^\phi = \frac{4 B^2 \xi \log(\tau)}{(\Delta_i^\phi)^2} \;.
$$ 
In doing so, on the event $E_3$ the following holds:
\begin{align*}
c_t(\tau,i) &= B\sqrt{\frac{\xi \log( \min(t, \tau))}{N_t(\tau,i)} } 
< B\sqrt{\frac{\xi \log( \min(t, \tau))}{A_i^\phi} } 
<  \frac{ \Delta_{i}^\phi}{2}\sqrt{\frac{\log(\min(t, \tau))}{\log(\tau)}} 
< \frac{\Delta_i^\phi}{2} \;.
\end{align*}
Therefore, this choice of $A_i^\phi$ ensures that the event $E_3$ never occurs.
Bounding the probability of the events $E_1$ and $E_2$ can be done with the concentration inequality established in \citep{garivier2011upper}.
For any $\eta>0$, by selecting a specific value of $\xi$ one can obtain,
$$
\mathbb{P}(E_1) \leq \frac{ \Big\lceil \frac{\log(\min(t,\tau))}{\log(1+\eta)} \Big\rceil}{\min(t,\tau)} 
\quad \text{and}
\quad
\mathbb{P}(E_2) \leq \frac{ \Big\lceil \frac{\log(\min(t,\tau))}{\log(1+\eta)} \Big\rceil}{\min(t,\tau)} \;.
$$
Consequently we have,
\begin{align*}
\mathbb{E}[N_i^\phi ] & \leq \tau + 
\lceil T^\phi /  \tau  \rceil  \frac{4 B^2 \xi \log(\tau)}{(\Delta_i^\phi)^2} + 2 \sum_{t= b_{\phi-1}  +
 \tau }^{b_{\phi}} \frac{ \Big\lceil \frac{\log(\min(t,\tau))}{\log(1+\eta)} \Big\rceil}{\min(t,\tau)}  \;.
\end{align*}
Plugging this in the regret's upper bound gives:
\begin{align*}
\mathbb{E}[R_T] &\leq \sum_{i, \phi: \Delta_i^\phi > \Delta} \Delta_i^\phi 
\left( \tau + \lceil T^\phi /  \tau  \rceil  \frac{4 B^2 \xi \log(\tau)}{(\Delta_i^\phi)^2}
+ 
2 \sum_{t= b_{\phi-1}  + \tau }^{b_{\phi}} 
\frac{ \Big\lceil \frac{\log(\min(t,\tau))}{\log(1+\eta)} \Big\rceil}{\min(t,\tau)}   \right) + \Delta T
\\
& \leq \sum_{i, \phi: \Delta_i^\phi > \Delta}  \frac{4 B^2 \xi \log(\tau)}{\Delta_i^\phi}  
\lceil T^\phi /  \tau  \rceil 
+ \sum_{i, \phi: \Delta_i^\phi > \Delta} \Delta_i^\phi \left( \tau + 
2 \sum_{t= b_{\phi-1}  + \tau }^{b_{\phi}} 
\frac{ \Big\lceil \frac{\log(\min(t,\tau))}{\log(1+\eta)} \Big\rceil}{\min(t,\tau)} \right) 
+ \Delta T \\
& \leq \frac{4 B^2 \xi \log(\tau) K}{\Delta}  \frac{T}{\tau} +  \tau K \Gamma_T B 
+ 2 K B \sum_{\phi = 1}^{\Gamma_T} \sum_{t= b_{\phi-1}  + \tau }^{b_{\phi}} 
\frac{ \Big\lceil \frac{\log(\min(t,\tau))}{\log(1+\eta)} \Big\rceil}{\min(t,\tau)} + \Delta T \;.
\end{align*}
In the last inequality we have used $\Delta_i^\phi \leq B$ coming from $\mu_i(t) \in [0,B]$ for all $i$ and 
all $t \leq T$.
%for simplifying the calculations but the orders for the different terms would remain the same when keeping them.
Furthermore,
\begin{align*}
\sum_{\phi = 1}^{\Gamma_T} \sum_{t= b_{\phi-1}  + \tau }^{b_{\phi}} \frac{ \Big\lceil \frac{\log(\min(t,\tau))}{\log(1
+\eta)} \Big\rceil}{\min(t,\tau)} 
\leq  \sum_{t= \tau}^T \frac{ \frac{\log(\min(t,\tau))}{\log(1+ \eta)} + 1}{\min(t,\tau)} = \frac{T}{\tau} 
\left(\frac{\log(\tau)}{\log(1+ \eta)} + 1  \right) \;.
\end{align*}
Hence, 
\begin{align*}
\mathbb{E}[R_T] & \leq \frac{4 B^2 \xi \log(\tau) K}{\Delta}  \frac{T}{\tau} + \Delta T + \tau K \Gamma_T B 
+ 2KB\left(\frac{\log(\tau)}{\log(1+ \eta)} + 1  \right)
 \frac{ T}{\tau}\;.
\end{align*}
By differentiating with respect to $\Delta$, the right hand side is maximized when setting $\Delta = 2 B
 \sqrt{\frac{ \xi \log(\tau ) K} {\tau}}$.
With this value of $\Delta$,
\begin{align*}
\mathbb{E}[R_T] & \leq 4 B \sqrt{\xi \log(\tau)} \sqrt{K} \frac{T}{\sqrt{\tau}} +  B K \tau \Gamma_T + 
2 B  K  \log(\tau) \frac{T}{\tau}\;.
\end{align*}
Now by selecting  $\tau = \frac{T^{2/3}}{ K^{1/3} \Gamma_T^{2/3}}$, we obtain the announced scaling.
\end{proof}

\begin{rem}
The term $T/\sqrt{\tau}$ that can be seen in the worst case bound proposed in Theorem
\ref{th:upper_bound_problem_indep} also appears in the gap independent
bound of $\SW$ (Theorem
\ref{thm:regret_wost_case_SW}). 
When focusing on gap dependent bounds, there is also a strong similarity.
In the $K$-arm setting, Equation \eqref{eq:garivier_utile} has a $T/\tau$ dependency. 
This term can also be seen in the GLB setting in Theorem \ref{thm:fjqiqposejf_SW} using 
an analogous assumption
on the gap. This analogy explains why the upper-bounds have the same scaling in the $K$-arm and in 
the GLB setting. Going from $T/\sqrt{\tau}$ to $T/\tau$ when adding the assumption on the gaps 
is the key step allowing a scaling of the regret of order $\widetilde{\mathcal{O}}(\sqrt{T \Gamma_T})$.
\end{rem}

\end{document}